\documentclass{article}

\usepackage{microtype}
\usepackage{graphicx}
\usepackage{subfigure}
\usepackage{booktabs} %

\newcommand{\papertitle}{Infinite attention: NNGP and NTK for deep attention networks}

\usepackage{tabularx}
\usepackage{multirow} 

\usepackage{hyperref}
\hypersetup{
    unicode=false,
    pdftoolbar=true,
    pdfmenubar=true,
    pdffitwindow=false,
    pdfstartview={FitH},
    pdftitle={\papertitle},
    pdfauthor={Jiri Hron, Yasaman Bahri, Jascha Sohl-Dickstein, Roman Novak},
    pdfsubject={\papertitle},
    pdfkeywords={Infinite Neural Networks, Neural Tangent Kernel, Gaussian Processes, Attention},
    pdfnewwindow=true,
    colorlinks=True,
    linkcolor=BurntOrange,
    citecolor=RawSienna,
    filecolor=magenta,
    urlcolor=blue
}

\usepackage[accepted]{icml2020}

\usepackage{pifont}

\makeatletter
\newcommand*\bigcdot{\mathpalette\bigcdot@{.5}}
\newcommand*\bigcdot@[2]{\mathbin{\vcenter{\hbox{\scalebox{#2}{$\m@th#1\bullet$}}}}}
\makeatother

\newcommand{\activationSymbol}{f}
\newcommand{\activitySymbol}{g}

\newcommand{\depthSymbol}{\ell}

\newcommand{\inputSymbol}{x}

\newcommand{\indexedObject}[4]{{#1}^{#2}_{#3}{#4}}
\newcommand{\indexedActivation}[3]{\indexedObject{\activationSymbol}{#1}{#2}{(#3)}}
\newcommand{\indexedActivity}[3]{\indexedObject{\activitySymbol}{#1}{#2}{(#3)}}

\newcommand{\nonlinearity}{\phi}
\newcommand{\depth}{L}

\newcommand{\sequenceVariable}{n}

\newcommand{\indexSet}{\mathcal{X}}

\newcommand{\genericRV}{X}

\newcommand{\rowIndex}{n}
\newcommand{\rowVariance}{\sigma^2_{\rowIndex}}

\newcommand{\limitStd}{\sigma_{*}}
\newcommand{\limitVariance}{\sigma^2_{*}}

\newcommand{\colIndex}{i}

\newcommand{\generalSum}{S}
\newcommand{\littleO}{\mathrm{o}}

\newcommand{\summand}{\gamma}

\newcommand{\projection}{\mathcal{T}}
\newcommand{\projectionIndeces}{\mathcal{L}}
\newcommand{\projectionCoefficients}{\alpha}

\usepackage{microtype}
\usepackage{graphicx}
\usepackage{booktabs} 
\usepackage{hyperref}

\usepackage[utf8]{inputenc} 
\usepackage{url}            
\usepackage{nicefrac}       

\usepackage[british]{babel}
\usepackage{natbib}
\usepackage{enumerate}
\usepackage{cancel}
\usepackage{soul}
\usepackage{multirow}
\usepackage{mathtools}
\usepackage{amsmath,amsthm,verbatim,amssymb,amsfonts,amscd}
\usepackage{commath}
\usepackage{bbm}
\usepackage{thmtools,thm-restate}
\usepackage{cleveref} 

\allowdisplaybreaks

\newtheorem{theorem}{Theorem}
\newtheorem{lemma}[theorem]{Lemma}

\newtheorem{corollary}[theorem]{Corollary}
\newtheorem{definition}[theorem]{Definition}

\DeclareMathOperator*{\E}{\mathbb{E}}

\DeclareMathOperator{\Prob}{\mathbb{P}}

\DeclareMathOperator{\poly}{poly}

\DeclareMathOperator{\vectorise}{vect}

\DeclarePairedDelimiter\floor{\lfloor}{\rfloor}

\DeclarePairedDelimiterX{\infdivx}[2]{(}{)}{%
	#1 \, \delimsize\| \, #2%
}
\DeclarePairedDelimiterX{\commaparen}[2]{(}{)}{%
	#1 , #2%
}
\DeclarePairedDelimiterX{\innerprod}[2]{\langle}{\rangle}{%
	#1 , #2%
}
\DeclarePairedDelimiterX{\parenththree}[3]{\lparen}{\rparen}{
	#1 ; #2 ; #3	
}

\newcommand*\xbar[1]{%
	\hbox{%
		\vbox{%
			\hrule height 0.5pt 
			\kern0.5ex
			\hbox{%
				\kern-0.1em
				\ensuremath{#1}%
				\kern-0.1em
			}%
		}%
	}%
}

\newcommand{\rv}[1]{\expandafter\MakeUppercase\expandafter{#1}}
\newcommand{\distr}[1]{\expandafter\MakeUppercase\expandafter{\mathrm{#1}}}

\newcommand{\natnum}{\mathbb{N}}

\newcommand{\indicator}[1]{\mathbbm{1}_{#1}}

\newcommand{\kernel}{\kappa}

\newcommand{\kerntilde}{\tilde{\kernel}}
\newcommand{\kernelf}[2]{\kernel_{#1}^{#2} \, \commaparen}

\newcommand{\kerntildef}[2]{\kerntilde_{#1}^{#2} \, \commaparen}
\newcommand{\ntk}{\Theta}
\newcommand{\ntktilde}{\widetilde{\ntk}}
\newcommand{\ntkhat}{\widehat{\ntk}}
\newcommand{\ntkf}[2]{\ntk_{#1}^{#2} \, \commaparen}
\newcommand{\ntktildef}[2]{\ntktilde_{#1}^{#2} \, \commaparen}
\newcommand{\ntkhatf}[2]{\ntkhat_{#1}^{#2} \, \commaparen}
\newcommand{\interpKern}{\mathcal{I}}
\newcommand{\interpKernCoeff}{\alpha}

\newcommand{\R}[1]{\mathbb{R}^{#1}}

\newcommand{\gauss}{\mathcal{N}}

\newcommand{\gp}{\mathcal{GP}}

\newcommand{\convergeDist}{\rightsquigarrow}
\newcommand{\convergeProb}{\overset{P}{\longrightarrow}}

\newcommand\restr[2]{{
  \left.\kern-\nulldelimiterspace 
  #1 
  \right|_{#2} 
  }}

\newcommand{\generalMean}{\bar{\generalSum}}

\newcommand{\genericDimenstion}{d}
\newcommand{\inputDimension}{\genericDimenstion^{0}}

\newcommand{\spatialDimension}{\genericDimenstion^s}
\newcommand{\layerDimension}[1]{\genericDimenstion_{\sequenceVariable}^{#1}}
\newcommand{\layerDimensionN}{\layerDimension{\depthSymbol}}

\newcommand{\headIndex}{h}
\newcommand{\headSymbol}{H}

\newcommand{\softmax}{\zeta}
\newcommand{\softmaxMean}{\bar{\softmax}}
\newcommand{\params}{\theta}

\newcommand{\keySymbol}{K}
\newcommand{\querySymbol}{Q}
\newcommand{\valueSymbol}{V}
\newcommand{\logitSymbol}{G}
\newcommand{\outputSymbol}{O}
\newcommand{\key}[2]{\keySymbol_{#1}^{#2}}
\newcommand{\query}[2]{\querySymbol_{#1}^{#2}}
\newcommand{\val}[2]{\valueSymbol_{#1}^{#2}}
\newcommand{\logit}[2]{\logitSymbol_{#1}^{#2}}
\newcommand{\keyN}{\key{\sequenceVariable}{\depthSymbol \headIndex}}
\newcommand{\queryN}{\query{\sequenceVariable}{\depthSymbol \headIndex}}
\newcommand{\valueN}{\val{\sequenceVariable}{\depthSymbol \headIndex}}
\newcommand{\logitN}{\logit{\sequenceVariable}{\depthSymbol \headIndex}}

\newcommand{\posEmbSymbol}{E}
\newcommand{\posEmb}[2]{\posEmbSymbol_{#1}^{#2}}

\newcommand{\covPosEmb}[2]{R_{#1}^{#2}}

\newcommand{\weightMatSymbol}{W}
\newcommand{\W}{\weightMatSymbol}
\newcommand{\WQ}{\weightMatSymbol^{\querySymbol}}
\newcommand{\WK}{\weightMatSymbol^{\keySymbol}}
\newcommand{\WV}{\weightMatSymbol^{\valueSymbol}}

\newcommand{\weightQ}[1]{\weightMatSymbol_{\sequenceVariable, #1}^{\depthSymbol \headIndex, \querySymbol}}
\newcommand{\weightK}[1]{\weightMatSymbol_{\sequenceVariable, #1}^{\depthSymbol \headIndex, \keySymbol}}
\newcommand{\weightV}[1]{\weightMatSymbol_{\sequenceVariable, #1}^{\depthSymbol \headIndex, \valueSymbol}}
\newcommand{\weightOGen}[2]{\weightMatSymbol_{#1}^{#2}}
\newcommand{\weightO}[1]{\weightOGen{\sequenceVariable, #1}{\depthSymbol, \outputSymbol}}
\newcommand{\weightsQ}{\weightMatSymbol_{\sequenceVariable}^{\depthSymbol \headIndex, \querySymbol}}
\newcommand{\weightsK}{\weightMatSymbol_{\sequenceVariable}^{\depthSymbol \headIndex, \keySymbol}}
\newcommand{\weightsV}{\weightMatSymbol_{\sequenceVariable}^{\depthSymbol \headIndex, \valueSymbol}}
\newcommand{\weightsO}{\weightMatSymbol_{\sequenceVariable}^{\depthSymbol, \outputSymbol}}

\newcommand{\headDimension}{\genericDimenstion_{\sequenceVariable}^{\depthSymbol, \headSymbol}}

\newcommand{\valueDimension}{\genericDimenstion_{\sequenceVariable}^{\depthSymbol, \valueSymbol}}
\newcommand{\logitDimension}{\genericDimenstion_{\sequenceVariable}^{\depthSymbol, \logitSymbol}}
\newcommand{\peDimension}{\genericDimenstion_{\sequenceVariable}^{\depthSymbol, \posEmb{}{}}}

\newcommand{\stdSymbol}{\sigma}
\newcommand{\keyStd}{\stdSymbol_{\keySymbol}}
\newcommand{\queryStd}{\stdSymbol_{\querySymbol}}
\newcommand{\valueStd}{\stdSymbol_{\valueSymbol}}
\newcommand{\outStd}{\stdSymbol_{\outputSymbol}}
\newcommand{\QKStd}{\stdSymbol_{\querySymbol \keySymbol}}
\newcommand{\OVStd}{\stdSymbol_{\outputSymbol \valueSymbol}}
\newcommand{\keyVar}{\keyStd^2}
\newcommand{\queryVar}{\queryStd^2}
\newcommand{\valueVar}{\valueStd^2}
\newcommand{\outVar}{\outStd^2}
\newcommand{\QKVar}{\QKStd^2}
\newcommand{\OVVar}{\OVStd^2}

\newcommand{\activations}[2]{\activationSymbol_{#1}^{#2}}
\newcommand{\activationsN}{\activations{\sequenceVariable}{\depthSymbol}}

\newcommand{\logits}[2]{\logitSymbol_{#1}^{#2}}
\newcommand{\logitsN}{\logits{\sequenceVariable}{\depthSymbol}}

\newcommand{\scaling}{\tau}
\newcommand{\fIndexA}{a}
\newcommand{\fIndexB}{b}
\newcommand{\gIndexA}{a'}
\newcommand{\gIndexB}{b'}
\newcommand{\gIndexZA}{i'}
\newcommand{\gIndexZB}{j'}

\newcommand{\projectionN}{\projection_{\sequenceVariable}}
\newcommand{\summandN}{\summand_{\sequenceVariable, \headIndex}}

\newcommand{\projectionLogitSymbol}{\projection}
\newcommand{\summandLogitSymbol}{\varphi}
\newcommand{\projectionLogit}{\projectionLogitSymbol_{\sequenceVariable}^{\logitSymbol}}
\newcommand{\projectionCoefficientsLogit}{\beta}
\newcommand{\summandLogit}[1]{\summandLogitSymbol_{#1}}
\newcommand{\summandLogitN}{\summandLogit{\sequenceVariable, j}}

\newcommand{\tildeActivation}[3]{\indexedObject{\tilde{\activationSymbol}}{#1}{#2}{(#3)}}
\newcommand{\tildeActivity}[3]{\indexedObject{\tilde{\activitySymbol}}{#1}{#2}{(#3)}}
\newcommand{\tildeLogit}[2]{\widetilde{\logitSymbol}_{#1}^{#2}}
\newcommand{\tildeLogitN}[3]{\tildeLogit{#1}{#2} (#3)}
\newcommand{\uW}{\widetilde{\weightMatSymbol}}

\definecolor{darkgreen}{rgb}{0,0.6,0}

\icmltitlerunning{\papertitle}

\begin{document}

\twocolumn[
\icmltitle{\papertitle}

\icmlsetsymbol{equal}{*}

\begin{icmlauthorlist}
\icmlauthor{Jiri Hron}{c}
\icmlauthor{Yasaman Bahri}{g}
\icmlauthor{Jascha Sohl-Dickstein}{g}
\icmlauthor{Roman Novak}{g}
\end{icmlauthorlist}

\icmlaffiliation{g}{Google Brain}
\icmlaffiliation{c}{University of Cambridge. Work done while interning at Google Brain.}

\icmlcorrespondingauthor{Jiri Hron}{jh2084@cam.ac.uk}
\icmlkeywords{Machine Learning, ICML, nngp, ntk, gaussian process, infinite, attention, wide, neural networks}

\vskip 0.3in
]

\printAffiliationsAndNotice{}  %

\begin{abstract}
There is a growing amount of literature on the relationship between wide neural networks (NNs) and Gaussian processes (GPs), identifying an equivalence between the two for a variety of NN architectures.
This equivalence enables, for instance, accurate approximation of the behaviour of wide Bayesian NNs without MCMC or variational approximations, or characterisation of the distribution of randomly initialised wide NNs optimised by gradient descent without ever running an optimiser.
We provide a rigorous extension of these results to NNs involving attention layers,
showing that unlike single-head attention, which induces non-Gaussian behaviour, multi-head attention architectures behave as GPs as the number of heads tends to infinity. 
We further discuss the effects of positional encodings and layer normalisation, and propose modifications of the attention mechanism which lead to improved results for both finite and infinitely wide NNs.
We evaluate attention kernels empirically, leading to a moderate improvement upon the previous state-of-the-art on CIFAR-10 for GPs without trainable kernels and advanced data preprocessing.
Finally, we introduce new features to the Neural Tangents library \citep{novak2020neural} allowing applications of NNGP/NTK models, with and without attention, to variable-length sequences, with an example on the IMDb reviews dataset.
\end{abstract}

\section{Introduction}\label{sect:intro}

One of the currently most active research directions in theoretical deep learning is the study of NN behaviour as the number of parameters in each layer goes to infinity \citep[e.g.,][]{matthews2018gaussian,lee2018deep,garriga2019deep,novak2019bayesian,li2018learning,allenzhu19convergence,du2019gd,arora2019,yang2019v2}. 
Building upon these efforts, we study the asymptotic behaviour of NNs with attention layers \citep{bahdanu2015nmt,vaswani2017attention} and derive the corresponding neural network Gaussian proccess (NNGP) and Neural Tangent kernels \citep[NTK,][]{jacot2018ntk,lee2019wide}.

Beyond their recent empirical successes \citep[e.g.,][]{radford2019language,devlin2019bert}, attention layers are also interesting from the theoretical perspective as the standard proof techniques used to establish asymptotic Gaussianity of the input-to-output mappings represented by wide NNs \citep{matthews2018gaussian,yang2019v2} cannot be applied. 

To understand why, consider the following simplified attention layer model: let $\inputSymbol \in \R{\spatialDimension \times \genericDimenstion'}$ be the input with $\spatialDimension$ \emph{spatial} and $\genericDimenstion'$ \emph{embedding} dimensions (by spatial, we mean, e.g., the number of tokens in a string or pixels in an image),  
$\WQ, \WK, \WV \in \R{\genericDimenstion' \times \genericDimenstion}$ be weight matrices, and define queries
$\querySymbol(\inputSymbol) \coloneqq \inputSymbol \WQ$, keys $\keySymbol(\inputSymbol) \coloneqq \inputSymbol \WK$, and values $\valueSymbol(\inputSymbol) \coloneqq \inputSymbol \WV$ as usual.
The attention layer output is then
\begin{align}\label{eq:single_head_informal}
    \activationSymbol(\inputSymbol)
    \coloneqq
    \softmax \biggl(
        \frac{
            \querySymbol(\inputSymbol) \keySymbol(\inputSymbol)^\top
        }{\sqrt{\genericDimenstion}}
    \biggr) 
    \valueSymbol(\inputSymbol)
    =
    \softmax (
        \logitSymbol(\inputSymbol)
    )
    \valueSymbol(\inputSymbol)
    \, ,
\end{align}
where $\softmax$ is the row-wise softmax function.

Now observe that 
$\dim \logitSymbol(\inputSymbol) = \spatialDimension \times \spatialDimension$ 
where the spatial dimension $\spatialDimension$ stays finite even as the number of parameters---here proportional to $\genericDimenstion$---goes to infinity.
As we will show rigorously in \Cref{sect:theory}, this fact combined with the $\genericDimenstion^{-1/2}$ scaling 
causes each column of $\activationSymbol(\inputSymbol)$ to be a linear combination of the \emph{same stochastic matrix} $\softmax(\logitSymbol(\inputSymbol))$, and thus statistically dependent even in the infinite width limit.

Since the exchangeability based arguments
\citep{matthews2018gaussian,garriga2019deep} require that certain moment statistics of $\activationSymbol(\inputSymbol)$ asymptotically behave as if its columns were independent \citep[see condition {\em b} in lemma 10,][]{matthews2018gaussian}, they do not extend to attention layers in a straightforward manner.
Similarly, the proofs based on Gaussian conditioning \citep{novak2019bayesian,yang2019v2} require that given the input $\inputSymbol$, the conditional covariance of each column of $\activationSymbol(x)$ converges (in probability) to the same \emph{deterministic} positive semidefinite matrix \citep[see propositions 5.5 and G.4 in][]{yang2019v2} which will not be the case due to the aforementioned stochasticity of $\softmax(\logitSymbol(\inputSymbol))$.

Among the many interesting contributions in \citep{yang2019v2}, the author proposes to resolve the above issue by replacing the $\genericDimenstion^{-1/2}$ scaling in \Cref{eq:single_head_informal} by $\genericDimenstion^{-1}$ which does enable application of the Gaussian conditioning type arguments.
However, it also forces the attention layer to only perform computation similar to average pooling in the infinite width limit, and reduces the overall expressivity of attention even if suitable modifications preventing the pooling behaviour are considered (see \Cref{sect:linear_scaling_limit}).

We address the above issues by modifying the exchangability based technique and provide a rigorous characterisation of the infinite width behaviour under both the $\genericDimenstion^{-1/2}$ and $\genericDimenstion^{-1}$ scalings.
We also show that positional encodings \citep{gehring2017convolutional,vaswani2017attention} can improve empirical performance even in the infinite width limit, and propose modifications to the attention mechanism which results in further gains for both finite and infinite NNs.
In experiments, we moderately improve upon the previous state-of-the-art result on CIFAR-10 for GP models without data augmentation and advanced preprocessing \citep[cf.][]{li2019enhanced}.
Finally, since attention is often applied to text datasets, we release code allowing applications of NNGP/NTK models to variable-length sequences, including an example on the IMDb reviews dataset.

\begin{figure}[tbp]
    \begin{center}
        \centerline{
            \includegraphics[width=0.5\columnwidth,keepaspectratio]{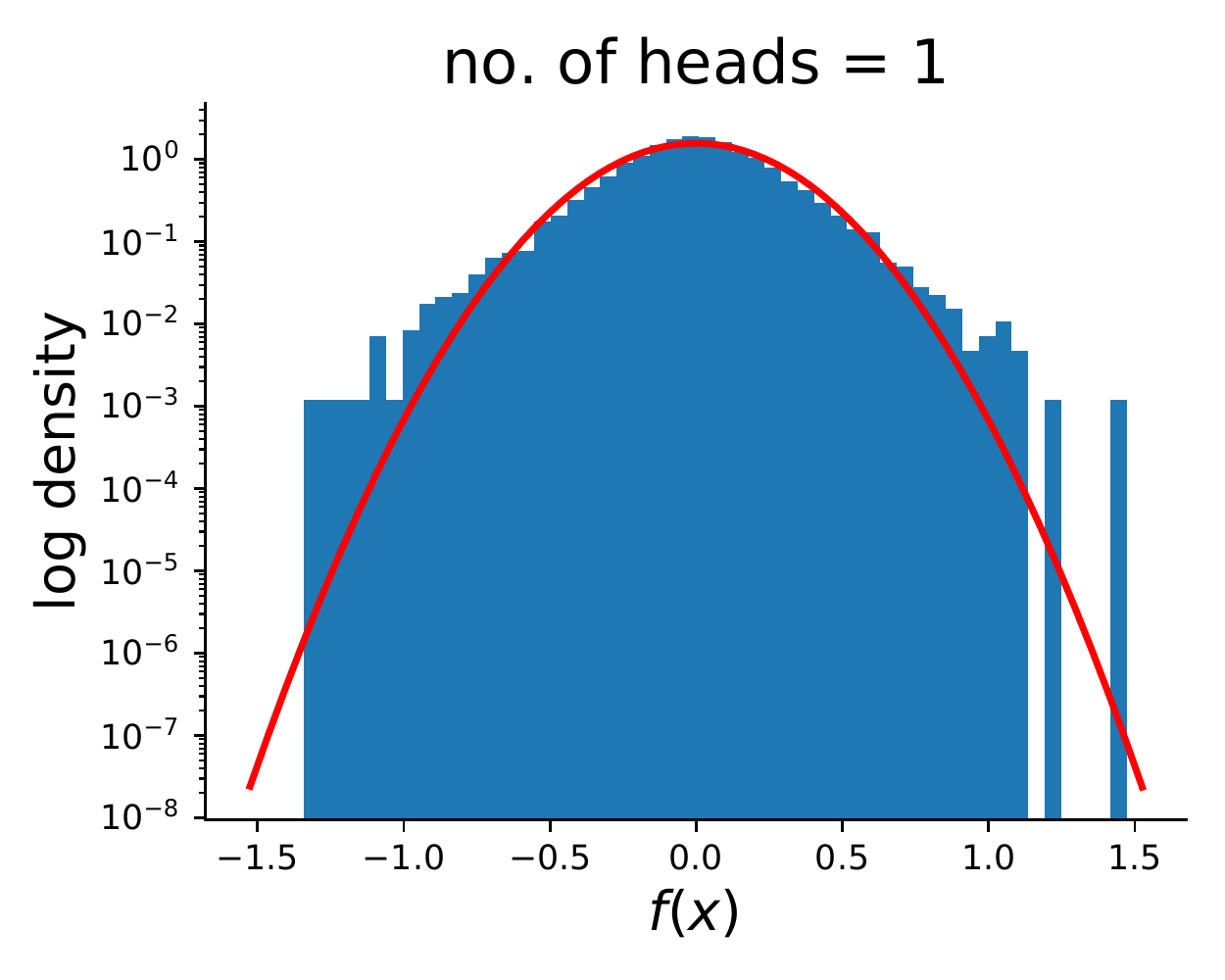}
            \hfill
            \includegraphics[width=0.5\columnwidth,keepaspectratio]{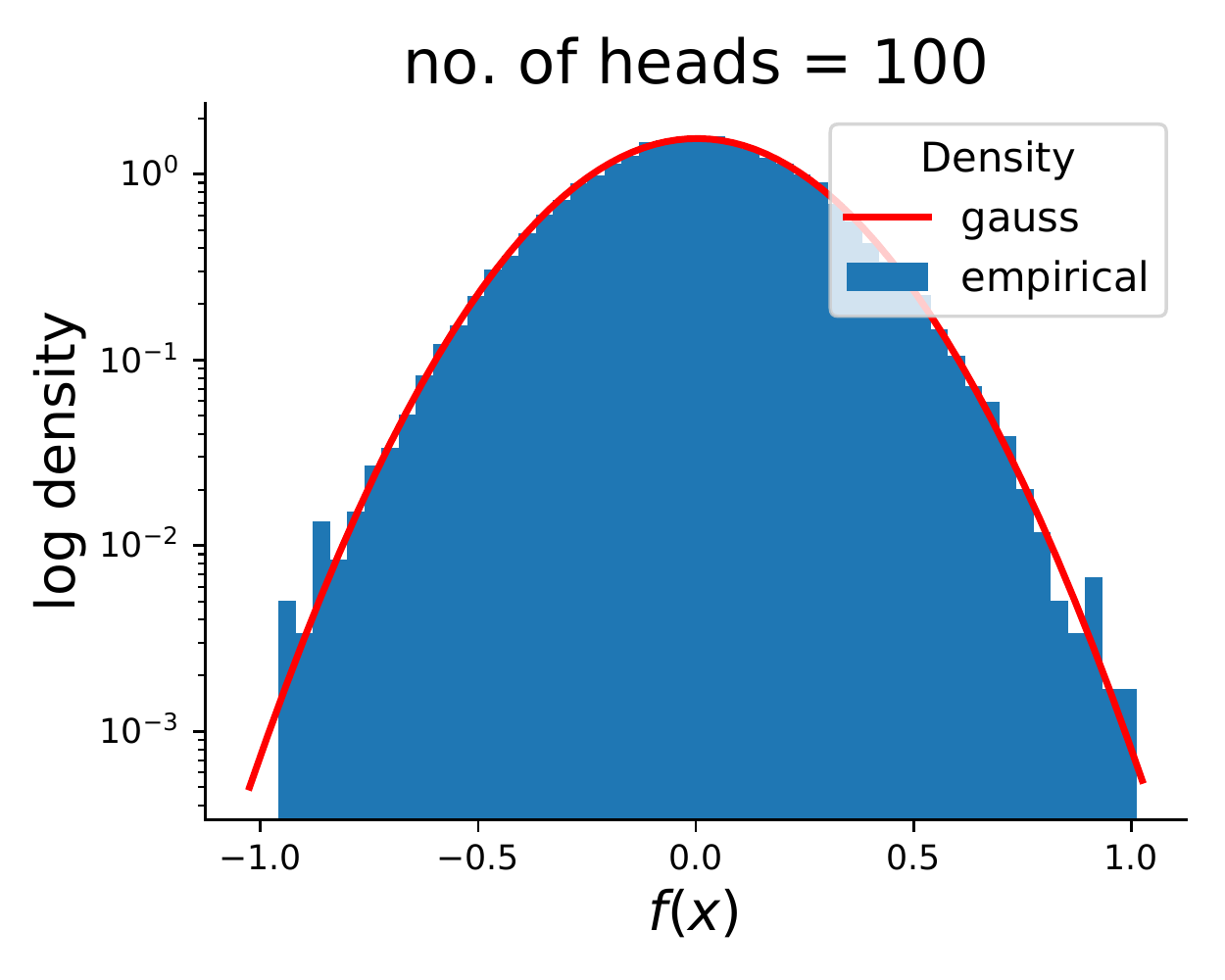}
        }
        \vskip -0.2in
        \caption{
        \textbf{Distribution of an attention layer output for single-head} (left) \textbf{and 100-head} (right) architecture at initialisation under the $\genericDimenstion^{-1/2}$ scaling when $\genericDimenstion$ is large ($1000$).
        Red line is the Gaussian density with sample mean and variance substituted for its parameters.
        Unlike multi-head, the empirical distribution of single-head attention significantly deviates from Gaussian despite $\genericDimenstion \gg 0$.}
        \label{fig:scale_mixture}
    \end{center}
    \vskip -0.25in
\end{figure}

\newlength{\tableFS}
\setlength{\tableFS}{8.4pt}

\begin{table*}[t]
\caption{
\textbf{Overview of the discussed kernels.}
The $\genericDimenstion$ column refers to the $\genericDimenstion^{-1}$ and $\genericDimenstion^{-1/2}$ scaling of the $\querySymbol (\inputSymbol) \keySymbol(\inputSymbol)^\top$.
$(\kerntilde, \ntktilde)$ denote the input and $(\kernel, \ntk)$ the output NNGP and NTK kernels.
\textsc{NNGP} and \textsc{NTK} columns are stated as updates for full $\spatialDimension \times \spatialDimension$ covariance blocks unless the generic spatial dimension subscripts $a b$ are used.
To fit to page width, we use superscripts to denote dependence on inputs, e.g., replacing $\kerntildef{}{}{\inputSymbol}{\inputSymbol}$ by $\kerntilde^{\inputSymbol\inputSymbol'}$.
$\langle A, B \rangle_F = \sum_{ij} A_{ij} B_{ij}$ is the Frobenius product of matrices $A, B$, with $\| A \|_F^2 = \langle A, A \rangle$.
$\interpKern$ denotes interpolation,
e.g.,
$\interpKern \circ \kerntildef{}{}{\inputSymbol}{\inputSymbol'} = \interpKernCoeff \kerntildef{}{}{\inputSymbol}{\inputSymbol'} + (1 - \interpKernCoeff) \covPosEmb{}{}$ with fixed hyperparameters $\interpKernCoeff \in [0, 1]$ and $\covPosEmb{}{}$ (a generic covariance related to initialisation of positional encodings);
the special case $\covPosEmb{}{} = I$ is denoted by $\interpKern_{I}$.
$\ddag$ is for optional operators (e.g., $\interpKern^\ddag \circ  \kerntilde^{\inputSymbol\inputSymbol'}$ can be replaced with $\kerntilde^{\inputSymbol\inputSymbol'}$).
$\WQ = \WK$ initialisation assumed for all $\genericDimenstion^{-1}$, and $\softmax = \text{identity}$ for all $\genericDimenstion^{-1/2}$ kernels (see \Cref{sect:linear_scaling_limit,sect:softmax_alternatives} respectively).
See \Cref{sect:theory,sect:beyon_vanilla_attn} for derivations, and \citep{yang2019v2} for the \textsc{LayerNorm} kernel (stated here for ease of reference).
}
\label{tab:kernel_overview}
\begin{center}
\fontsize{\tableFS}{1.2\tableFS}
\begin{sc}
\renewcommand{\arraystretch}{1.75}
\begin{tabular}{lcccc}
    \toprule
    Kernel & $d$ &  NNGP & NTK  \\
    \midrule
    \multirow{2}{*}{Vanilla}    &
    $1$
    &
    $
    \softmax(
        \kerntilde^{\inputSymbol\inputSymbol}
    )
        \kerntilde^{\inputSymbol\inputSymbol'}
    \softmax(
        \kerntilde^{\inputSymbol'\inputSymbol'}
    )^\top
    $ 
    &
    $
    2 
    \kernel^{\inputSymbol\inputSymbol'}
    +
    \softmax(
        \kerntilde^{\inputSymbol\inputSymbol}
    )
        \ntktilde^{\inputSymbol\inputSymbol'}
    \softmax(
        \kerntilde^{\inputSymbol'\inputSymbol'}
    )^\top
    $ 
    \\
    
    &
    $\frac{1}{2}$
    &
    $
    \kerntilde^{\inputSymbol\inputSymbol'}
    \norm{
        \kerntilde^{\inputSymbol\inputSymbol'}
    }_F^2
    $
    & 
    $
    4 
    \kernel_{ab}^{\inputSymbol\inputSymbol'}
    +
    \langle
        \kerntilde^{\inputSymbol\inputSymbol'}
        ,
        2 
        \kerntilde_{ab}^{\inputSymbol\inputSymbol'}
        \ntktilde^{\inputSymbol\inputSymbol'}
        +
        \ntktilde_{ab}^{\inputSymbol\inputSymbol'}
        \kerntilde^{\inputSymbol\inputSymbol'}
    \rangle_F
    $
    \\
    
    \midrule
    
    \multirow{2}{*}{
    \begin{minipage}[t]{0.2\columnwidth}%
    Random Positional Encoding
    \end{minipage}
    } 
    &
    $1$
    &
    $
    \softmax (
        \interpKern_I \circ
        \kerntilde^{\inputSymbol\inputSymbol}
    )
    [
        \interpKern_I \circ
        \kerntilde^{\inputSymbol\inputSymbol'}
    ]
    \softmax (
        \interpKern_I \circ
        \kerntilde^{\inputSymbol'\inputSymbol'}
    )^\top
    $
    &
    $
    2
    \kernel^{\inputSymbol\inputSymbol'}
    +
    \softmax (
        \interpKern_I \circ
        \kerntilde^{\inputSymbol\inputSymbol}
    )
    [
        \interpKern_I \circ
        \ntktilde^{\inputSymbol\inputSymbol'}
    ]
    \softmax (
        \interpKern_I \circ
        \kerntilde^{\inputSymbol'\inputSymbol'}
    )^\top
    $
    \\ 
    
    &
    $\frac{1}{2}$
    &
    $
    \interpKern_I \circ
    \kerntilde^{\inputSymbol\inputSymbol'}
    \norm{
        \interpKern_I \circ \kerntilde^{\inputSymbol\inputSymbol'}
    }_F^2
    $
    & 
    $
    4 
    \kernel_{ab}^{\inputSymbol\inputSymbol'}
    +
    \langle
        \interpKern_I \circ
        \kerntilde^{\inputSymbol\inputSymbol'}
        ,
        2
        [
            \interpKern_I \circ
            \kerntilde_{ab}^{\inputSymbol\inputSymbol'}
        ]
            \interpKern_I \circ
            \ntktilde^{\inputSymbol\inputSymbol'}
        +
        [
            \interpKern_I \circ
            \ntktilde_{ab}^{\inputSymbol\inputSymbol'}
        ]
        \interpKern_I \circ
        \kerntilde^{\inputSymbol\inputSymbol'}
    \rangle_F
    $
    \\
    
    \midrule
    
    \multirow{2}{*}{
    \begin{minipage}[t]{0.2\columnwidth}%
    Structured Positional Encoding
    \end{minipage}
    } 
    &
    $1$
    &
    $
    \softmax (
        \interpKern \circ
        \kerntilde^{\inputSymbol\inputSymbol}
    )
    [
        \interpKern^\ddag \circ
        \kerntilde^{\inputSymbol\inputSymbol'}
    ]
    \softmax (
        \interpKern \circ
        \kerntilde^{\inputSymbol'\inputSymbol'}
    )^\top
    $
    &
    $
    2 
    \kernel^{\inputSymbol\inputSymbol'}
    +
    \softmax (
        \interpKern \circ
        \kerntilde^{\inputSymbol\inputSymbol}
    )
    [
        \interpKern^\ddag \circ
        \ntktilde^{\inputSymbol\inputSymbol'}
    ]
    \softmax (
        \interpKern \circ
        \kerntilde^{\inputSymbol'\inputSymbol'}
    )^\top
    $
    \\
    
    &
    $\frac{1}{2}$
    &
    $
    \interpKern \circ
    \kerntilde^{\inputSymbol\inputSymbol'}
    \langle
        \interpKern^\ddag \circ
        \kerntilde^{\inputSymbol\inputSymbol'}
        ,
        \interpKern \circ \kerntilde^{\inputSymbol\inputSymbol'}
    \rangle_F
    $
    &
    $
    \begin{aligned}
    4 \kernel_{ab}^{\inputSymbol\inputSymbol'}
    &+
    \langle
        \interpKern^\ddag \circ
        \kerntilde^{\inputSymbol\inputSymbol'}
        ,
        [
            \interpKern \circ
            \kerntilde_{ab}^{\inputSymbol\inputSymbol'}
        ]
        \interpKern \circ
        \ntktilde^{\inputSymbol\inputSymbol'}
        +
        [
            \interpKern \circ
            \ntktilde_{ab}^{\inputSymbol\inputSymbol'}
        ]
        \interpKern \circ
        \kerntilde^{\inputSymbol\inputSymbol'}
    \rangle_F
    \\
    &+
    \langle
        \interpKern \circ
        \kerntilde^{\inputSymbol\inputSymbol'}
        \phantom{\ddag}
        ,
        [
            \interpKern \circ
            \kerntilde_{ab}^{\inputSymbol\inputSymbol'}   
        ]
        \interpKern^\ddag \circ
        \ntktilde^{\inputSymbol\inputSymbol'}
    \rangle_F
    \end{aligned}
    $
    \\

    \midrule
    
    Residual
    &
    --
    &
    $
        \interpKernCoeff 
        \kerntilde^{\inputSymbol\inputSymbol'}
        +
        (1 - \interpKernCoeff)
        \covPosEmb{}{}
        \kerntilde^{\inputSymbol\inputSymbol'}
        \covPosEmb{}{\top}
    $
    &
    $
    2 (1 - \interpKernCoeff)
    \kernel^{\inputSymbol\inputSymbol'}
    +
    \interpKernCoeff
    \ntktilde^{\inputSymbol \inputSymbol'}
    +
    (1 - \interpKernCoeff) 
    \covPosEmb{}{} 
    \ntktilde^{\inputSymbol\inputSymbol'}
    \covPosEmb{}{\top}
    $
    \\
    
    \midrule
    
    LayerNorm  &
    --
    &
    $
        \kerntilde_{ab}^{\inputSymbol\inputSymbol'}
        [
            \kerntilde_{aa}^{\inputSymbol\inputSymbol}
            \kerntilde_{bb}^{\inputSymbol\inputSymbol'}
        ]^{-1/2}
    $
    &
    $
        \ntktilde_{ab}^{\inputSymbol\inputSymbol'}
        [
            \ntktilde_{aa}^{\inputSymbol\inputSymbol}
            \ntktilde_{bb}^{\inputSymbol'\inputSymbol'}
        ]^{-1/2}
    $

    \\
    \bottomrule
\end{tabular}
\end{sc}
\end{center}
\vskip -0.18in
\end{table*}

\section{Definitions and notation}\label{sect:background}

\textbf{Neural networks:}
$\indexedActivation{\depthSymbol}{}{\inputSymbol}$ denotes the~output of $\depthSymbol$\textsuperscript{th} layer for an~input $\inputSymbol \in \indexSet \subset \R{\spatialDimension \times \inputDimension}$, and $\indexedActivity{\depthSymbol}{}{\inputSymbol} \coloneqq \nonlinearity (\indexedActivation{\depthSymbol}{}{\inputSymbol})$ the~corresponding post-nonlinearity where $\nonlinearity \colon \R{} \to \R{}$ is the~activation function applied elementwise (for convenience, we set $\indexedActivity{0}{}{\inputSymbol} = \inputSymbol$).
We assume the~network has $\depth \in \natnum$ hidden layers, making $\indexedActivation{\depth + 1}{}{\inputSymbol}$
the output,
and that the input set $\indexSet$ is \emph{countable}.
As we will be examining behaviour of sequences of increasingly wide NNs, the~variables corresponding to the~$\sequenceVariable$\textsuperscript{th} network are going to be denoted by a~subscript $\sequenceVariable$ (e.g., $\indexedActivation{\depthSymbol}{\sequenceVariable}{\inputSymbol}$ is the~output of $\depthSymbol$\textsuperscript{th} layer of the~$\sequenceVariable$\textsuperscript{th} network in the~sequence evaluated at $\inputSymbol$).
We also use
\begin{align*}
    \activations{\sequenceVariable, \cdot j}{\depthSymbol} 
    &\coloneqq 
    \{ 
        \indexedActivation{\depthSymbol}{\sequenceVariable, ij}{\inputSymbol} 
        \colon 
        \inputSymbol \in \indexSet,
        i \in [\spatialDimension]
    \} 
    \\
    \activations{\sequenceVariable}{\depthSymbol} 
    &\coloneqq 
    \{ 
        \activations{\sequenceVariable, \cdot j}{\depthSymbol} 
        \colon 
        j \in \natnum 
    \} 
    \, ,
\end{align*}
with $[\spatialDimension] = \{1, 2, \ldots, \spatialDimension \}$.
To reduce clutter, we omit the $\depthSymbol$ index where it is clear from the context or unimportant.

\textbf{Shapes:} $\indexedActivation{\depthSymbol}{\sequenceVariable}{\inputSymbol}, \indexedActivity{\depthSymbol}{\sequenceVariable}{\inputSymbol} \in \R{\spatialDimension \times \layerDimensionN}$ with $\spatialDimension, \layerDimensionN \in \natnum$ respectively the~\emph{spatial} and \emph{embedding} dimensions.
If there are multiple spatial dimensions, such as height and width for images, we assume these have been flattened into a single dimension.
Finally, we will allow the row space dimension of $\W_{\sequenceVariable}^{\depthSymbol,\querySymbol}, \W_{\sequenceVariable}^{\depthSymbol,\keySymbol} \in \R{\layerDimension{\depthSymbol - 1} \times \logitDimension}$ to differ from that of $\W_{\sequenceVariable}^{\depthSymbol,\valueSymbol} \in \R{\layerDimension{\depthSymbol - 1} \times \layerDimensionN}$, leading to the modified definition 
\begin{align}\label{eq:logit_def}
    \logitSymbol_{\sequenceVariable}^{\depthSymbol} (\inputSymbol)
    =
    \frac{\querySymbol_{\sequenceVariable}^{\depthSymbol}(\inputSymbol) \keySymbol_{\sequenceVariable}^{\depthSymbol} (\inputSymbol)^\top}{\sqrt{\logitDimension}}
\end{align}

\textbf{Multi-head attention:} \Cref{eq:single_head_informal} describes a vanilla version of a \emph{single-head} attention layer.
Later in this paper, we examine the \emph{multi-head} attention alternative in which the~output $\indexedActivation{\depthSymbol}{\sequenceVariable}{\inputSymbol}$ is computed as
\begin{equation}\label{eq:attention_out}
    \indexedActivation{
    \depthSymbol
    }{\sequenceVariable}{\inputSymbol}
    =
    \bigl[
        \indexedActivation{
        \depthSymbol
        1}{\sequenceVariable}{\inputSymbol},
        \ldots,
        \indexedActivation{
        \depthSymbol
        \headDimension}{\sequenceVariable}{\inputSymbol}
    \bigr]
    \weightsO
    \, ,
\end{equation}
i.e., by stacking the outputs of $\headDimension \in \natnum$ independently parametrised heads into a $\spatialDimension \times \headDimension \layerDimensionN$ matrix and projecting back into $\spatialDimension \times \layerDimensionN$ by $\weightsO \in \R{\headDimension \layerDimensionN \times \layerDimensionN}$.
The embedding dimension of each head $\valueDimension$ can optionally differ from $\layerDimensionN$.
To distinguish the weight matrices corresponding to the individual heads, we will be using a superscript $\headIndex$, e.g., 
$\queryN (\inputSymbol) = \indexedActivity{\depthSymbol - 1}{\sequenceVariable}{\inputSymbol} \weightsQ$.

\textbf{Weight distribution:} 
As usual, we will assume Gaussian initialisation of the weights, i.e., 
$\weightQ{ij} \sim \gauss (0, \queryVar / \layerDimension{\depthSymbol - 1} )$, 
$\weightK{ij} \sim \gauss (0, \keyVar / \layerDimension{\depthSymbol - 1} )$,
$\weightV{ij} \sim \gauss (0, \valueVar / \layerDimension{\depthSymbol - 1} )$, 
and
$\weightO{ij} \sim \gauss (0, \outVar / (\headDimension \layerDimension{\depthSymbol}) )$, all i.i.d.\ over the $i, j$ and $\depthSymbol, \headIndex$ indices for all $\sequenceVariable$.
The scaling of variance by inverse of the input dimension is standard and ensures that the asymptotic variances do not diverge \citep{neal1996,lecun1998efficient,he2015delving}.
Throughout \Cref{sect:beyon_vanilla_attn,sect:theory}, we assume all the $\sigma^2$ parameters are equal to one, and only state the results in full generality in the appendix.

\textbf{NNGP/NTK:} 
As discussed in the introduction, randomly initialised NNs induce a distribution over the $\activationSymbol_{\sequenceVariable}^{\depth + 1}$ mappings.
For a variety of architectures, this distribution converges (weakly) to that of a GP as $\min_{\depthSymbol \in [\depth]} \layerDimensionN \to \infty$, both at initialisation (NNGP), and after continuous gradient descent optimisation of the randomly initialised NN
with respect to a mean squared error loss (NTK).
Both the NNGP and NTK distributions are typically zero mean, and we use $\kernel^{\depth + 1}$ and $\ntk^{\depth + 1}$ to denote their respective kernel functions.

These kernel functions tend to have a recursive structure where each layer in the underlying NN architecture is associated with a mapping $(\kernel^{\depthSymbol - 1}, \ntk^{\depthSymbol - 1}) \mapsto (\kernel^{\depthSymbol}, \ntk^{\depthSymbol})$ transforming the NNGP and NTK kernels according to the layer's effect on the outputs in the infinite width limit.
Since nonlinearities are typically not treated as separate layers, we use $\kerntilde^{\depthSymbol}$ and $\ntktilde^{\depthSymbol}$ to denote the intermediate transformation $(\kernel^{\depthSymbol - 1}, \ntk^{\depthSymbol - 1}) \mapsto (\kerntilde^{\depthSymbol}, \ntktilde^{\depthSymbol})$ they induce.

We generally assume every layer is followed by a nonlinearity, setting $(\kerntilde^{\depthSymbol}, \ntktilde^{\depthSymbol}) = (\kernel^{\depthSymbol-1}, \ntk^{\depthSymbol-1})$ if none is used.
In the next two sections,
we uncover the mappings $(\kerntilde^{\depthSymbol}, \ntktilde^{\depthSymbol}) \mapsto (\kernel^{\depthSymbol}, \ntk^{\depthSymbol})$ induced by various attention architectures.

\section{Attention and Gaussian process behaviour}\label{sect:theory}

Throughout the rest of this paper, we restrict our focus to increasingly wide NNs including at least one attention layer.
In particular, we consider sequences of NNs such that
\begin{align}\label{eq:sim_limit}
    \lim_{\sequenceVariable \to \infty} \min_{\depthSymbol \in [\depth]} \layerDimensionN = \infty
    \, ,
\end{align}
and the reader should thus interpret any statements involving $\sequenceVariable \to \infty$ as implicitly assuming
\Cref{eq:sim_limit} holds.

Due to the space constraints, most of the technical discussion including derivation of the NTK is relegated to \Cref{sect:proofs}.
In this section, we only focus on the key step in our proof which relies on an inductive argument adapted from \citep{matthews2018gaussian}.
On a high level, the induction is applied from $\depthSymbol = 1$ to $\depthSymbol = \depth + 1$, and establishes that whenever 
$\activationSymbol_{\sequenceVariable}^{\depthSymbol - 1}$ converges in distribution to 
$\gp(0 ,\kernel^{\depthSymbol - 1})$ at initialisation, $\activationSymbol_{\sequenceVariable}^{\depthSymbol}$ also converges in distribution to $\gp(0 ,\kernel^{\depthSymbol})$ as $\sequenceVariable \to \infty$.
Since this fact is known for dense, convolutional, and average pooling layers, and almost all nonlinearities \citep{matthews2018gaussian,lee2018deep,garriga2019deep,novak2019bayesian,yang2019v2}, it will be sufficient to show the same for attention layers.

\subsection{Infinite width limit under the $\genericDimenstion^{-1}$ scaling}\label{sect:linear_scaling}

As illustrated in \Cref{fig:scale_mixture}, use of the $\genericDimenstion^{-1/2}$ scaling within a \emph{single-head} architecture leads to a scale mixture behaviour of the attention layer outputs as the number of parameters goes to infinity.
To obtain a Gaussian limit, \citet[appendix A]{yang2019v2} proposes to replace the definition in \Cref{eq:logit_def} by 
$
\logitSymbol_{\sequenceVariable} %
(\inputSymbol) 
= 
(\genericDimenstion_{\sequenceVariable}^{\logitSymbol})^{-1}
\querySymbol_{\sequenceVariable} %
(\inputSymbol) 
\keySymbol_{\sequenceVariable} %
(\inputSymbol)^\top 
$,
i.e., the use of $\genericDimenstion^{-1}$ scaling.
The desired result then follows:

\begin{theorem}[$\genericDimenstion^{-1}$ limit \citep{yang2019v2}]\label{thm:linear_scaling_limit}
    Under the $\genericDimenstion^{-1}$ scaling and the assumptions stated in \citep{yang2019v2}:
    \vspace{-0.5\baselineskip}
    \begin{enumerate}[\hspace{-0.5em}(I)]
        \renewcommand\labelenumi{\emph{\bfseries(\theenumi)}}
        \item 
        For any $(\inputSymbol, \inputSymbol') \in \indexSet \times \indexSet$ and $a, b, i, j \in [\spatialDimension]$, there exist constants $(\softmaxMean_{a i}^\inputSymbol, \softmaxMean_{b j}^{\inputSymbol'}) \in \R{} \times \R{}$ such that
        \begin{align}\label{eq:logit_determ_limit}
            (
                \softmax(
                    \logitSymbol_{\sequenceVariable}%
                    (\inputSymbol)
                )_{ai},
                \softmax(
                    \logitSymbol_{\sequenceVariable}%
                    (\inputSymbol')
                )_{bj}
            ) 
            \to 
            ( 
                \softmaxMean_{a i}^\inputSymbol,
                \softmaxMean_{b j}^{\inputSymbol'}
            )
        \end{align}
        in probability as $\sequenceVariable \to \infty$.
        
        \item 
        $\activations{\sequenceVariable}{}$ %
        converges in distribution to $\activationSymbol%
        \sim \gp(0, \kernel)$ with 
        \begin{align}\label{eq:determ_kernel}
            \kernelf{a b}{}{\inputSymbol}{\inputSymbol'}
            &=
            \E [
                \indexedActivation{}{a1}{\inputSymbol}
                \indexedActivation{}{b1}{\inputSymbol'}
            ]
            \\
            &=
            \sum_{i , j = 1}^{\spatialDimension} 
                \kerntildef{i j}{}{\inputSymbol}{\inputSymbol'}
                    \softmaxMean_{a i}^\inputSymbol
                    \softmaxMean_{b j}^{\inputSymbol'}
            \nonumber
            \, ,
        \end{align}
        and $\activations{\cdot k}{}$ and $\activations{\cdot l}{}$ are independent for any $k \neq l$.
    \end{enumerate}
\end{theorem}

An analogous result also holds for multi-head attention architectures
which follows by the usual argument for fully connected layers as long as either the number of embedding dimensions per head or the number of heads goes to infinity.

\subsection{Limitations of the $\genericDimenstion^{-1}$ scaling}\label{sect:linear_scaling_limit}

While \Cref{thm:linear_scaling_limit} is a good starting point, several issues have to be resolved before using the kernel function described in \Cref{eq:determ_kernel} in practice.
Firstly, since $\W_{\sequenceVariable}^{\querySymbol}$ and $\W_{\sequenceVariable}^{\keySymbol}$ are initialised independently, the $\genericDimenstion^{-1}$ scaled inner products of keys and queries will converge to zero (the mean), and thus for any $a, i$ and $\inputSymbol$, $\softmaxMean_{a i}^{\inputSymbol} \to (\spatialDimension)^{-1}$ in probability by the continuous mapping theorem.
This issue was already noted by \citeauthor{yang2019v2} in appendix~A but not discussed further as the main focus of the paper lies elsewhere.
In any case, substituting $(\spatialDimension)^{-1}$ for all the $\softmaxMean$ coefficients will make $\kernelf{a b}{}{\inputSymbol}{\inputSymbol '} = \kernelf{i j}{}{\inputSymbol}{\inputSymbol '}$ for any $a, b, i, j \in [\spatialDimension]$, and in fact all of these entries will be equivalent to output of a simple global average pooling kernel \citep[equation 17]{novak2019bayesian}.\footnote{In fact, the asymptotic distribution induced by such an attention layer followed by flatten and dense layers is the same as that induced by global average pooling followed by a dense layer.}

Perhaps the simplest way to address the above issue is by drawing the initial weights such that $\W_{\sequenceVariable}^{\querySymbol} = \W_{\sequenceVariable}^{\keySymbol}$.
This will ensure that the key and query for a particular spatial dimension will point in the same direction and thus the attention weight corresponding to itself will be large with high probability.
The resulting formula for $\kernelf{a b}{}{\inputSymbol}{\inputSymbol'}$ is
\begin{align}\label{eq:fixed_determ_kernel}
    \sum_{i , j = 1}^{\spatialDimension} 
        \kerntildef{i j}{}{\inputSymbol}{\inputSymbol'}
        \softmax(
            \kerntildef{a i}{}{\inputSymbol}{\inputSymbol}
        )
        \softmax(
            \kerntildef{b j}{}{\inputSymbol'}{\inputSymbol'}
        )
    \, .
\end{align}

Since \Cref{eq:fixed_determ_kernel} resolves the issue of reduction to average pooling, a natural question is whether  swapping $\genericDimenstion^{-1/2}$ for $\genericDimenstion^{-1}$ has any undesirable consequences in the infinite width limit.
As we will see, this question can be answered in affirmative.
In particular, we start by a proposition inspired by \citep{cordonnier2020On} in which the authors show that an attention layer with a sufficient number of heads is at least as expressive as a standard convolutional layer, 
and that attention layers often empirically learn to perform computation akin to convolution.
In contrast, \Cref{prop:linear_scale_no_conv} proves that there is no initial distribution of $\W_{\sequenceVariable}^{\querySymbol}$ and $\W_{\sequenceVariable}^{\keySymbol}$ which would recover the convolutional kernel \citep{novak2019bayesian,garriga2019deep} in the infinite width limit.

\begin{restatable}{proposition}{linearNoConv}
\label{prop:linear_scale_no_conv}
    There is no set of attention coefficients $\{ \softmaxMean_{a i}^{\inputSymbol} \! \in \R{} \! \colon a, i \in [\spatialDimension], \inputSymbol \in \indexSet \}$ such that for \emph{all} positive semidefinite kernels $\kerntilde$ simultaneously
    \begin{align*}
        \sum_{i , j = 1}^{\spatialDimension} 
            \kerntildef{i j}{}{\inputSymbol}{\inputSymbol'}
                \softmaxMean_{a i}^\inputSymbol
                \softmaxMean_{b j}^{\inputSymbol'}
        =
        \sum_{i = 1}^{\genericDimenstion_f} 
            \kerntildef{N_a(i) N_b(i)}{}{\inputSymbol}{\inputSymbol'}
            \frac{1}{\genericDimenstion_f}
        \, ,
    \end{align*}
    where $\genericDimenstion_f$ is the dimension of the (flattened) convolutional filter, $N_a, N_b \subset [\spatialDimension]$ are the ordered subsets of pixels which are used to compute the new values of pixels $a$ and $b$, respectively, and $N_a(i), N_b(i)$ are the $i$\textsuperscript{th} pixels in $N_a, N_b$.
\end{restatable}

In the next section, we will see that the convolutional kernel can be recovered under the $\genericDimenstion^{-1/2}$ scaling (\Cref{prop:sqrt_scale_conv_recover}). 
However, we need to establish convergence scaling first.

\begin{figure}[tbp]
    \begin{center}
        \centerline{
            \includegraphics[width=0.5\columnwidth,keepaspectratio]{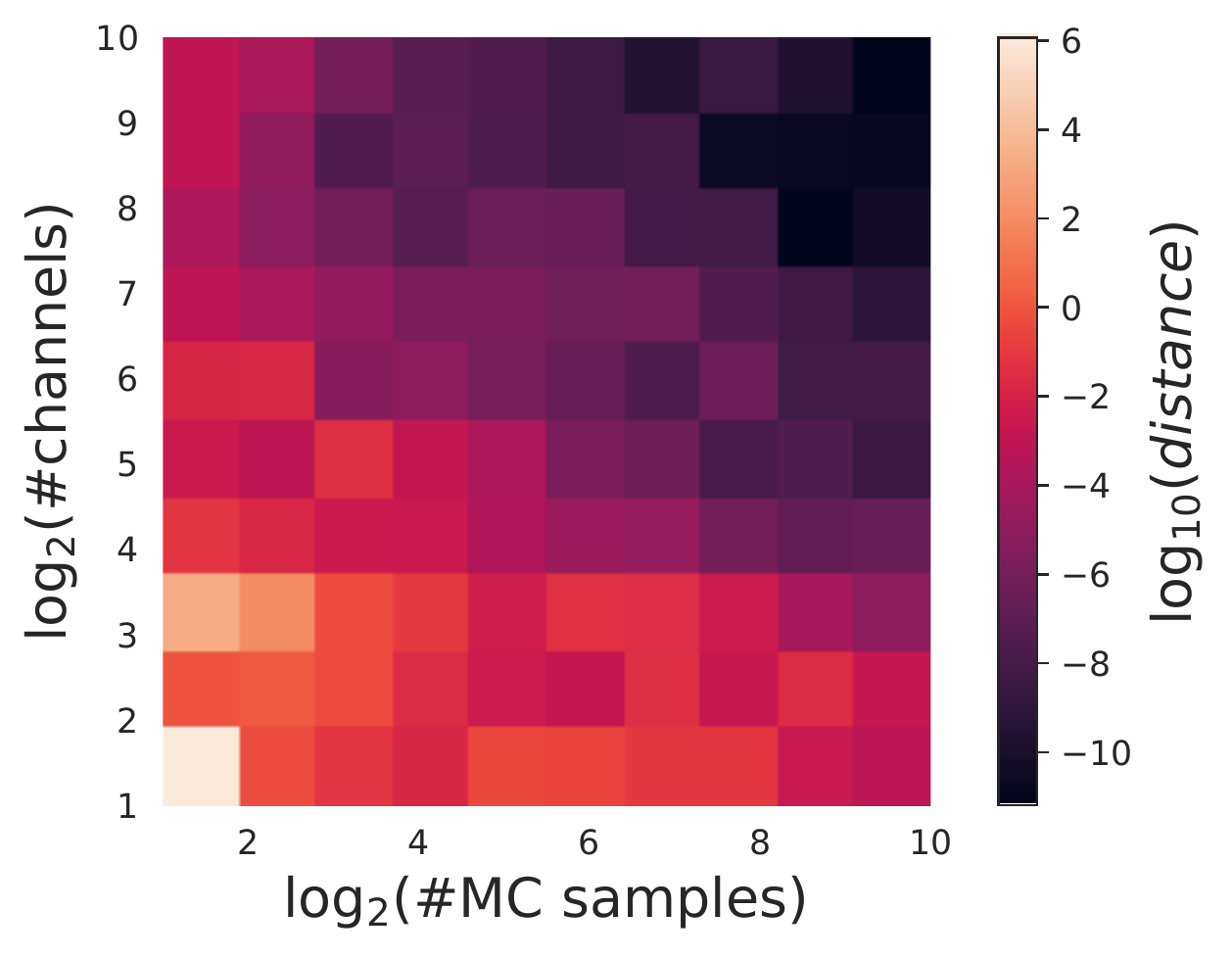}
            \hfill
            \includegraphics[width=0.5\columnwidth,keepaspectratio]{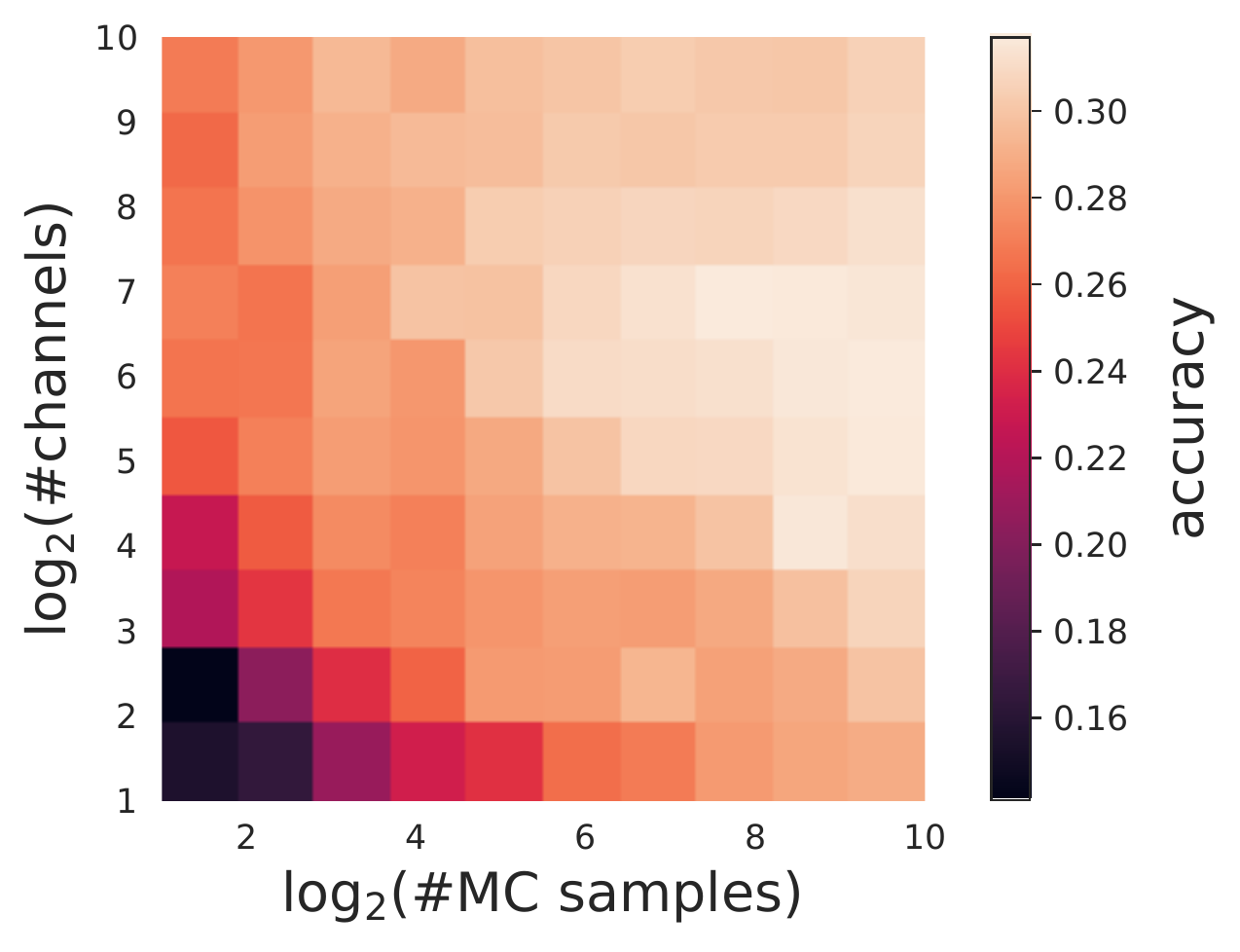}
            }
        \vskip -0.15in
        \caption{
        \textbf{Convergence} (left) \textbf{and validation accuracy} (right) \textbf{plots for an empirical NNGP kernel} estimated by Monte Carlo on a 2K/4K train/validation subset of 8x8-downsampled CIFAR-10, as the number weight samples averaged over (x-axis) and the number of parameters (y-axis) grows. Architecture: Convolution + ReLU, 2x Attention + ReLU, Flatten, Dense. For attention layers, $\logitDimension = \text{\#channels}$ but $\headDimension = \valueDimension =  \floor{\sqrt{\text{\#channels}}}$ to reduce the memory footprint. Details in \Cref{sect:convergence_appendix}.}
        \label{fig:convergence_plots}
    \end{center}
    \vskip -0.25in
\end{figure}

\subsection{Infinite width limit under the $\genericDimenstion^{-1/2}$ scaling}\label{sect:multi_head}

As discussed in \Cref{sect:intro}, single-head attention architectures can exhibit non-Gaussian asymptotic behaviour under the $\genericDimenstion^{-1/2}$ scaling.
This is inconvenient for our purposes as many modern NN architectures combine attention with fully connected, convolutional, and other layer types, all of which have Gaussian NNGP and NTK limits \citep[e.g.,][]{novak2019bayesian,garriga2019deep,yang2019v2}.
This Gaussianity simplifies derivation of the infinite width behaviour of many architectures and allows for easy integration with existing software libraries \citep{novak2020neural}. 
Fortunately, the output of an attention layer becomes asymptotically Gaussian when the number of heads becomes large.

\begin{restatable}[$\genericDimenstion^{-1/2}$ limit]{theorem}{nngpConvergence}
\label{thm:gp_convergence_sqrt}
    Let 
    $\depthSymbol \in \{ 2, \ldots , \depth + 1 \}$, and 
    $\nonlinearity$ be such that $|\nonlinearity(\inputSymbol)| \leq c + m |\inputSymbol|$ for some $c, m \in \R{}_+$.
    Assume $\activations{\sequenceVariable}{\depthSymbol - 1}$
    converges in distribution to $\activationSymbol^{\depthSymbol - 1} \sim \gp(0, \kernel^{\depthSymbol-1})$, such that $\activations{\cdot j}{\depthSymbol - 1}$ and $\activations{\cdot k}{\depthSymbol - 1}$ are independent for any $j \neq k$,
    the~variables $\{ \activations{\sequenceVariable, \cdot j}{\depthSymbol - 1} \colon j \in \natnum \}$ are exchangeable over $j$.
    
    Then as $ \min \, \{ \sequenceVariable, \headDimension , \logitDimension \} \to \infty \, $:
    \vspace{-0.5\baselineskip}
    \begin{enumerate}[\hspace{-0.5em}(I)]
        \renewcommand\labelenumi{\emph{\bfseries(\theenumi)}}
        \item 
        $\logitSymbol_{\sequenceVariable}^{\depthSymbol} = \{ \logitN(x) \colon x \in \indexSet , \headIndex \in \natnum  \}$ converges in distribution to $\logitSymbol^{\depthSymbol} \sim \gp (0, \kernel^{\depthSymbol, \logitSymbol})$ with
        \begin{align}\label{eq:logit_cov}
            \E [\logitSymbol_{a i}^{\depthSymbol \headIndex}(\inputSymbol) \logitSymbol_{b j}^{\depthSymbol \headIndex'}(\inputSymbol')]
            &=
            \delta_{\headIndex = \headIndex'}
            \kerntildef{a b}{\depthSymbol}{\inputSymbol}{\inputSymbol'} \,
            \kerntildef{i j}{\depthSymbol}{\inputSymbol}{\inputSymbol'}
            \, .
        \end{align}
        
        \item 
        $\activationsN$ converges in distribution to $\activationSymbol^{\depthSymbol} \sim \gp(0, \kernel^{\depthSymbol})$ with
        \begin{align}\label{eq:sqrt_scaling_kernel}
            &\kernelf{a b}{\depthSymbol}{\inputSymbol}{\inputSymbol'}
            =
            \E [
                \activationSymbol_{a1}^{\depthSymbol}(\inputSymbol)
                \activationSymbol_{b1}^{\depthSymbol}(\inputSymbol')
            ]
            \\
            &=
            \sum_{i , j = 1}^{\spatialDimension} 
                \kerntildef{i j}{\depthSymbol}{\inputSymbol}{\inputSymbol'}
                \E [\softmax(\logitSymbol^{\depthSymbol 1}(\inputSymbol))_{a i} \softmax(\logitSymbol^{\depthSymbol 1}(\inputSymbol'))_{b j}]
            \nonumber
            \, ,
        \end{align}
        and $\activations{\cdot k}{\depthSymbol}$ and $\activations{\cdot l}{\depthSymbol}$ are independent for any $k \neq l$.
    \end{enumerate}
\end{restatable}

We can now revisit our argument from the previous section, and prove that unlike in \Cref{prop:linear_scale_no_conv}, $\genericDimenstion^{-1/2}$ scaling ensures a convolutional kernel can in principle be recovered.

\begin{restatable}{proposition}{sqrtConvRecover}
\label{prop:sqrt_scale_conv_recover}
    Under the $\genericDimenstion^{-1/2}$ scaling, there exists a distribution over $\logitSymbol$ such that
    for any $\inputSymbol, \inputSymbol'$ and $a, b, i, j$
    \begin{align}
        &\E [
            \softmax(\logitSymbol(\inputSymbol))_{a i}
            \softmax(\logitSymbol(\inputSymbol'))_{b j}
        ]
        \nonumber
        \\
        &=
        \begin{cases}
            \frac{1}{\genericDimenstion_f}\, , &  \exists k \in [\genericDimenstion] \text{ s.t.\ } i = N_a(k) \, , j = N_b(k) \, ,  \\
            0 \, ,&  \text{ otherwise.}
        \end{cases}
    \end{align}
\end{restatable}

\section{Beyond the vanilla attention definition}\label{sect:beyon_vanilla_attn}

Before progressing to empirical evaluation of infinitely wide attention architectures,
two practical considerations have to be addressed: (i)~the $\genericDimenstion^{-1/2}$ scaling induced kernel in \Cref{eq:sqrt_scaling_kernel} involves an analytically intractable integral $\E [\softmax(\logitSymbol^{\depthSymbol 1}(\inputSymbol)) \softmax(\logitSymbol^{\depthSymbol 1}(\inputSymbol'))]$; (ii)~incorporation of positional encodings \citep{gehring2017convolutional,vaswani2017attention}.

\subsection{Alternatives to softmax in attention networks}\label{sect:softmax_alternatives}

We propose to resolve the analytical intractability of the $\E [\softmax(\logitSymbol^{\depthSymbol 1}(\inputSymbol)) \softmax(\logitSymbol^{\depthSymbol 1}(\inputSymbol'))]$ in \Cref{eq:sqrt_scaling_kernel} by substituting functions other than softmax for $\softmax$.
In particular, we consider two alternatives:
(i)~$\softmax(x) = \text{ReLU}(x)$, and (ii)~$\softmax(x) = x$, both applied elementwise. Besides analytical tractability of the expectation, our motivation for choosing (i) and (ii) is that ReLU removes the normalisation while still enforcing positivity of the attention weights, while the identity function allows the attention layer to learn an arbitrary linear combination of the values without constraints.

To see if either is a sensible modification, we evaluated performance of \emph{finite} attention networks on CIFAR-10 for different choices of $\softmax$.
Since softmax typically dampens the marginal variance of attention layer outputs (variance of a convex combination of random variables is upper bounded by the maximum of the individual variances), and both ReLU and identity can also significantly affect scale of the outputs, we optionally add layer normalisation as is common in attention architectures.
We consider no normalisation (\texttt{none}),\footnote{Despite the similarity between attention with ReLU or identity for $\softmax$ and dense layers with cubic nonlinearities, which are known to be hard to train, we found that layer normalisation was not strictly necessary.
We believe this is partly because we only used a single attention layer, and partly because the weights for keys, queries, and values are initialised independently which leads to relatively better behaved distribution of gradients at initialisation.} 
normalisation applied after each head prior to multiplication by $\weightsO$ (\texttt{per\_head}), and normalisation applied to the output after $\weightsO$ (\texttt{at\_output}).

\Cref{fig:softmax_replacements} shows the results across varying hyperparameters and random seeds, and \Cref{tab:finite_attention_best} (\Cref{sect:replacements_appendix}) reports accuracies attained under optimal hyperparameter settings.
As you can see, both the replacement of softmax and addition of layer normalisation significantly increases the performance of the NN, with $\softmax(x) = x$ and \texttt{at\_output} normalisation being the best across variety of hyperparameter choices. 

In light of the above, we will restrict our attention to the identity function alternative for $\softmax$ in the rest of the paper, and contrast its performance with the standard softmax choice where possible (finite NNs, and infinite attention NNs under the $\genericDimenstion^{-1}$ scaling---see \Cref{thm:linear_scaling_limit}).
Similarly, we will also leverage the \texttt{at\_output} layer normalisation over the embedding dimension in our experiments.
As shown by \citet[appendix A]{yang2019v2}, layer normalisation does not prevent Gaussianity of the infinite width limit (see \Cref{tab:kernel_overview} for the associated NNGP and NTK kernel transformations).

\begin{figure}[tbp]
    \begin{center}
        \centerline{
            \includegraphics[width=0.9\columnwidth,keepaspectratio]{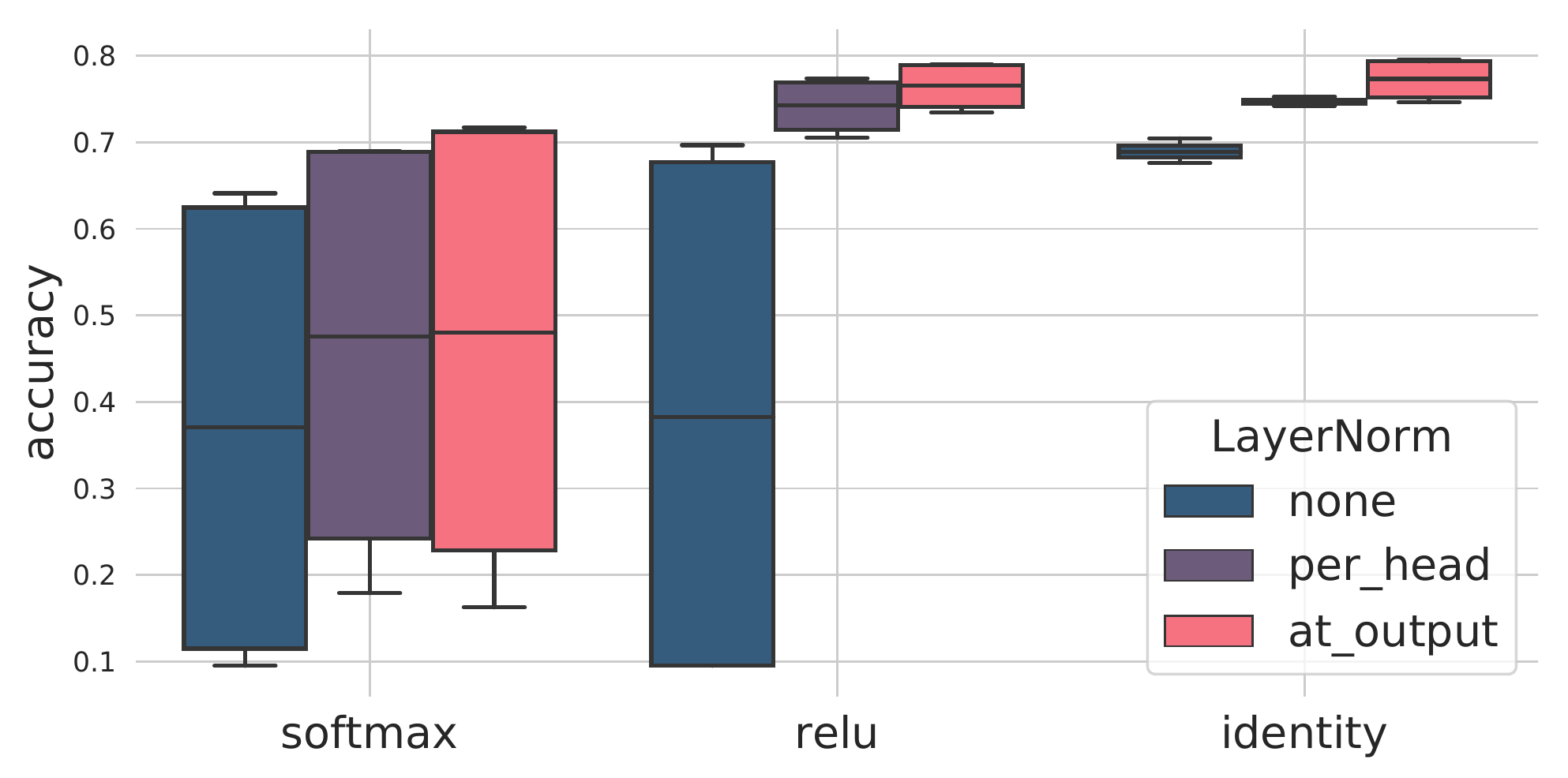}
        }
        \vskip -0.15in
        \caption{\textbf{Comparison of $\softmax$ alternatives.}
        Architecture: 4x Convolution + ReLU, Attention, Flatten, Dense.
        The captured variability is due to multiple random seeds, varying learning rate and network width, illustrating robustness of the reported results.
        Softmax significantly underperforms other $\softmax$ alternatives whenever attention is followed by layer normalisation. Details in \Cref{sect:replacements_appendix}.}
        \label{fig:softmax_replacements}
    \end{center}
    \vskip -0.25in
\end{figure}

\subsection{Positional encodings}\label{sect:positional_encodings}

While substituting the identity function for $\softmax$ as suggested in \Cref{sect:softmax_alternatives} would technically allow us to move on to the experimental evaluation already, we found that positional encodings are as important in the infinite width limit as they are for the finite attention layers \citep{vaswani2017attention}.
Since there are many possible variants of the positional encoding implementation, we focus only on the major points here and provide more detail in \Cref{sect:positional_encodings_appendix}.

In finite networks, some of the most common ways to implement positional encodings is to modify the attention layer input by either \texttt{add}-ing $\indexedActivity{\depthSymbol - 1}{\sequenceVariable}{\inputSymbol} + \posEmb{\sequenceVariable}{\depthSymbol}$ or \texttt{append}-ing $[ \indexedActivity{\depthSymbol - 1}{\sequenceVariable}{\inputSymbol} \, , \posEmb{\sequenceVariable}{\depthSymbol} ]$ a matrix 
$\posEmb{\sequenceVariable}{\depthSymbol}$ 
which may be either fixed or a trainable parameter.
The purpose of $\posEmb{\sequenceVariable}{\depthSymbol}$ is to provide the attention layer with information about the relationships between individual spatial dimensions (e.g., position of a particular pixel in an image, or of a token in a string).

\subsubsection{Effect on the infinite width limit}

If we assume $\posEmb{\sequenceVariable}{\depthSymbol}$ is trainable and each of its columns is initialised independently from $\gauss (0 , R)$,
$\covPosEmb{}{}$ positive semidefinite,
it can be shown that both in the \texttt{add} and \texttt{append} case, the attention layer output converges (in distribution) to a Gaussian infinite width limit (see \Cref{sect:positional_encodings_appendix}).
The corresponding kernels can be stated in terms of an operator $\interpKern$ which interpolates any given kernel $\kernel$ with $\covPosEmb{}{}$
\begin{align}\label{eq:kernel_interp_op}
    \interpKern
    \colon
    \kernelf{}{}{\inputSymbol}{\inputSymbol'}
    \mapsto
    \interpKernCoeff \kernelf{}{}{\inputSymbol}{\inputSymbol'}
    +
    (1 - \interpKernCoeff)
    \covPosEmb{}{}
    \, ,
\end{align}
where $\interpKernCoeff \in [0, 1]$ is a hyperparameter,\footnote{If $\posEmb{\sequenceVariable}{\depthSymbol}$ is \texttt{append}-ed, $\interpKernCoeff = \lim_{\sequenceVariable \to \infty} \layerDimension{\depthSymbol - 1} / (\peDimension + \layerDimension{\depthSymbol - 1})$ with $\peDimension$ the row space dimension of $\posEmb{\sequenceVariable}{\depthSymbol}$.
When $\posEmb{\sequenceVariable}{\depthSymbol}$ is \texttt{add}-ed, we replace $\indexedActivity{\depthSymbol - 1}{\sequenceVariable}{\inputSymbol}$ by 
$\sqrt{\interpKernCoeff} \indexedActivity{\depthSymbol - 1}{\sequenceVariable}{\inputSymbol} + \sqrt{1 - \interpKernCoeff} \posEmb{\sequenceVariable}{\depthSymbol}$ 
so as to prevent increase of the layer's input variance (see \Cref{sect:positional_encodings_appendix}).} yielding the following modification of the kernel induced by the $\genericDimenstion^{-1}$ scaling and $\WQ = \WK$ initialisation (\Cref{eq:fixed_determ_kernel}):
\begin{align}\label{eq:sum_pe_lin_scale_kernel}
    \kernelf{a b}{}{\inputSymbol}{\inputSymbol'}
    =
    \softmaxMean_{a}^{\inputSymbol}
    [
        \interpKern \circ
        \kerntildef{}{}{\inputSymbol}{\inputSymbol'}
    ]
    (\softmaxMean_{b}^{\inputSymbol'})^\top
    \, ,
\end{align}
where 
$
    \softmaxMean_{a}^{\inputSymbol}
    \coloneqq
    \softmax (
        \interpKern \circ
        \kerntildef{}{}{\inputSymbol}{\inputSymbol}
    )_{a \cdot}
$
and similarly for $\softmaxMean_{b}^{\inputSymbol'}$.
The modification of the kernel induced by the $\genericDimenstion^{-1/2}$ scaling, $\WQ, \WK$ initialised independently, and $\softmax$ replaced by the identity function (\Cref{eq:sqrt_scaling_kernel}), then leads to:
\begin{align}\label{eq:sum_pe_sqrt_scale_kernel}
    \kernelf{a b}{}{\inputSymbol}{\inputSymbol'}
    =
    \interpKern \circ
    \kerntildef{a b}{}{\inputSymbol}{\inputSymbol'}
    \!
    \sum_{i, j=1}^{\spatialDimension}
         [
            \interpKern \circ
            \kerntildef{i j}{}{\inputSymbol}{\inputSymbol'}
         ]^2
    .
\end{align}

Several comments are in order.
Firstly, the typical choice of the initialisation covariance for $\posEmb{\sequenceVariable}{\depthSymbol}$ is $\covPosEmb{}{} = \rho I$, $\rho > 0$. This may be reasonable for the
$
    \softmaxMean_{a}^{\inputSymbol}
    =
    \softmax (
        \interpKern \circ
        \kerntildef{}{}{\inputSymbol}{\inputSymbol}
    )_{a \cdot}
$
in \Cref{eq:sum_pe_lin_scale_kernel} when $\softmax$ is the softmax function as it increases attention to the matching input spatial dimension, but does not seem to have any ``attention-like'' interpretation in \Cref{eq:sum_pe_sqrt_scale_kernel} where the effect of applying $\interpKern$ 
to $\kerntilde^{}$ with $\covPosEmb{ }{} = \rho I$
is essentially analogous to that of just adding i.i.d.\ Gaussian noise to each of the attention layer inputs.

Secondly, the right hand side of \Cref{eq:sum_pe_sqrt_scale_kernel} is just a scaled version of the discussed $\interpKern \circ \kerntilde$ kernel,
with the scaling constant disappearing when the attention layer is followed by layer normalization (\Cref{tab:kernel_overview}).
Both of these call into question whether the performance of the corresponding finite NN architectures will translate to its infinite width equivalent.
We address some of these issues next.

\subsubsection{Structured positional encodings}\label{sect:struct_pe}

As mentioned, the main purpose of positional encodings is to inject structural information present in the inputs which would be otherwise ignored by the attention layer.
A natural way to resolve the issues discussed in previous section is thus to try to incorporate similar information directly into the $\covPosEmb{}{}$ covariance matrix.
In particular, we propose
\begin{align}\label{eq:decay_pos_emb}
    \covPosEmb{a b}{}
    =
    \rho
    \begin{cases}
        \exp \{ -\varphi [r_{h}(a, b)^2 + r_{v}(a, b)^2 ] \}
        & \text{(image)} \\
        \exp \{ -\varphi \, r_{s}(a, b)^2 \}
        & \text{(string)}
    \end{cases}
\end{align}
where $\rho, \varphi > 0$ are hyperparameters, $r_h (a, b)$ and $r_v(a, b)$ are the absolute horizontal and vertical distances between the pixels $a$ and $b$ divided by the image width and height respectively, and $r_s(a, b)$ is the absolute distance between the relative position of tokens $a$ and $b$, e.g., if $a$ is the 4\textsuperscript{th} token out of 7 in the first, and $b$ is the 2\textsuperscript{nd} token out of 9 in the second string, then $r_s(a, b) = |\frac{4}{7} - \frac{2}{9}|$.

To motivate the above definition, let us briefly revisit \Cref{eq:sum_pe_lin_scale_kernel}.
Intuitively, the $d^{-1}$ kernel
$
\softmaxMean_{a}^{\inputSymbol}
[
    \interpKern \circ
    \kerntildef{}{}{\inputSymbol}{\inputSymbol'}
]
(\softmaxMean_{b}^{\inputSymbol'})^\top
$
is a result of multiplying the asymptotically Gaussian values $\val{}{} \sim \gp (0, \interpKern \circ \kerntilde)$ by matrices of row-wise stacked $\softmaxMean^\inputSymbol = [ \softmaxMean_1^\inputSymbol ; \ldots ; \softmaxMean_{\spatialDimension}^\inputSymbol]$ vectors, e.g.,
$
\indexedActivation{}{}{\inputSymbol} =   
\softmaxMean^\inputSymbol
\val{}{}(\inputSymbol)
$,\footnote{By standard Gaussian identities, if $Z \sim \gauss (0, \Sigma)$, and $A$ is a deterministic matrix, then $A Z \sim \gauss (0, A \Sigma A^\top).$}
meaning that the $\softmaxMean$
vectors serve the role of attention weights in the infinite width limit.
This in turn implies that the greater the similarity under $\kerntildef{a b}{}{\inputSymbol}{\inputSymbol}$ the higher the attention paid by $a$ to $b$.
Thus, if we want to inject information about the relevance of neighbouring pixels in an image or tokens in a string, we need to increase the corresponding entries of
$
    \interpKern \circ
    \kerntildef{}{}{\inputSymbol}{\inputSymbol}
    =
    \interpKernCoeff \kerntildef{}{}{\inputSymbol}{\inputSymbol'}
    +
    (1 - \interpKernCoeff)
    \covPosEmb{}{}
$
which can be achieved exactly by substituting the $\covPosEmb{}{}$ from \Cref{eq:decay_pos_emb}.

The above reasoning only provides the motivation for modifying the attention weights using positional encodings but not necessarily for modifying the asymptotic distribution of the values $\val{}{}$.
Adding positional encodings only inside the $\softmax$ is not uncommon \citep[e.g.,][]{shaw2018self}, and thus we will also experiment with kernels induced by adding positional encodings
only to the inputs of $\query{\sequenceVariable}{}$ and $\key{\sequenceVariable}{}$, leading to
\begin{align}\label{eq:decay_lin_scale_kernel}
    \kernelf{a b }{}{\inputSymbol}{\inputSymbol'}
    =
    \softmaxMean_{a}^{\inputSymbol}
        \kerntildef{}{}{\inputSymbol}{\inputSymbol'}
    (\softmaxMean_{b}^{\inputSymbol'})^\top
    \, ,
\end{align}
under the $\genericDimenstion^{-1}$ scaling (cf.\ \Cref{eq:sum_pe_lin_scale_kernel}), and
\begin{align*} %
    \kernelf{a b}{}{\inputSymbol}{\inputSymbol'}
    =
    \interpKern \circ
    \kerntildef{a b}{}{\inputSymbol}{\inputSymbol'}
    \!
    \sum_{i, j=1}^{\spatialDimension}
        \kerntildef{i j}{}{\inputSymbol}{\inputSymbol'}
            \interpKern \circ
            \kerntildef{i j}{}{\inputSymbol}{\inputSymbol'}
    \, ,
\end{align*}
under the $\genericDimenstion^{-1/2}$ scaling (cf.\ \Cref{eq:sum_pe_sqrt_scale_kernel}).

Finally, note that the last kernel remains a scaled version of the aforementioned $\interpKern \circ \kerntilde$ kernel, albeit now with $\covPosEmb{}{}$ as in \Cref{eq:decay_pos_emb}.
In our experience, using just $\interpKern \circ \kerntilde$ without the scaling leads to improved empirical performance, and further gains can be obtained with the related kernel
\begin{align}\label{eq:residual_kernel}
    \kernelf{a b }{}{\inputSymbol}{\inputSymbol'}
    =
    \interpKernCoeff \kerntildef{a b}{}{\inputSymbol}{\inputSymbol'}
    +
    (1 - \interpKernCoeff)
    \covPosEmb{a \cdot}{}
    \kerntildef{}{}{\inputSymbol}{\inputSymbol'}
    \covPosEmb{b \cdot}{\top}
    \, .
\end{align}
We call \Cref{eq:residual_kernel} the \emph{residual} attention kernel, as it can be obtained as a limit of architecture with a skip connection, $\indexedActivation{\depthSymbol}{\sequenceVariable}{\inputSymbol} = \sqrt{\interpKernCoeff} \indexedActivity{\depthSymbol - 1}{\sequenceVariable}{\inputSymbol} + \sqrt{1 - \interpKernCoeff} \tildeActivation{\depthSymbol}{\sequenceVariable}{\inputSymbol}$, where $\tildeActivation{\depthSymbol}{\sequenceVariable}{\inputSymbol}$ is output of an attention layer (details in \Cref{sect:residual_attention_appendix}).

\section{Experiments}

We evaluate the attention NNGP/NTK kernels on the CIFAR-10 \citep{cifar10} and IMDb reviews \citep{maas2011learning} datasets.
While IMDb is a more typical setting for attention models (\Cref{sec:imdb}), we included CIFAR-10 experiments (\Cref{sec:cifar}) due to desire to compare with other NNGPs/NTKs on an established benchmark \citep[e.g.,][]{novak2019bayesian,du2019gd,li2019enhanced}, 
and the recent successes of attention on vision tasks \citep[e.g.,][]{wang2017residual,wang2018non,hu2018squeeze,woo2018cbam,chen2018,ramachandran2018stand,bello2019attention}.
Our experimental code utilises the JAX \citep{jax2018github} and Neural Tangents \citep{novak2020neural} libraries.

\begin{figure}[tp]
    \begin{center}
        \centerline{
            \includegraphics[width=\columnwidth,keepaspectratio]{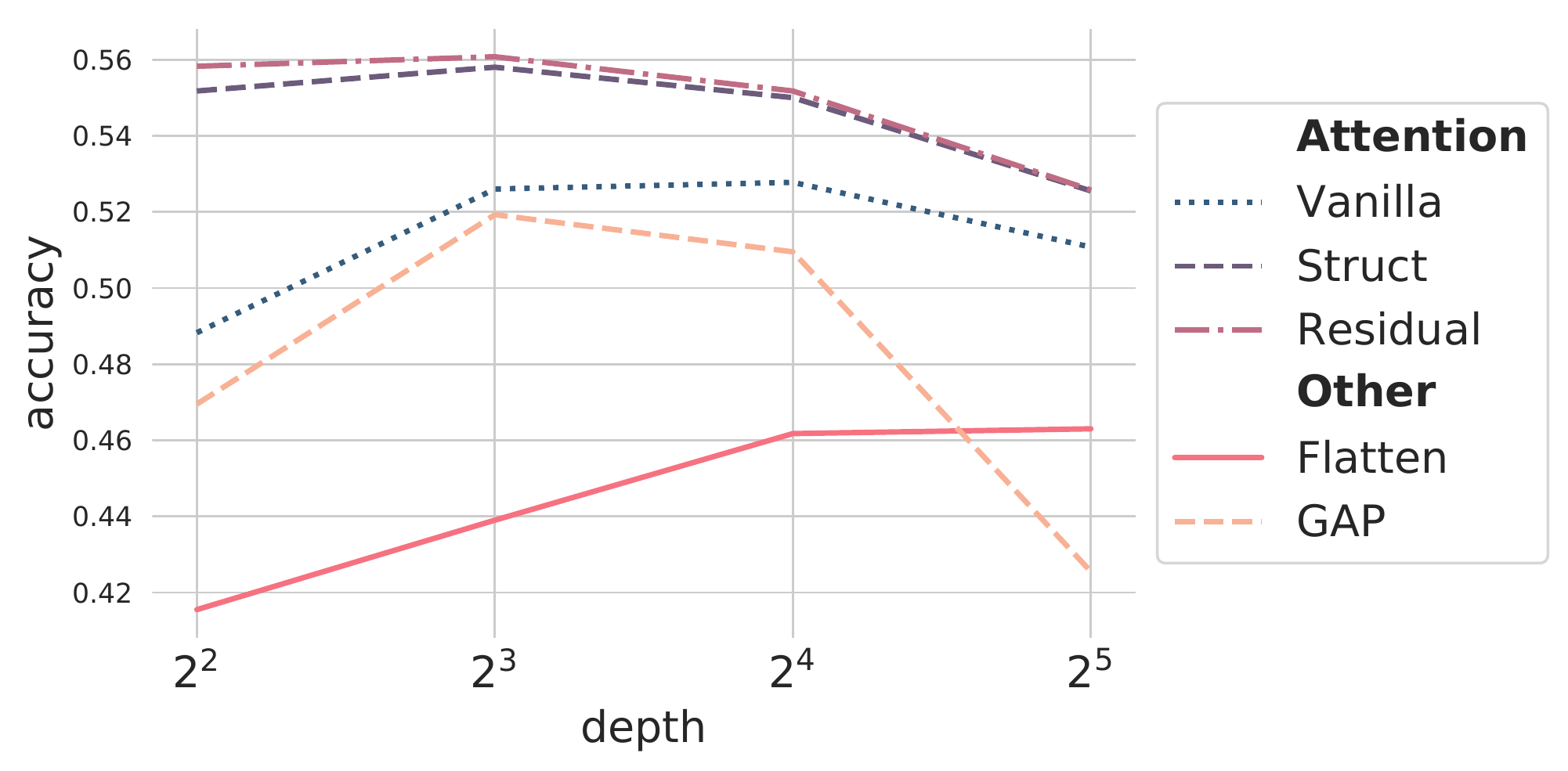}
        }
        \vskip -0.15in
        \caption{
        \textbf{Validation accuracy as a function of depth for various NNGP kernels} on a 2K/4K train/validation split of CIFAR-10 (no pixel downsampling).
        Architecture: \texttt{[depth]}x Convolution + ReLU, followed by a single instance of the kernel specified in the legend (attention kernels combined with additional Flatten), and Dense.
        See \Cref{tab:kernel_overview} for attention, and \citep{novak2019bayesian,garriga2019deep} for Convolutional, Flatten, and Global Average Pooling (GAP) kernel descriptions.
        Results reported for best hyperparameters ($\genericDimenstion^{-1}$ scaling generally resulted in better performance for the \texttt{Struct} kernel).
        More experimental details in \Cref{sect:depth_appendix}.
        Notice the improved performance of attention kernels with positional embeddings and layer normalisation (i.e., \texttt{Struct}, \texttt{Residual}) over their \texttt{Vanilla} counterpart.
        }
        \label{fig:cifar_pe}
    \end{center}
    \vskip -0.25in
\end{figure}

\subsection{CIFAR-10}\label{sec:cifar}

We have run two types of experiments on CIFAR-10: 
(i)~smaller scale experiments focused on understanding how different hyperparameters of the attention kernel affect empirical performance;
(ii)~a larger scale experiment comparing attention kernels to existing NNGP/NTK benchmarks.
The smaller scale experiments were run on a randomly selected subset of six thousand observations from the \emph{training} set, with the 2K/4K train/validation split.
This subset was used in \Cref{fig:convergence_plots,fig:cifar_pe}, and for hyperparameter tuning.
Selected hyperparameters were then employed in the larger scale experiment with the usual 50K/10K train/test split.

All kernels evaluated in this section correspond to NN architectures composed of multiple stacked convolutional layers with ReLU activations, followed by either simple flattening, global average pooling (GAP), or one of our attention kernels itself followed by flattening and, except for the \texttt{Vanilla} attention case (see \Cref{tab:kernel_overview}), also by layer normalisation;
the output is then computed by a single dense layer placed on top.
The choice to use only one attention layer was made to facilitate comparison with \citep{novak2019bayesian,du2019gd,li2019enhanced} where the same set-up with a stack of convolutional layers was considered.
Adding more attention layers did not result in significant gains during hyperparameter search though.
Exact details regarding data normalisation, hyperparameter tuning, and other experimental settings can be found in \Cref{sect:experimental_details}.

\begin{table}[tbp]
\caption{\textbf{CIFAR-10 test accuracies} of attention kernels and existing NNGP/NTK alternatives.
The standard 50K/10K train/test split is used (no pixel downsampling).
Best hyperparemeters from the 2K/4K subset experiments used for each kernel, $\genericDimenstion^{-1}$ scaling for the \texttt{Struct} kernel (see \Cref{tab:kernel_overview}). Details in \Cref{sect:full_cifar_appendix}.}
\label{tab:full_data_results}
\vskip 0.15in
\begin{center}
\begin{small}
\begin{sc}
\noindent
\begin{tabularx}{\columnwidth}{X l c}
    \toprule
    Kernel & NNGP & NTK \\
    \midrule
    Flatten    &  65.54  & 66.27 \\
    GAP \cite{li2019enhanced} & 77.85  & 77.39 \\
    LAP \cite{li2019enhanced} & 80.36 & 79.71 \\
    Struct & 80.55  & 79.93 \\
    Residual & \textbf{80.72}  &  \textbf{80.10} \\
    \bottomrule
\end{tabularx}
\end{sc}
\end{small}
\end{center}
\vskip -0.18in
\end{table}

The most important observations from the smaller scale experiments are captured in \Cref{fig:cifar_pe} which shows the validation accuracy of various NNGP models as a function of kernel choice and number of convolutional layers (\texttt{depth}) preceding the final flatten/GAP/attention plus dense block.
Firstly, notice that except for the \texttt{Flatten} model, all other kernel choices achieve their best performance at smaller depths which is consistent with existing literature \citep{arora2019,li2019enhanced}.

Secondly, observe that both the \texttt{Struct} and \texttt{Residual} attention kernels significantly outperform the \texttt{Vanilla} one, demonstrating that the use of positional embeddings and layer normalisation is helpful even in the infinite width limit as claimed in \Cref{sect:positional_encodings}.
In contrast, we did not find significant evidence for $\softmax(x) = x$ outperforming the standard softmax choice as was the case for finite networks (see \Cref{fig:softmax_replacements}), with the best set of hyperparameters for \texttt{Struct} $\genericDimenstion^{-1}$ with softmax being only marginally better than the best results with the identity function (recall that no $\genericDimenstion^{-1/2}$ kernels use $\softmax = \text{softmax}$ due to the intractability discussed in \Cref{sect:beyon_vanilla_attn}).
This finding provides hope that the $\genericDimenstion^{-1/2}$ kernels also do not sacrifice much in terms of performance by using identity for $\softmax$, but also points to salient differences between the qualitative effects of individual hyperparameter choices in finite and infinite attention layers.

Using the insights from the smaller scale experiments, we ran the larger scale experiment on the full dataset using eight layer models and the \texttt{Struct} and \texttt{Residual} attention kernels.
We used the positional embedding covariance matrix defined in \Cref{eq:decay_pos_emb} in both cases, and $\genericDimenstion^{-1}$ with softmax for the \texttt{Struct} kernel (further details in \Cref{sect:full_cifar_appendix}).
The results can be found in \Cref{tab:full_data_results}.
As you can see, attention performs significantly better than the \texttt{GAP} kernel \citep{arora2019}, and also provides a moderate improvement over the recent local average pooling (\texttt{LAP}) results \citep{li2019enhanced}.
Since we used the validation accuracy from smaller scale experiments to determine our hyperparameters, we are comparing against the best cross-validation results from \citep{li2019enhanced} for fairness.

\begin{table}[t]
\caption{\textbf{IMDb sentiment classification, test accuracies} of simple NNGP/NTK models on the  25K/25K train/test split using GloVe word embeddings (\citet{pennington2014glove}; 840B.300d). \texttt{GAP-only} corresponds to a single global average pooling layer followed by a linear fully connected readout. \texttt{GAP-FCN} has 2 ReLU fully connected layers after GAP. \texttt{Struct} has an attention layer preceding GAP, followed by one (NNGP) or two (NTK) fully connected layers. Models selected on a validation set of 10K reviews. Details in \Cref{sect:imdb_full_appendix}.}
\label{tab:imdb_full_results}
\vskip 0.15in
\begin{center}
\begin{small}
\begin{sc}
\begin{tabularx}{\columnwidth}{X l c}
    \toprule
    Kernel & NNGP & NTK \\
    \midrule
    GAP-only &  \multicolumn{2}{c}{--\,\,\,\,\,84.98\,\,\,\,\,--} \\
    GAP-FCN & 85.82 & 85.80 \\
    Struct & \textbf{86.09} & \textbf{86.09} \\
    \bottomrule
\end{tabularx}
\end{sc}
\end{small}
\end{center}
\vskip -0.18in
\end{table}

\subsection{IMDb reviews}\label{sec:imdb}

Although there has been interest in applying attention in vision, to date it has been predominantly recognized for performance on  language tasks. However, most of available NNGP/NTK kernel implementations  \citep{matthews2018gaussian, lee2018deep, garriga2019deep, arora2019, yang2019v2, li2019enhanced} are hard-coded for the specific experiments performed in the respective paper. Neural Tangents \citep{novak2020neural} allows for some flexibility, yet still accepts only inputs of fixed length and having exactly zero (i.e. inputs to fully connected networks) or two (images for CNNs) spatial dimensions.

We release code allowing use of NNGP/NTK models (with or without attention) on inputs of variable spatial extent and arbitrary dimensionality (e.g., one spatial dimension for texts and time series, three spatial dimensions for videos).
Our implementation seamlessly extends the Neural Tangents library, enabling research and application of NNGP and NTK models to new domains with almost no extra effort.

As an example, we present the first benchmarks of simple NNGP and NTK models on the IMDb sentiment classification dataset in \Cref{tab:imdb_full_results}. We observe that \texttt{Struct} kernels outperform the \texttt{GAP-only} kernel (corresponding to linear regression on the word embeddings mean), but provides marginal benefit compared to a fully connected model on top of the pooling layer (\texttt{GAP-FCN}). We conjecture this is due to high-quality word embeddings  partially incorporating the inductive bias of the considered model. Indeed, we further demonstrate this effect by contrasting the gaps in performance between different kernel families on high- and low-quality word embeddings in \Cref{tab:imdb_small_results}.

\begin{table}[t]
\caption{\textbf{IMDb sentiment classification, test accuracies} on a 3.2K/1.6K train/test split. When high-quality word embeddings are used (300-dimensional GloVe trained on 840B tokens), complex models yield diminishing returns. Contrarily, simple embeddings (50-dimensional GloVe trained on 6B tokens) lead to significant gaps in model performance due to respective inductive biases ($\texttt{GAP-only} < \texttt{GAP-FCN} << \texttt{CNN-GAP} \approx \texttt{Struct}$). Models selected on a validation set of 1.6K reviews. Details in \Cref{sect:imdb_small_appendix}.}\label{tab:imdb_small_results}
\vskip 0.15in
\begin{sc}
\begin{small}
\begin{tabularx}{\columnwidth}{X X c c c}
\toprule
\multicolumn{2}{l}{Embeddings:}     & GloVe 840B      & GloVe 6B        \\
\multicolumn{2}{l}{(dimension)} & (300)             & (50)               \\
\midrule
\multicolumn{2}{l}{GAP only}            & 83.81          & 73.00         \\
\midrule
\multirow{3}{*}{NNGP}      & GAP-FCN        & 83.75          & 74.44         \\
                           & CNN-GAP        & 84.69          & 81.00         \\
                           & Struct     & 83.56          & 80.88         \\
\midrule
\multirow{3}{*}{NTK}       & GAP-FCN        & 83.81          & 74.88          \\
                           & CNN-GAP        & \textbf{84.88} & 80.31          \\
                           & Struct     & 84.00          & \textbf{81.06}  \\
\bottomrule
\end{tabularx}
\end{small}
\end{sc}
\vskip -0.18in
\end{table}

Naturally, our sample IMDb results are not competitive with the state-of-the-art, which achieve up to 97.4\% \citep[Table 4]{thongtan-phienthrakul-2019-sentiment}. However, we hope they will be a useful baseline for future research in infinite width sequence models, and that our codebase will substantially facilitate the process by enabling variable-length, arbitrary-dimensional input processing.

\section{Conclusion}

Unlike under the $\genericDimenstion^{-1}$ scaling of $\querySymbol(\inputSymbol) \keySymbol (\inputSymbol)^\top$ proposed in \citep{yang2019v2}, the standard $\genericDimenstion^{-1/2}$ scaling may lead to non-Gaussian asymptotic behaviour of attention layer outputs.
Gaussianity of the limit can however be obtained by taking the number of heads to infinity.
We explored the effect of positional encodings and replacements for the softmax function in attention layers, leading to improved performance for both finite and infinite attention architectures.
On CIFAR-10, attention improves moderately upon the previous state-of-the-art for GPs without trainable kernels and advanced data preprocessing \citep{li2019enhanced}.
We further released code allowing application of NNGP/NTK kernels to variable-length sequences and demonstrated its use on the IMDb reviews dataset.
While caution is needed in extrapolation of any results, we hope that particularly \Cref{fig:softmax_replacements} and \Cref{tab:full_data_results} inspire novel NN architectures and kernel designs.

\section*{Acknowledgements}
    We thank Jaehoon Lee for frequent discussion, help with scaling up the experiments, and feedback on the manuscript. We thank Prajit Ramachandran for frequent discussion about attention architectures. We thank Greg Yang, Niki Parmar, and Ashish Vaswani, for useful discussion and feedback on the project. Finally, we thank Sam Schoenholz for insightful code reviews.

\bibliography{ref}
\bibliographystyle{icml2020}

\onecolumn

\appendix

\section{Experimental details}\label{sect:experimental_details}

\subsection{CIFAR-10}

The CIFAR-10 datasest \citep{cifar10} was fetched using the TensorFlow datasets\footnote{\url{https://www.tensorflow.org/datasets/catalog/cifar10}}.

In all of the CIFAR-10 experiments, the data was preprocessed by subtracting mean and dividing by a standard deviation for each pixel and data point separately (equivalent to using \texttt{LayerNorm} as the first layer). 
We inflated all of the standard deviations by $10^{-15}$ to avoid division by zero. 

All the classification tasks were converted into regression tasks by encoding the targets as $C$--dimensional vectors, where $C$ is the number of classes, with the entry corresponding to the correct label set to $\frac{C - 1}{C}$ and all other entries to $-\frac{1}{C}$.
This enabled us to perform closed form NNGP and NTK inference using the Gaussian likelihood/MSE loss.

\subsubsection{Hyperparameter search}\label{sect:hyperparam_appendix}

The hyperparameter search was on a fixed architecture with 8x Convolution + ReLU, Attention, Flatten, and a Dense readout layer.
We used $1.7562$ and $0.1841$ respectively for the weight and bias variances as in \citep[appendix G.1]{novak2019bayesian} except for the attention output variance $\outVar$ which was set to one.
The convolutional layers were used with the \texttt{SAME} padding, stride one, and filter size $3 \times 3$.
For attention kernels with positional encodings, the reported $\rho$ parameter (\Cref{eq:decay_pos_emb}) is actually $\rho / (\queryVar \keyVar)$ so that the relative scale of the contribution of $\covPosEmb{}{}$ remains the same with changing $\queryVar \keyVar$.

There were two stages of the hyperparameter search, first to identify the most promising candidates (\Cref{tab:hypers_stage_one}), and second to refine the parameters of these candidate kernels (\Cref{tab:hypers_stage_two}).
The second stage also included the \emph{residual attention kernel} (\Cref{eq:residual_kernel}); the $\interpKernCoeff$ in the second table should thus be interpreted as the one stated in \Cref{eq:residual_kernel} (cf.\ \Cref{sect:residual_attention_appendix}).
The best hyperparameters used in \Cref{fig:cifar_pe} and \Cref{tab:full_data_results} can be found in a bold typeset in \Cref{tab:hypers_stage_two}.

All computation was done in 32-bit precision, and run on up to 8 NVIDIA V100 GPUs with 16Gb of RAM each.

\begin{table}[htbp]
\caption{Hyperparameter values for the first stage of search. 
\textsc{value positional encoding} stands for whether the positional encodings should be added to all $\querySymbol$, $\keySymbol$, and $\valueSymbol$ (\textsc{True}), or only to the inputs of $\querySymbol$ and $\keySymbol$ (\textsc{False}; see \Cref{sect:struct_pe}).
\textsc{encodings covariance} represent whether positional encodings should be added ($0$ for no), and if so, what should their initialisation covariance be ($I$ for identity, and $\covPosEmb{}{}$ for the covariance defined in \Cref{eq:decay_pos_emb}).
$\softmax = \text{softmax}$ was only used when \textsc{query/key scaling} was $\genericDimenstion^{-1}$ (see \Cref{sect:beyon_vanilla_attn}).
$\varphi, \rho, \interpKernCoeff$ were skipped when \textsc{value positional encoding} was \textsc{False}, and $\varphi$ was only varied when \textsc{encodings covariance} was $\covPosEmb{}{}$.}
\label{tab:hypers_stage_one}
\vskip 0.15in
\begin{center}
\begin{small}
\begin{sc}
\begin{tabular}{lc}
    \toprule
    hyperparameter & values \\
    \midrule
    query/key scaling & $\{ \genericDimenstion^{-1/2} , \genericDimenstion^{-1}  \}$ \\
    $\softmax$ (\Cref{sect:softmax_alternatives}) & \{softmax, identity\} \\
    value positional encoding & \{True, False\} \\
    encodings covariance & $\{ 0, I , \covPosEmb{}{} \}$ \\
    $\varphi$ (\Cref{eq:decay_pos_emb}) & $\{ 1, 5 \}$ \\
    $\rho$ (\Cref{eq:decay_pos_emb}) & $\{ 1\}$ \\
    $\interpKernCoeff$ (\Cref{eq:kernel_interp_op}) & $\{ 0.5, 0.8 \}$ \\
    $\queryStd \cdot \keyStd$ & $\{ 0.1, 1.0 \}$ \\
    \bottomrule
\end{tabular}
\end{sc}
\end{small}
\end{center}
\vskip -0.18in
\end{table}

\begin{table}[htbp]
\caption{Hyperparameter values for the first stage of search.
See \Cref{tab:hypers_stage_one} for description of the individual hyperparameters.
The parameters that achieved the best \emph{NNGP} validation accuracy and were selected for the subsequent experiments are in a bold typeset.}
\label{tab:hypers_stage_two}
\vskip 0.15in
\begin{center}
\begin{small}
\begin{sc}
\begin{tabular}{lcc}
    \toprule
    hyperparameter & struct & residual \\
    \midrule
    query/key scaling & $\{ \boldsymbol{\genericDimenstion^{-1}}  \}$ & $\{ \boldsymbol{\genericDimenstion^{-1}}  \}$ \\
    $\softmax$ (\Cref{sect:softmax_alternatives}) & \{\textbf{softmax}\} & \{\textbf{identity}\} \\
    value positional encoding & \{True, \textbf{False}\} & -- \\
    encodings covariance & $\{ \boldsymbol{\covPosEmb{}{}} \}$ & $\{ \boldsymbol{\covPosEmb{}{}} \}$ \\
    $\varphi$ (\Cref{eq:decay_pos_emb}) & $\{ 1, \boldsymbol{5}, 10 \}$ & $\{ 1, \boldsymbol{5}, 10 \}$ \\
    $\interpKernCoeff$ (\Cref{eq:kernel_interp_op,eq:residual_kernel}) & $\{ \boldsymbol{0.4}, 0.5, 0.65, 0.8, 0.9 \}$ & $\{ 0.4, 0.5, \boldsymbol{0.65}, 0.8, 0.9 \}$ \\
    $\rho$ (\Cref{eq:decay_pos_emb}) & $\{ 0.5, 1, \boldsymbol{1.5} \}$ & $\{ \boldsymbol{0.5} , 1 , 1.5 \}$ \\
    $\queryStd \cdot \keyStd$ & $\{ 0.001, \boldsymbol{0.1}, 1.0 \}$ & -- \\
    \bottomrule
\end{tabular}
\end{sc}
\end{small}
\end{center}
\vskip -0.18in
\end{table}

\subsubsection{Details for \Cref{fig:convergence_plots}}\label{sect:convergence_appendix}

The downsampling was performed using \texttt{skimage.transform.resize} with parameters \texttt{mode="reflect"} and \texttt{anti\_aliasing=True}, using downsampled height and width of size 8 as mentioned.

Both the convergence and accuracy plots are for the $d^{-1/2}$ vanilla NNGP kernel with $\softmax = \text{softmax}$.
The intractable softmax integral of the limiting covariance function was estimated using MC integration with 2048 samples.

We used $1.7562$ and $0.1841$ respectively for the weight and bias variances as in \citep[appendix G.1]{novak2019bayesian} for all the convolutional and dense layers, $1.7562$ for the $\keyVar, \queryVar$ and $\valueVar$, and $\outVar = 1$.
The convolutional layer used \texttt{VALID} paddingstride one, and filter size $3 \times 3$.

As in \citep{novak2019bayesian}, The reported distance between kernel matrices is the logarithm of
\begin{align}
    \frac{
        \|
            \hat{\mathcal{K}} - \mathcal{K}
        \|_F^2
    }{
        \|
            \mathcal{K}
        \|_F^2
    }
    \, ,
\end{align}
where $\hat{\mathcal{K}}$ and $\mathcal{K}$ are respectively the empirical and the predicted theoretical covariance matrices for the training set.

All computation was done in 32-bit precision, and run on up to 8 NVIDIA V100 GPUs with 16Gb of RAM each.

\subsubsection{Details for \Cref{fig:softmax_replacements}}\label{sect:replacements_appendix}

We used a 45K/5K train/validation split of the usual 50K CIFAR-10 training set and reported the validation set accuracy after training for 1000 epochs with batch size 64 and the Adam optimiser.

The attention layers used the usual $\genericDimenstion^{-1/2}$ scaling of the query/key inner products, and the convolutional layers used the \texttt{SAME} padding, stride one, and filter size $3 \times 3$.
We used $2.0$ and $10^{-2}$ respectively for the weight and bias variances except in the attention where $\queryVar = \keyVar = \valueVar = 2$ but $\outVar = 1$.
Further, we used the \texttt{append} type positional encodings (\Cref{sect:positional_encodings}) with the same embedding dimension as \textsc{n\_channels} (\Cref{tab:softmax_replacements_hypers}), thus doubling the embedding dimension of the attention layer inputs.

All computation was done in 32-bit precision, and run on a single NVIDIA V100 GPU with 16Gb of RAM each.

\begin{table}[htbp]
\caption{Hyperparameter values for which results are reported in \Cref{fig:softmax_replacements}.
\textsc{n\_channels} is the number of channels used in the convolutional layers.
The same number was used for $\logitDimension$ and output dimension in the attention layer, but $\headDimension = \valueDimension = \floor{\textsc{n\_channels}}$ to reduce the memory footprint.
The learning rate was fixed throughout the training, relying only on Adam to adapt step size.
Each configuration was run with three random seeds and each of the corresponding results was included in the appropriate column in \Cref{fig:softmax_replacements}.}
\label{tab:softmax_replacements_hypers}
\vskip 0.15in
\begin{center}
\begin{small}
\begin{sc}
\begin{tabular}{lc}
    \toprule
    hyperparameter & values \\
    \midrule
    $\zeta$ (attention) & \{relu, abs, softmax\} \\
    LayerNorm & \{none, per\_head , at\_output\} \\
    \midrule
    n\_channels   &  $\{ 32, 192 \}$  \\
    learning rate  &  $\{ 10^{-3}, 10^{-2} \}$ \\
    \bottomrule
\end{tabular}
\end{sc}
\end{small}
\end{center}
\vskip -0.18in
\end{table}

\subsubsection{Details for \Cref{fig:cifar_pe}}\label{sect:depth_appendix}

We used $1.7562$ and $0.1841$ respectively for the weight and bias variances as in \citep[appendix G.1]{novak2019bayesian} except for the attention output variance $\outVar$ which was set to one.
The convolutional layers were used with the \texttt{SAME} padding, stride one, and filter size $3 \times 3$.
For the vanilla attention kernels, we report the best performance over $\queryStd \keyStd = \{ 10^{-3}, 10^{-1}, 1, 2, 10 \}$ at each depth.
The \texttt{Struct} and \texttt{Residual} were used with the best hyperparameters found during hyperparameter search as reported in \Cref{sect:hyperparam_appendix}.

All computation was done in 32-bit precision, and run on up to 8 NVIDIA V100 GPUs with 16Gb of RAM each.

\subsubsection{Details for \Cref{tab:full_data_results}}\label{sect:full_cifar_appendix}

The best set-up from \Cref{sect:hyperparam_appendix} was used (including the best hyperparameters as stated in \Cref{tab:hypers_stage_two}).

All computation was done in 64-bit precision, and run on up to 8 NVIDIA V100 GPUs with 16Gb of RAM each.

\subsection{IMDb}

\subsubsection{General settings for \Cref{tab:imdb_full_results} and \Cref{tab:imdb_small_results}.}

    The IMDb reviews dataset \citep{maas2011learning} was fetched using TensorFlow datasets\footnote{\url{https://www.tensorflow.org/datasets/catalog/imdb\_reviews}}.

    All sentences were truncated or padded to 1000 tokens using the default settings of \texttt{tf.keras.preprocessing.text.Tokenizer}\footnote{\url{https://www.tensorflow.org/api_docs/python/tf/keras/preprocessing/text/Tokenizer}}. No words were removed from the embedding model dictionary. Tokens were embedded using GloVe embeddings \citep{pennington2014glove} with no other pre-processing. Binary targets were mapped to $\left\{-0.5, 0.5\right\}$ values. Diagonal regularizers for inference were selected based on validation performance among the values of $10^{-7},10^{-6},\dots,1$ multiplied by the mean trace of the kernel.
    
    When applicable, all models used ReLU nonlinearities, \texttt{Struct} (Structured positional encoding, $\genericDimenstion^{-1}$ scaling, \Cref{tab:kernel_overview}) kernel with $\softmax$ being the row-wise softmax function (\Cref{eq:projection_def}), decaying positional embeddings used only for the attention keys and queries, with $\varphi=2.5$ (\Cref{eq:decay_pos_emb}), $\alpha = 0.75$, and $\rho = 1$ (\Cref{eq:kernel_interp_op}). These parameters were selected based on preliminary experiments with CIFAR-10, and fine-tuning on IMDb specifically is an interesting avenue for future research.
    
    All preliminary and validation experiments were carried out in 32-bit precision, while test evaluation (reported in the \Cref{tab:imdb_full_results} and \Cref{tab:imdb_small_results}) were done in 64-bit precision. All experiments were run on machines with up to 8 NVIDIA V100 GPUs with 16Gb of RAM each.

\subsubsection{Details for \Cref{tab:imdb_full_results}}\label{sect:imdb_full_appendix}
    Words were embedded using GloVe 840B.300d embeddings.

    The embedding model was selected on a small-scale experiment (4000 train and 4000 validation sets) among GloVe 6B 50-, 100-, 200-, and 300-dimensional variants, as well as GloVe 840B.300d, and 1024-dimensional ELMO \citep{Peters:2018} embeddings (using TensorFlow Hub\footnote{\url{https://tfhub.dev/google/elmo/3}}). In this preliminary experiment, GloVe 840B.300d, GloVe6B.300d, and ELMO.1024d performed similarly, and GloVe 840B.300d was chosen for the full dataset experiment.
    
    The validation experiment was run on the 25K training set partitioned into a 15K and 10K training and validation sets, with the best models then evaluated on the 25K training and 25K test sets.\footnote{Precisely, subsets of sizes 14880/9920 and 24960/24960 were used to make the dataset be divisible by 8 (the number of GPUs) times 20 (the batch size), which is a technical limitation of the Neural Tangents \citep{novak2020neural} library.}
    
    All layers used weight and bias variances $2$ and $0.01$ respectively, expect for attention outputs and values variances which were set to $1$, and the top linear readout layer with weight variance 1 and no bias.
    
    Three classes of models were considered:
    \begin{enumerate}
        \item \texttt{GAP-only}, doing only global average pooling over inputs followed by the linear readout.
        \item \texttt{GAP-FCN}, in which GAP was followed by 0, 1, or 2 fully connected layers.
        \item \texttt{Struct}, allowing the same models as \texttt{GAP-FCN}, except for necessarily having an attention layer before GAP.
    \end{enumerate}
    Each class could also have an optional \texttt{LayerNorm} layer following GAP. The best model from each class was then evaluated on the test set.

\subsubsection{Details for \Cref{tab:imdb_small_results}}\label{sect:imdb_small_appendix}
    All convolutional layers used the total window (context) size of 9 tokens, stride 1, and \texttt{SAME} (zero) padding.
    
    Experiments were run on a 3200/1600/1600 train/validation/test splits. Four classes of models were considered:
    \begin{enumerate}
        \item \texttt{GAP-only}, identical to the one in \Cref{sect:imdb_full_appendix}.
        \item \texttt{GAP-FCN}, also identical to the one in \Cref{sect:imdb_full_appendix}.
        \item \texttt{CNN-GAP}, allowing the same models as in \texttt{GAP-FCN}, but having GAP preceeded by 0, 1, 2, 4, or 8 CNN layers.
        \item \texttt{Struct}, allowing the same models as in \texttt{CNN-GAP}, but having 1 or 2 attention layers (each optionally followed by \texttt{LayerNorm} over channels) before GAP. If the model also had CNN layers, attention and CNN layers were interleaved, attention layers being located closer to GAP (for example, a model with 8 CNN layers and 2 attention layers would have 7 CNN layers followed by attention, CNN, attention, GAP).
    \end{enumerate}
    All models were allowed to have either ReLU or Erf nonlinearity, with weight and bias variances set to 2 and 0.01 for ReLU, and 1.7562 and 0.1841 for Erf, with the same values used by attention keys and queries layers, but having variance 1 for values and output layers. The readout linear layer had weight variance 1 and no bias.

\begin{table}[htbp]
\caption{Best validation accuracy for various finite attention architectures. The reported numbers are an average over three random seeds.}
\label{tab:finite_attention_best}
\vskip 0.15in
\begin{center}
\begin{small}
\begin{sc}
\begin{tabular}{lccc}
    \toprule
     & softmax & relu & identity \\
    \midrule
    none    &  64.10  & 69.68  & 70.46 \\
    per\_head   &  68.96 & 77.40  & 75.28   \\
    at\_output   & 71.70 & 79.00  & 79.56 \\
    \bottomrule
\end{tabular}
\end{sc}
\end{small}
\end{center}
\vskip -0.18in
\end{table}

\section{Proofs}\label{sect:proofs}

\textbf{Assumptions:} We assume the~input set $\indexSet \subset \R{\natnum \times \inputDimension}$ is \emph{countable}, and
the~usual \emph{Borel product $\sigma$-algebra} on any of the involved countable real spaces (inputs, weights, outputs of intermediary layers).
We also assume that the nonlinearities $\nonlinearity$ and $\softmax$ are \emph{continuous} and (entrywise) \emph{polynomially bounded}, i.e.,
$|\nonlinearity(z)| \leq \sum_{t = 0}^{m} c_t |z|^t$ for some $m \in \natnum$ and $c_0, \ldots, c_m \in \R{}_+$ independent of $z$,\footnote{This is a relaxation of the original `linear envelope' condition $|\nonlinearity(z)| \leq c + m |z|$ for some $c, m \in \R{}_+$, used in \citep{matthews2018gaussian,garriga2019deep} and stated in \Cref{thm:gp_convergence_sqrt}.
We decided to keep the reference to the linear envelope condition in the main text since it is general enough to guarantee convergence for all bounded (e.g., softmax, tanh) and ReLU like (e.g., ReLU, Leaky ReLU, SeLU) nonlinearities, and matches the existing literature with which the readers may already be familiar.
Nevertheless, all the presented proofs are valid for the polynomially bounded nonlinearities, similarly to \citep{yang2019v2}.} and $|\softmax(\logitSymbol)_{a i}| \leq \sum_{t = 0}^m c_0 |\logitSymbol_{a i}|^t$ for some $m \in \natnum$ and $c_0, \ldots , c_M \in \R{}_+$ independent of $\logitSymbol$.
For the NTK proofs, we further assume that $\nabla \nonlinearity$ and $\nabla \softmax$ are continuous bounded almost everywhere, where for ReLU, Leaky ReLU, or similar, we set $\nabla \nonlinearity(0) \coloneqq \lim_{z \to 0^-} \nabla \nonlinearity(z)$ which for ReLU/Leaky ReLU is equal to zero.

As \citet{matthews2018gaussian}, we will need to use the~`infinite width, finite fan-out' construction of the~sequence of NNs.
In particular, we will assume that for any attention layer $\depthSymbol \in [\depth + 1]$ and $\sequenceVariable \in \natnum$, the~output is computed as defined in~\Cref{eq:attention_out}, but we will add a~countably infinite number of additional heads which do not affect the output of the $\sequenceVariable$\textsuperscript{th} network, but are used by wider networks,
i.e., each head $h > \headDimension$ is only used to compute the~outputs by networks with index $m \in \natnum$ such that $\genericDimenstion_{m}^{\depthSymbol,\headSymbol} \geq h$.
Similar construction can be used for fully connected, convolutional, and other types of layers as demonstrated in~\citep{matthews2018gaussian,garriga2019deep}.
Since the~outputs remain unchanged, a~proof of convergence of the~`infinite width, finite fan-out networks' implies convergence of the~standard finite width networks, and thus the~construction should be viewed only as an~analytical tool which will allow us to treat all the~random variables
\begin{equation*}
    \{
        \indexedActivation{\depthSymbol}{\sequenceVariable, ij}{\inputSymbol}, \indexedActivation{\depthSymbol h}{\sequenceVariable, ij}{\inputSymbol}
        \colon
        n, h, i, j \in \natnum, \depthSymbol \in [\depth + 1], x \in \indexSet
    \}
    \, ,
\end{equation*}
as defined on the~same probability space, and thus allows us to make claims about convergence in probability and similar.

Finally, we will be using the \emph{NTK parametrisation} \citep{jacot2018ntk} within the NTK convergence proofs, i.e., we implicitly treat each weight $\W_{ij} \sim \gauss(0, \sigma^2 / \genericDimenstion)$, i.i.d., as $\W = \frac{\sigma}{\sqrt{\genericDimenstion}} \uW$ where only $\uW$ is trainable.
This parametrisation ensures that not only the forward but also the backward pass are properly normalised; under certain conditions, proofs for NTK parametrisation can be extended to standard parametrisation \citep{lee2019wide}.

\textbf{Notation:} For random variables $X, (X_n)_{n \geq 1}$, $X_n \convergeDist X$ denotes convergence in distribution, and $X_n \convergeProb X$ convergence in probability. 
For vectors $x, y \in \R{m}$, $\langle x, y \rangle = \sum_{j=1}^m x_j y_j$ denotes the~usual inner product, and for matrices $A, B \in \R{m \times m}$, $\langle A , B \rangle = \langle \vectorise(A) , \vectorise(B) \rangle = \sum_{i, j = 1}^{m} A_{i j} B_{i j}$ denotes the~Frobenius inner product.
For any $A \in \R{m \times k}$, we will use $A_{i \cdot} \in \R{1 \times k}$ and $A_{\cdot j} \in \R{m \times 1}$ to respectively denote $i$\textsuperscript{th} row and $j$\textsuperscript{th} column of the~matrix.
Later on, we will be working with finite subsets $\projectionIndeces \subset \indexSet \times \natnum$ for which we define the~coordinate projections
\begin{align*}
    \projectionIndeces_{\indexSet} 
    \coloneqq 
    \{ x \in \indexSet \colon \exists i \in \natnum \text{ s.t. } (x,  i) \in \projectionIndeces \}
    \, , \qquad
    \projectionIndeces_{\natnum} 
    \coloneqq 
    \{ i \in \natnum \colon \exists x \in \indexSet \text{ s.t. } (x,  i) \in \projectionIndeces \}
    \, .
\end{align*}
Since \citet{yang2019v2} provides convergence for attention architectures under the $\genericDimenstion^{-1}$ only in the NNGP regime, we use $\scaling \in \{ 1, \frac{1}{2} \}$ to refer to the different $\genericDimenstion^{-\scaling}$ within the NTK proofs.
As explained in \Cref{sect:linear_scaling_limit}, the $\scaling = 1$ limit is not very interesting when $\WQ$ and $\WK$ are initialised independently with zero mean, and thus we will be assuming $\WQ = \WK$ a.s.\ whenever $\scaling = 1$. 
Finally, we use $\OVStd \coloneqq \outStd \valueStd$, $\QKStd \coloneqq \queryStd \keyStd$, $\lesssim$ as `less then up to a universal constant', $\poly(x_1, \ldots , x_M)$ for a polynomial in $x_1, \ldots, x_m \in \R{}$, and the shorthand
\begin{align}
    \tildeLogit{\sequenceVariable, a i}{\depthSymbol \headIndex} (\inputSymbol)
    \coloneqq
    \softmax(\logitSymbol_{\sequenceVariable, a i}^{\depthSymbol \headIndex}(\inputSymbol))
    \, .
\end{align}

\textbf{Proof technique:} The now common way of establishing convergence of various deep NNs architectures is to inductively prove that whenever a preceding layer's outputs converge in distribution to a GP, the outputs of the subsequent layer converge to a GP too under the same assumptions on the nonlinearities and initialisation \citep[e.g.,][]{matthews2018gaussian,lee2018deep,novak2019bayesian,garriga2019deep,yang2019v1,yang2019v2}.
We prove this induction step for NNGP under the $\genericDimenstion^{-1/2}$ scaling in \Cref{thm:gp_convergence_sqrt} (recall that the equivalent result under the $\genericDimenstion^{-1}$ is already known due to \citet{yang2019v2}), and for NTK in \Cref{thm:ntk_convergence}.
As in \citep{matthews2018gaussian}, our technique is based on exchangeability (\Cref{lem:exchangeability}), and we repeatedly make use of \Cref{thm:mean_convergence} which says that if a sequence of real valued random variables $(\genericRV_\sequenceVariable)_{\sequenceVariable \geq 1}$ converges in distribution to $\genericRV$, and the $(X_\sequenceVariable)_{\sequenceVariable \geq 1}$ are uniformly integrable (\Cref{def:ui} below), then $X$ is integrable and $\E [ \genericRV_{\sequenceVariable} ] \to \E [ \genericRV ]$.

\begin{lemma}[Exchangeability]\label{lem:exchangeability}
    For any $\sequenceVariable \in \natnum$, the outputs of an attention layer $\indexedActivation{\depthSymbol}{\sequenceVariable, a i}{\inputSymbol}$ are exchangeable along the $i$ index.
    Furthermore, each of $\indexedActivation{\depthSymbol\headIndex}{\sequenceVariable}{\inputSymbol}, \logitSymbol_{\sequenceVariable}^{\depthSymbol \headIndex}(\inputSymbol), \queryN(\inputSymbol), \keyN(\inputSymbol), \valueN(\inputSymbol)$ is exchangeable over the $\headIndex$ index, and for a fixed $\headIndex$, each of $\indexedActivation{\depthSymbol\headIndex}{\sequenceVariable, a i}{\inputSymbol}, \query{\sequenceVariable, a i}{\depthSymbol \headIndex}(\inputSymbol), \key{\sequenceVariable, a i}{\depthSymbol\headIndex}(\inputSymbol), \val{\sequenceVariable, a i}{\depthSymbol \headIndex}(\inputSymbol)$ is exchangeable over the $i$ index.
\end{lemma}

\begin{definition}[Uniform integrability]\label{def:ui}
    A collection of real valued random variables $\mathcal{C}$ is called uniformly integrable if for any $\varepsilon > 0$ there exists $c_{\varepsilon} \geq 0$ s.t.\ $\E |\genericRV| \indicator{|\genericRV| \geq c_{\varepsilon}} \leq \varepsilon$ for all $\genericRV \in \mathcal{C}$ simultaneously.
\end{definition}

\begin{proof}[Proof of \Cref{lem:exchangeability}]
    Recall that by the de~Finnetti's theorem, it is sufficient to exhibit a set of random variables conditioned on which the set of random variables becomes i.i.d.
    This is is trivial for the columns of $\indexedActivation{\depthSymbol}{\sequenceVariable, a i}{\inputSymbol}$ as we can simply condition on $\{ \indexedActivation{\depthSymbol\headIndex}{\sequenceVariable, a i}{\inputSymbol} \colon \headIndex \in [\headDimension] \}$.
    The remainder of the claims can be obtained by observing that
    \begin{equation*}
        \indexedActivation{\depthSymbol \headIndex}{\sequenceVariable}{\inputSymbol}
        =
        \softmax \biggl(
            \frac{1}{\sqrt{\logitDimension}}
            \indexedActivity{\depthSymbol - 1}{\sequenceVariable}{\inputSymbol}
            \weightsQ
            (\indexedActivity{\depthSymbol - 1}{\sequenceVariable}{\inputSymbol} \weightsK)^\top
        \biggr)
        \indexedActivity{\depthSymbol - 1}{\sequenceVariable}{\inputSymbol}
        \weightsV
    \, ,
    \end{equation*}
    and thus if we condition on $\indexedActivity{\depthSymbol - 1}{\sequenceVariable}{\inputSymbol}$, the variables associated with individual heads are i.i.d.
\end{proof}

\subsection{$\genericDimenstion^{-1/2}$ NNGP convergence proof}\label{sect:nngp_proofs}

\nngpConvergence*

\begin{proof}%
Since we have assumed that the~input set $\indexSet$ is countable, we can use \Cref{lem:fin_dim_marg} to see that all that we need to do to prove \Cref{thm:gp_convergence_sqrt} is to show that every finite dimensional marginal of $\activationsN$ converges to the~corresponding Gaussian limit.
Because the~finite coordinate projections are continuous by definition of the~product topology, the~continuous mapping theorem \citep[theorem~9.3.7]{dudley02} tells us it is sufficient to prove convergence of the~finite dimensional marginals of
\begin{equation}\label{eq:def_channels}
    \{
        \indexedActivation{\depthSymbol}{\sequenceVariable, \cdot j}{\inputSymbol} \colon \inputSymbol \in \indexSet, j \in \natnum
    \}
    \, ,
\end{equation}
as any finite dimensional marginal of $\activationsN$ can be obtained by a~finite coordinate projection.

Focusing on an~arbitrary finite marginal $\projectionIndeces \subset \indexSet \times \natnum$, we follow \citeauthor{matthews2018gaussian}\ and use the~Cram{\' e}r-Wold device \citep[p.~383]{billingsey86} to reduce the~problem to that of establishing convergence of
\begin{equation}\label{eq:projection_def}
    \projectionN
    \coloneqq
    \sum_{(\inputSymbol, i) \in \projectionIndeces}
        \langle 
            \projectionCoefficients^{\inputSymbol, i} , 
            \indexedActivation{\depthSymbol}{\sequenceVariable, \cdot i}{\inputSymbol}
        \rangle
    \, ,
\end{equation}
for any choice of $\{ \projectionCoefficients^{(\inputSymbol, i)} \in \R{\spatialDimension} \colon (x, i) \in \projectionIndeces \} \subset \R{\spatialDimension \times \projectionIndeces}$.
We can rewrite $\projectionN$ as
\begin{align*}
    \projectionN
    &=
    \sum_{\inputSymbol, i \in \projectionIndeces}
        \langle 
            \projectionCoefficients^{\inputSymbol, i} , 
            \indexedActivation{\depthSymbol}{\sequenceVariable, \cdot i}{\inputSymbol}
        \rangle
    =
    \sum_{\inputSymbol, i \in \projectionIndeces}
        \bigl\langle 
            \projectionCoefficients^{\inputSymbol, i} , 
            \bigl[
                \indexedActivation{\depthSymbol 1}{\sequenceVariable}{\inputSymbol},
                \ldots,
                \indexedActivation{\depthSymbol \headDimension}{\sequenceVariable}{\inputSymbol}
            \bigr]
            \weightO{\cdot i}
        \bigr\rangle
    \\
    &=
    \frac{1}{\sqrt{\headDimension}}
    \sum_{\headIndex = 1}^{\headDimension}
    \sum_{(\inputSymbol, i)}
        \bigl\langle
            \projectionCoefficients^{\inputSymbol, i} , 
            \sqrt{\headDimension}
            \indexedActivation{\depthSymbol \headIndex}{\sequenceVariable}{\inputSymbol}
            \weightOGen{\sequenceVariable, \cdot i}{\depthSymbol \headIndex, \outputSymbol}
        \bigr\rangle
    \eqqcolon
    \frac{1}{\sqrt{\headDimension}}
    \sum_{\headIndex = 1}^{\headDimension}
        \summandN
    \, ,
\end{align*}
where we have defined $\weightOGen{\sequenceVariable, \cdot i}{\depthSymbol \headIndex, \outputSymbol} \coloneqq [\weightOGen{\sequenceVariable, (\headIndex \layerDimensionN + 1) i}{\depthSymbol, \outputSymbol}, \ldots, \weightOGen{\sequenceVariable, (\headIndex \layerDimensionN + \layerDimensionN - 1) i}{\depthSymbol, \outputSymbol}] \subset \R{\layerDimensionN}$.

We are now prepared to apply lemma~10 from \citep{matthews2018gaussian} which we restate (with minor modifications) here. 

\begin{lemma}[Adaptation of theorem~2 from \citep{Blum1958}]\label{lemma:eCLT}
For each $\rowIndex \in \natnum$, let $\{ \genericRV_{\rowIndex, \colIndex} \colon \colIndex = 1,2, \ldots \}$ be an infinitely exchangeable sequence with $\E \genericRV_{\rowIndex, 1} = 0$ and $\E \genericRV_{\rowIndex, 1}^2 = \rowVariance$, such that $\lim_{\rowIndex \to \infty} \rowVariance = \limitVariance$ for some $\limitVariance \geq 0$.
Let 
\begin{equation}
    \generalSum_{\rowIndex} 
    \coloneqq
    \frac{1}{\sqrt{\genericDimenstion_\rowIndex}}
    \sum_{\colIndex=1}^{\genericDimenstion_\rowIndex} 
        \genericRV_{\rowIndex, \colIndex} 
    \, , 
\end{equation}
for some sequence $(\genericDimenstion_\rowIndex)_{\rowIndex \geq 1} \subset \natnum$ s.t.\ $\lim_{\rowIndex \to \infty} \genericDimenstion_\rowIndex = \infty$.
Assume:

\begin{enumerate}[\hspace{1em}(a)]
    \item $\E{\genericRV_{\rowIndex, 1} \genericRV_{\rowIndex, 2}} = 0 $
    \item $ \lim_{\rowIndex \to \infty }\E{\genericRV_{\rowIndex, 1}^{2} \genericRV_{\rowIndex, 2}^{2}}  = \limitStd^{4} $
    \item $ \E{|\genericRV_{\rowIndex, 1}|^{3}} = \littleO(\sqrt{\genericDimenstion_\rowIndex}) $
\end{enumerate}

Then $\generalSum_{\rowIndex} \convergeDist Z$, where $Z = 0$ (a.s.) if $\limitVariance = 0$, and $Z \sim \gauss (0, \limitVariance)$ otherwise.
\end{lemma}

Substituting $\generalSum_{\rowIndex} = \projectionN$ and $\genericRV_{\rowIndex, \headIndex} = \summandN$, convergence of $\projectionN$ follows from \Cref{lemma:eCLT}:
\begin{itemize}
    \item Exchangeability requirement is satisfied by \Cref{lem:head_exchangeability}.
    \item Zero mean and covariance follow from \Cref{lem:head_mean_corr_zero}.
    \item Convergence of variance is established in \Cref{lem:head_var_convergence}.
    \item Convergence of $\E \lbrack \summandLogit{\sequenceVariable , 1}^2 \summandLogit{\sequenceVariable , 2}^2 \rbrack$ and $\E |\summandN|^3 = \littleO(\sqrt{\layerDimensionN})$ are implied by \Cref{lem:head_all_sqmoments_converge}.
\end{itemize}
Combining the~above with \Cref{lem:inner_prod_converge,lem:logit_dist_convergence} concludes the~proof.
\end{proof}

\vspace{0.5\baselineskip}
\begin{lemma}[Infinite exchangeability]\label{lem:head_exchangeability}
    $\summandN$ are exchangeable over the~index $\headIndex$.
\end{lemma}

\begin{proof}
    Recall that by the~de~Finetti's theorem, it is sufficient to exhibit a~set of random variables conditioned on which the~$\{ \summandN \colon \headIndex \in \natnum \}$ are i.i.d.
    From \Cref{sect:background}, we have 
    \begin{equation*}
        \indexedActivation{\depthSymbol \headIndex}{\sequenceVariable}{\inputSymbol}
        =
        \softmax \biggl(
            \frac{1}{\sqrt{\logitDimension}}
            \indexedActivity{\depthSymbol - 1}{\sequenceVariable}{\inputSymbol}
            \weightsQ
            (\indexedActivity{\depthSymbol - 1}{\sequenceVariable}{\inputSymbol} \weightsK)^\top
        \biggr)
        \indexedActivity{\depthSymbol - 1}{\sequenceVariable}{\inputSymbol}
        \weightsV
    \, .
    \end{equation*}
    Hence if we condition on $\{ \indexedActivity{\depthSymbol - 1}{\sequenceVariable, \cdot j}{\inputSymbol} \colon j \in [\layerDimension{\depthSymbol - 1}], x \in \projectionIndeces_{\indexSet} \}$, where $\projectionIndeces_{\indexSet} \coloneqq \{ x \in \indexSet \colon \exists i \in \natnum \text{ s.t. } (x,  i) \in \projectionIndeces \}$, it is easy to see that $\{ \summandN \}_{\headIndex \geq 1}$ are exchangeable.   
\end{proof}

\begin{lemma}[Zero mean and covariance]\label{lem:head_mean_corr_zero}
    $\E \summand_{\sequenceVariable , 1} = \E \summand_{\sequenceVariable , 1}  \summand_{\sequenceVariable , 2} = 0$.
\end{lemma}

\begin{proof}
    Using $\E \weightOGen{\sequenceVariable, \cdot i}{\depthSymbol 1, \outputSymbol} = 0$, $\E \summand_{\sequenceVariable , 1} = 0$ if $|\! \E \indexedActivation{\depthSymbol 1}{\sequenceVariable, ij}{\inputSymbol} | < \infty$ for all $(i, j) \in [\spatialDimension] \times \natnum$.
    Substituting $\E \indexedActivation{\depthSymbol 1}{\sequenceVariable, ij}{\inputSymbol} = \E \softmax(\logitSymbol_{\sequenceVariable}^{\depthSymbol 1})_{i \cdot} \indexedActivity{\depthSymbol - 1}{\sequenceVariable}{\inputSymbol} \weightV{\cdot j} = 0$ as long as $|\! \E \softmax(\logitSymbol_{\sequenceVariable}^{\depthSymbol 1})_{i a} \indexedActivity{\depthSymbol - 1}{\sequenceVariable, b k}{\inputSymbol}| < \infty$ for any $a, b, k \in [\spatialDimension]$.
    This can be obtained by combining H{\" o}lder's inequality with \Cref{lem:mmnt_propagation}.
    An~analogous argument applies for $\E \summand_{\sequenceVariable , 1}  \summand_{\sequenceVariable , 2} = 0$ since $\E [ \weightOGen{\sequenceVariable, \cdot i}{\depthSymbol 1, \outputSymbol} (\weightOGen{\sequenceVariable, \cdot i}{\depthSymbol 2, \outputSymbol})^\top] = 0$ by assumption.
\end{proof}

\begin{lemma}[Convergence of variance]\label{lem:head_var_convergence}
    $\lim_{\sequenceVariable \to \natnum} \E \summand_{\sequenceVariable , 1}^2 = \limitVariance$.
\end{lemma}

\begin{proof}
    Observe that $\E \summand_{\sequenceVariable , 1}^2$ can be written as
    \begin{align*}
        \frac{\outVar}{\layerDimensionN}
        \! \!
        \sum_{(\inputSymbol, i), (\inputSymbol', j)}
        \! \! \! \!
        (\projectionCoefficients^{\inputSymbol, i})^\top
        \E \biggl[
            \indexedActivation{\depthSymbol 1}{\sequenceVariable}{\inputSymbol}
            \varepsilon_{\cdot i}
            (\varepsilon_{\cdot j})^\top
            \indexedActivation{\depthSymbol 1}{\sequenceVariable}{\inputSymbol'}^\top
        \biggr]
        \projectionCoefficients^{\inputSymbol', j}
        =
        \frac{\outVar}{\layerDimensionN}
        \sum_{(\inputSymbol, i)}
        (\projectionCoefficients^{\inputSymbol, i})^\top
        \E \biggl[
            \indexedActivation{\depthSymbol 1}{\sequenceVariable}{\inputSymbol}
            \indexedActivation{\depthSymbol 1}{\sequenceVariable}{\inputSymbol}^\top
        \biggr]
        \projectionCoefficients^{\inputSymbol, i}
        \, ,
    \end{align*}
    and thus it will be sufficient to show that $\E [ \indexedActivation{\depthSymbol 1}{\sequenceVariable}{\inputSymbol}  \indexedActivation{\depthSymbol 1}{\sequenceVariable}{\inputSymbol}^\top ] / \layerDimensionN$ 
    converges to the mean of the weak distributional limit.
    \begin{align*}
        \frac{1}{\layerDimensionN}
        \E \biggl[
            \indexedActivation{\depthSymbol 1}{\sequenceVariable}{\inputSymbol}
            \indexedActivation{\depthSymbol 1}{\sequenceVariable}{\inputSymbol}^\top
        \biggr]
        &=
        \frac{1}{\layerDimensionN}
        \E \biggl[
            \softmax(\logitSymbol_{\sequenceVariable}^{\depthSymbol 1} (\inputSymbol))
            \indexedActivity{\depthSymbol - 1}{\sequenceVariable}{\inputSymbol}
            \weightMatSymbol_{\sequenceVariable}^{\depthSymbol 1 , \valueSymbol}
            (\weightMatSymbol_{\sequenceVariable}^{\depthSymbol 1 , \valueSymbol})^\top
            \indexedActivity{\depthSymbol - 1}{\sequenceVariable}{\inputSymbol}^\top
            \softmax(\logitSymbol_{\sequenceVariable}^{\depthSymbol 1} (\inputSymbol))^\top
        \biggr]
        \\
        &=
        \valueVar
        \E \biggl[
            \softmax(\logitSymbol_{\sequenceVariable}^{\depthSymbol 1} (\inputSymbol))
            \frac{
                \indexedActivity{\depthSymbol - 1}{\sequenceVariable}{\inputSymbol}
                \indexedActivity{\depthSymbol - 1}{\sequenceVariable}{\inputSymbol}^\top
            }{\layerDimension{\depthSymbol - 1}}
            \softmax(\logitSymbol_{\sequenceVariable}^{\depthSymbol 1} (\inputSymbol))^\top
        \biggr]
        \, ,
    \end{align*}
    suggests the~desired result could be obtained by application of \Cref{thm:mean_convergence} which requires that the~integrands converge in distribution to the~relevant limit, and that their collection is uniformly integrable.
    Combination of the continuous mapping theorem and \Cref{lem:inner_prod_converge,lem:logit_dist_convergence,lem:slutsky} yields convergence in distribution; application of the H{\" o}lder's inequality, the polynomial bound on $\softmax$, and \Cref{lem:mmnt_propagation} yields uniform integrability by \Cref{lem:sup_ui}, concluding the proof.
\end{proof}

\begin{lemma}\label{lem:head_all_sqmoments_converge}
    For any $\headIndex, \headIndex' \in \natnum$, $\E \lbrack  \summand_{\sequenceVariable , \headIndex}^2 \summand_{\sequenceVariable , \headIndex'}^2 \rbrack$ 
    to the mean of the weak limit of $\{ \summand_{\sequenceVariable , \headIndex}^2 \summand_{\sequenceVariable , \headIndex'}^2 \}_{\sequenceVariable \geq 1}$'s distributions.
\end{lemma}

\begin{proof}[Proof of \Cref{lem:head_all_sqmoments_converge}]
    Defining $\tildeActivation{\headIndex}{\sequenceVariable, \cdot i}{\inputSymbol} \coloneqq \sqrt{\headDimension} \indexedActivation{\depthSymbol \headIndex}{\sequenceVariable}{\inputSymbol} \weightOGen{\sequenceVariable, \cdot i}{\depthSymbol \headIndex, \outputSymbol}$, we observe $\E \lbrack  \summand_{\sequenceVariable , \headIndex}^2 \summand_{\sequenceVariable , \headIndex'}^2 \rbrack$ equals
    \begin{align*}
        \sum_{\substack{(\inputSymbol_1, i_1) \\ (\inputSymbol_2, i_2)}}
        \sum_{\substack{(\inputSymbol_1', j_1) \\ (\inputSymbol_2', j_2)}}
        (\projectionCoefficients^{\inputSymbol_1, i_1})^\top
        \E \biggl[
            \tildeActivation{\headIndex}{\sequenceVariable, \cdot i_1}{\inputSymbol_1}
            \tildeActivation{\headIndex}{\sequenceVariable, \cdot i_2}{\inputSymbol_2}
            \projectionCoefficients^{\inputSymbol_2, i_2}
            (\projectionCoefficients^{\inputSymbol_1', j_1})^\top
            \tildeActivation{\headIndex'}{\sequenceVariable, \cdot j_1}{\inputSymbol_1'} 
            \tildeActivation{\headIndex'}{\sequenceVariable, \cdot j_2}{\inputSymbol_2'} 
        \biggr]
        \projectionCoefficients^{\inputSymbol_2', j_2}
        \, ,
    \end{align*}
    which means that the~expectation can be evaluated as a~weighted sum of terms of the~form
    \begin{equation*}
        \E \bigl[
            \tildeActivation{\headIndex}{\sequenceVariable, a i}{s}
            \tildeActivation{\headIndex}{\sequenceVariable, b j}{t}
            \tildeActivation{\headIndex'}{\sequenceVariable, c k}{u} 
            \tildeActivation{\headIndex'}{\sequenceVariable, d l}{v} 
        \bigr]
        \, , 
    \end{equation*}
    where $a, b, c, d \in [\spatialDimension]$, $i, j, k, l \in \projectionIndeces_{\mathcal{I}}$, and $s, t, u, v \in \projectionIndeces_{\indexSet}$.
    We therefore only need to show convergence of these expectations.
    Substituting:
    \begin{align*}
        \E \bigl[
            \tildeActivation{\headIndex}{\sequenceVariable, a i}{s}
            &
            \tildeActivation{\headIndex}{\sequenceVariable, b j}{t}
            \tildeActivation{\headIndex'}{\sequenceVariable, c k}{u} 
            \tildeActivation{\headIndex'}{\sequenceVariable, d l}{v} 
        \bigr]
        \\
        &= 
        \biggl(\frac{\outVar}{\layerDimensionN}\biggr)^2
        \E \biggl[
            \indexedActivation{\depthSymbol \headIndex}{\sequenceVariable, a \cdot}{s} \varepsilon_{\cdot i}^{\headIndex}
            (\varepsilon_{\cdot j}^{\headIndex})^\top \indexedActivation{\depthSymbol \headIndex}{\sequenceVariable, b \cdot}{t}^\top
            \indexedActivation{\depthSymbol \headIndex'}{\sequenceVariable, c \cdot}{u} \varepsilon_{\cdot k}^{\headIndex'}
            (\varepsilon_{\cdot l}^{\headIndex'})^\top \indexedActivation{\depthSymbol \headIndex'}{\sequenceVariable, d \cdot}{v}^\top
        \biggr]
        \\
        &=
        \biggl(\frac{\outVar}{\layerDimensionN}\biggr)^2
        \E \biggl[
            \indexedActivation{\depthSymbol \headIndex}{\sequenceVariable, a \cdot}{s} 
            \indexedActivation{\depthSymbol \headIndex}{\sequenceVariable, b \cdot}{t}^\top
            \indexedActivation{\depthSymbol \headIndex'}{\sequenceVariable, c \cdot}{u} 
            \indexedActivation{\depthSymbol \headIndex'}{\sequenceVariable, d \cdot}{v}^\top
        \biggr]
        \delta_{i = j} \delta_{k = l}
        \, , 
    \end{align*}
    where $\varepsilon_{i}^\headIndex$ are i.i.d.\ standard normal random variables,
    and re-purposing the~$i, j$ indices, we have
    \begin{align*}
        &\frac{1}{(\layerDimensionN)^2}
        \E \biggl[
            \indexedActivation{\depthSymbol \headIndex}{\sequenceVariable, a \cdot}{s} 
            \indexedActivation{\depthSymbol \headIndex}{\sequenceVariable, b \cdot}{t}^\top
            \indexedActivation{\depthSymbol \headIndex'}{\sequenceVariable, c \cdot}{u} 
            \indexedActivation{\depthSymbol \headIndex'}{\sequenceVariable, d \cdot}{v}^\top
        \biggr]
        =
        \frac{1}{(\layerDimensionN)^2}
        \sum_{i, j = 1}^{\layerDimensionN}
            \E \biggl[
                \indexedActivation{\depthSymbol \headIndex}{\sequenceVariable, a i}{s} 
                \indexedActivation{\depthSymbol \headIndex}{\sequenceVariable, b i}{t}
                \indexedActivation{\depthSymbol \headIndex'}{\sequenceVariable, c j}{u} 
                \indexedActivation{\depthSymbol \headIndex'}{\sequenceVariable, d j}{v}
            \biggr]
        \\
        &=
        \frac{1}{\layerDimensionN}
        \E \biggl[
            \indexedActivation{\depthSymbol \headIndex}{\sequenceVariable, a 1}{s} 
            \indexedActivation{\depthSymbol \headIndex}{\sequenceVariable, b 1}{t}
            \indexedActivation{\depthSymbol \headIndex'}{\sequenceVariable, c 1}{u} 
            \indexedActivation{\depthSymbol \headIndex'}{\sequenceVariable, d 1}{v}
        \biggr]
        +
        \frac{\layerDimensionN - 1}{\layerDimensionN}
        \E \biggl[
            \indexedActivation{\depthSymbol \headIndex}{\sequenceVariable, a 1}{s} 
            \indexedActivation{\depthSymbol \headIndex}{\sequenceVariable, b 1}{t}
            \indexedActivation{\depthSymbol \headIndex'}{\sequenceVariable, c 2}{u} 
            \indexedActivation{\depthSymbol \headIndex'}{\sequenceVariable, d 2}{v}
        \biggr]
        \, .
    \end{align*}
    Note that we can bound the integrands by a universal constant (\Cref{lem:mmnt_propagation}), and thus we can focus only on the latter term on the r.h.s.
    We can thus turn to
    \begin{align*}%
        &\E \biggl[
            \indexedActivation{\depthSymbol \headIndex}{\sequenceVariable, a 1}{s} 
            \indexedActivation{\depthSymbol \headIndex}{\sequenceVariable, b 1}{t}
            \indexedActivation{\depthSymbol \headIndex'}{\sequenceVariable, c 2}{u} 
            \indexedActivation{\depthSymbol \headIndex'}{\sequenceVariable, d 2}{v}
        \biggr]
        \\
        &=
        \valueStd^4
        \E \biggl[
            \softmax(\logitSymbol_{\sequenceVariable}^{\depthSymbol \headIndex} (s))_{a \cdot} 
            \frac{
                \indexedActivity{\depthSymbol - 1}{\sequenceVariable}{s}
                \indexedActivity{\depthSymbol - 1}{\sequenceVariable}{t}^\top
            }{
                \layerDimension{\depthSymbol - 1}
            }
            \softmax(\logitSymbol_{\sequenceVariable}^{\depthSymbol \headIndex} (s))_{b \cdot}^\top
            \softmax(\logitSymbol_{\sequenceVariable}^{\depthSymbol \headIndex'} (u))_{c \cdot}
            \frac{
                \indexedActivity{\depthSymbol - 1}{\sequenceVariable}{u}
                \indexedActivity{\depthSymbol - 1}{\sequenceVariable}{v}^\top
            }{
                \layerDimension{\depthSymbol - 1}
            }
            \softmax(\logitSymbol_{\sequenceVariable}^{\depthSymbol \headIndex'} (v))_{d \cdot}^\top
        \biggr]
        \nonumber
        \, .
    \end{align*}
    Observe that by \Cref{lem:inner_prod_converge}, 
    \begin{equation*}
            \left(
                \frac{
                    \indexedActivity{\depthSymbol - 1}{\sequenceVariable}{s}
                    \indexedActivity{\depthSymbol - 1}{\sequenceVariable}{t}^\top
                }{
                    \layerDimension{\depthSymbol - 1}
                }
                \, , \,
            \frac{
                \indexedActivity{\depthSymbol - 1}{\sequenceVariable}{u}
                \indexedActivity{\depthSymbol - 1}{\sequenceVariable}{v}^\top
            }{
                \layerDimension{\depthSymbol - 1}
            }
            \right)
            \convergeProb
            (
                \kerntildef{}{\depthSymbol}{s}{t}
                \, , \,
                \kerntildef{}{\depthSymbol}{u}{v}
            )
            \, ,
    \end{equation*}
    and by \Cref{lem:logit_dist_convergence} and the~continuous mapping theorem 
    $$ \softmax(\logitSymbol_{\sequenceVariable}^{\depthSymbol \headIndex} (s))_{a \cdot} \softmax(\logitSymbol_{\sequenceVariable}^{\depthSymbol \headIndex} (s))_{b \cdot} \softmax(\logitSymbol_{\sequenceVariable}^{\depthSymbol \headIndex'} (u))_{c \cdot} \softmax(\logitSymbol_{\sequenceVariable}^{\depthSymbol \headIndex'} (v))_{d \cdot}  \, ,$$
    converges in distribution.
    By \Cref{lem:slutsky}, this means that the~integrand converges in distribution.
    Finally, to obtain the convergence of the expectation, we apply \Cref{thm:mean_convergence} where the required uniform integrability can be obtained by applying H{\" o}lder's inequality and \Cref{lem:mmnt_propagation}.
\end{proof}

\subsubsection{Convergence of $\logitN$}

\begin{lemma}\label{lem:logit_dist_convergence}
    Let the~assumptions of \Cref{thm:gp_convergence_sqrt} hold.
    Then 
    $\logitsN \coloneqq \{ \logitN(\inputSymbol) \colon x \in \indexSet , \headIndex \in \natnum \}$  %
    converges in distribution to a~centred GP with covariance as described in \Cref{eq:logit_cov}.
\end{lemma}

\begin{proof}
    Using \Cref{lem:fin_dim_marg} and the~Cram{\' e}r Wold device \citep[p.~383]{billingsey86}, we can again restrict our attention to one~dimensional projections of finite dimensional marginals of $\logitsN$
    \begin{align*}
        \projectionLogit
        &\coloneqq
        \sum_{(\inputSymbol, \headIndex) \in \projectionIndeces}
            \langle
                \projectionCoefficientsLogit^{\inputSymbol, \headIndex} ,
                \logitN(\inputSymbol)
            \rangle_F
        =
        \sum_{(\inputSymbol, \headIndex) \in \projectionIndeces}
            \bigl\langle
                \projectionCoefficientsLogit^{\inputSymbol, \headIndex} ,
                \frac{1}{\sqrt{\layerDimensionN}}
                \sum_{j = 1}^{\layerDimensionN}
                    \query{\sequenceVariable, \cdot j}{\depthSymbol \headIndex}(\inputSymbol)
                    (\key{\sequenceVariable, \cdot j}{\depthSymbol \headIndex}(\inputSymbol))^\top
            \bigr\rangle_F
        \\
        &=
        \frac{1}{\sqrt{\layerDimensionN}}
        \sum_{j = 1}^{\layerDimensionN}
        \underbrace{
            \sum_{(\inputSymbol, \headIndex) \in \projectionIndeces}
                \bigl\langle
                    \projectionCoefficientsLogit^{\inputSymbol, \headIndex} ,
                    \query{\sequenceVariable, \cdot j}{\depthSymbol \headIndex}(\inputSymbol)
                    (\key{\sequenceVariable, \cdot j}{\depthSymbol \headIndex}(\inputSymbol))^\top
                \bigr\rangle_F
        }_{\eqqcolon \summandLogitN}
        \, ,
    \end{align*}
    The~above formula suggests the~desired result follows from \Cref{lemma:eCLT}:
    \begin{itemize}
        \item Exchangeability requirement is satisfied by \Cref{lem:logit_exchangeability}.
        \item Zero mean and covariance follow from \Cref{lem:logit_mean_corr_zero}.
        \item Convergence of variance is established in \Cref{cor:logit_var_convergence}.
        \item Convergence of $\E \lbrack \summandLogit{\sequenceVariable , 1}^2 \summandLogit{\sequenceVariable , 2}^2 \rbrack$ is proved in \Cref{lem:logit_square_moments}.
        \item $\littleO(\layerDimensionN)$ growth of the~third absolute moments is implied by \Cref{lem:logit_third_moments}.
        \qedhere
    \end{itemize}
\end{proof}

\begin{lemma}\label{lem:logit_exchangeability}
    Under the~assumptions of \Cref{thm:gp_convergence_sqrt}, $\summandLogit{\sequenceVariable, j}$ are exchangeable over the~$j$ index.
\end{lemma}

\begin{proof}
    Observe
    \begin{align*}
        \query{\sequenceVariable, \cdot j}{\depthSymbol \headIndex}(\inputSymbol)
        (\key{\sequenceVariable, \cdot j}{\depthSymbol \headIndex}(\inputSymbol))^\top
        =
        \indexedActivity{\depthSymbol - 1}{\sequenceVariable}{\inputSymbol} 
        \weightQ{\cdot i}
        (\weightK{\cdot i})^\top
        \indexedActivity{\depthSymbol - 1}{\sequenceVariable}{\inputSymbol}^\top
        \, ,
    \end{align*}
    which means that the~individual terms $\summandLogit{\sequenceVariable, j}$ are i.i.d.\ if we condition on
        $\{ \indexedActivity{\depthSymbol - 1}{\sequenceVariable}{\inputSymbol} \colon x \in \projectionIndeces_{\indexSet} \}$.
    Application of de Finetti's theorem concludes the~proof.
\end{proof}

\begin{lemma}\label{lem:logit_mean_corr_zero}
    Under the~assumptions of \Cref{thm:gp_convergence_sqrt},
    $\E [\summandLogit{\sequenceVariable, 1}] = \E [\summandLogit{\sequenceVariable, 1} \summandLogit{\sequenceVariable, 2}] = 0$.
\end{lemma}

\begin{proof}
    For $\E [\summandLogit{\sequenceVariable, 1}] = 0$, note that for any $\headIndex \in \natnum$,
    $\E [\summandLogit{\sequenceVariable, 1}]$ can be expressed as a~sum over terms
    \begin{equation*}
        \bigl\langle
            \projectionCoefficientsLogit^{\inputSymbol, \headIndex} ,
            \E [
                \indexedActivity{\depthSymbol - 1}{\sequenceVariable}{\inputSymbol}
                \weightQ{\cdot j} 
                (\weightK{\cdot j})^\top
                \indexedActivity{\depthSymbol - 1}{\sequenceVariable}{\inputSymbol}^\top 
            ]
        \bigr\rangle_F
        =
        \bigl\langle
            \projectionCoefficientsLogit^{\inputSymbol, \headIndex} ,
            0
        \bigr\rangle_F
        =
        0 
        \, ,
    \end{equation*}
    as long as $\E [ \indexedActivity{\depthSymbol - 1}{\sequenceVariable, \cdot 1}{\inputSymbol} \indexedActivity{\depthSymbol - 1}{\sequenceVariable, \cdot 1}{\inputSymbol}^\top ]$
    is entry-wise finite for any $(\inputSymbol, \sequenceVariable) \in \projectionIndeces_{\indexSet} \times \natnum$ which can be obtained by \Cref{lem:mmnt_propagation}.
    For $\E [\summandLogit{\sequenceVariable, 1} \summandLogit{\sequenceVariable, 2}] = 0$, we have to evaluate a~weighted sum of terms of the~form
    \begin{align*}
        &(\projectionCoefficientsLogit^{\inputSymbol , \headIndex})^\top
        \E \biggl[
            \indexedActivity{\depthSymbol - 1}{\sequenceVariable}{\inputSymbol} 
            \weightQ{\cdot 1}
            (\weightK{\cdot 1})^\top
            \indexedActivity{\depthSymbol - 1}{\sequenceVariable}{\inputSymbol}^\top
            \indexedActivity{\depthSymbol - 1}{\sequenceVariable}{\inputSymbol'} 
            \weightMatSymbol_{\sequenceVariable, \cdot 2}^{\depthSymbol \headIndex', \querySymbol}
            (\weightMatSymbol_{\sequenceVariable, \cdot 2}^{\depthSymbol \headIndex', \keySymbol})^\top
            \indexedActivity{\depthSymbol - 1}{\sequenceVariable}{\inputSymbol'}^\top
        \biggr]
        \projectionCoefficientsLogit^{\inputSymbol' , \headIndex'}
        \, ,
    \end{align*}
    which are all equal to zero as long as 
    \begin{align*}
        \E \biggl[
            \frac{
                \indexedActivity{\depthSymbol - 1}{\sequenceVariable}{\inputSymbol} 
                \indexedActivity{\depthSymbol - 1}{\sequenceVariable}{\inputSymbol}^\top
            }{\layerDimension{\depthSymbol - 1}}
            \frac{
                \indexedActivity{\depthSymbol - 1}{\sequenceVariable}{\inputSymbol'} 
                \indexedActivity{\depthSymbol - 1}{\sequenceVariable}{\inputSymbol'}^\top
            }{\layerDimension{\depthSymbol - 1}}
        \biggr]
        \, ,
    \end{align*}
    is entry-wise finite.
    Since the~integrand converges in probability to $\kerntildef{}{\depthSymbol}{\inputSymbol}{\inputSymbol} \kerntildef{}{\depthSymbol}{\inputSymbol'}{\inputSymbol'}$ by \Cref{lem:inner_prod_converge}, an~argument analogous to the~one made above for the~$\E [\summandLogit{\sequenceVariable, 1}] = 0$ concludes the~proof.
\end{proof}

\begin{corollary}\label{cor:logit_var_convergence}
    Under the~assumptions of \Cref{thm:gp_convergence_sqrt}, $\lim_{\sequenceVariable \to \infty} \E \lbrack \summandLogit{\sequenceVariable , 1}^2 \rbrack = \limitVariance$.
\end{corollary}

\begin{proof}
    The~second half the~proof of \Cref{lem:logit_mean_corr_zero} establishes $\E [\summandLogit{\sequenceVariable, i} \summandLogit{\sequenceVariable, j}]$ converges for any $i, j$.
\end{proof}

\vspace{0.5\baselineskip}
\begin{lemma}\label{lem:logit_square_moments}
    Under the~assumptions of \Cref{thm:gp_convergence_sqrt}, $\lim_{\sequenceVariable \to \infty} \E \lbrack \summandLogit{\sequenceVariable , 1}^2 \summandLogit{\sequenceVariable , 2}^2 \rbrack = \limitStd^{4}$.
\end{lemma}

\begin{proof}
    Defining $R_{\sequenceVariable, j}^{\headIndex}(\inputSymbol) \coloneqq \query{\sequenceVariable, \cdot j}{\depthSymbol \headIndex}(\inputSymbol) (\key{\sequenceVariable, \cdot j}{\depthSymbol \headIndex}(\inputSymbol))^\top$, we can rewrite
    $\E \bigl[ \summandLogit{\sequenceVariable , 1}^2 \summandLogit{\sequenceVariable , 2}^2  \bigr]$ as
    \begin{align*}
        \sum_{\substack{(\inputSymbol, \headIndex) , (\inputSymbol', \headIndex')}}
        \!\!\!\!
            (\projectionCoefficientsLogit^{\inputSymbol, \headIndex})^\top
            \E \biggl[
                R_{\sequenceVariable, 1}^{\headIndex}(\inputSymbol)
                R_{\sequenceVariable, 1}^{\headIndex}(\inputSymbol)^\top
                \projectionCoefficientsLogit^{\inputSymbol, \headIndex}
                (\projectionCoefficientsLogit^{\inputSymbol', \headIndex'})^\top
                R_{\sequenceVariable, 2}^{\headIndex'}(\inputSymbol')
                R_{\sequenceVariable, 2}^{\headIndex'}(\inputSymbol')^\top
            \biggr]
            \projectionCoefficientsLogit^{\inputSymbol', \headIndex'}
        \, ,
    \end{align*}
    where we have w.l.o.g.\ assumed all matrices have been flattened as $\langle A , B \rangle_F = \vectorise(A)^\top \vectorise(B)$.
    The~above could be further rewritten as a~weighted sum of terms which take the following form:
    \begin{align*}
        &\E \biggl[
            \query{\sequenceVariable, a_1 1}{\depthSymbol \headIndex}(\inputSymbol)
            \key{\sequenceVariable, b_1 1}{\depthSymbol \headIndex}(\inputSymbol)
            \query{\sequenceVariable, a_2 1}{\depthSymbol \headIndex}(\inputSymbol)
            \key{\sequenceVariable, b_2 1}{\depthSymbol \headIndex}(\inputSymbol)
            \query{\sequenceVariable, a_3 2}{\depthSymbol \headIndex'}(\inputSymbol')
            \key{\sequenceVariable, b_3 2}{\depthSymbol \headIndex'}(\inputSymbol')
            \query{\sequenceVariable, a_4 2}{\depthSymbol \headIndex'}(\inputSymbol')
            \key{\sequenceVariable, b_4 2}{\depthSymbol \headIndex'}(\inputSymbol')
        \biggr]
        \\
        &\propto
        \E \biggl[
            \frac{
                \indexedActivity{\depthSymbol - 1}{\sequenceVariable, a_1 \cdot}{\inputSymbol}
                \indexedActivity{\depthSymbol - 1}{\sequenceVariable, a_2 \cdot}{\inputSymbol}^\top
            }{\layerDimension{\depthSymbol - 1}}
            \frac{
                \indexedActivity{\depthSymbol - 1}{\sequenceVariable, b_1 \cdot}{\inputSymbol}
                \indexedActivity{\depthSymbol - 1}{\sequenceVariable, b_2 \cdot}{\inputSymbol}^\top
            }{\layerDimension{\depthSymbol - 1}}
            \frac{
                \indexedActivity{\depthSymbol - 1}{\sequenceVariable, a_3 \cdot}{\inputSymbol'}
                \indexedActivity{\depthSymbol - 1}{\sequenceVariable, a_4 \cdot}{\inputSymbol'}^\top
            }{\layerDimension{\depthSymbol - 1}}
            \frac{
                \indexedActivity{\depthSymbol - 1}{\sequenceVariable, b_3 \cdot}{\inputSymbol'}
                \indexedActivity{\depthSymbol - 1}{\sequenceVariable, b_4 \cdot}{\inputSymbol'}^\top
            }{\layerDimension{\depthSymbol - 1}}
        \biggr]
        \, .
    \end{align*}
    Thanks to \Cref{lem:inner_prod_converge} and the continuous mapping theorem, we know that the~integrand converges in probability to
    \begin{equation*}
        \queryStd^4 \keyStd^4
        \kerntildef{a_1 a_2}{\depthSymbol}{\inputSymbol}{\inputSymbol}
        \kerntildef{b_1 b_2}{\depthSymbol}{\inputSymbol}{\inputSymbol}
        \kerntildef{a_3 a_4}{\depthSymbol}{\inputSymbol'}{\inputSymbol'}
        \kerntildef{b_3 b_4}{\depthSymbol}{\inputSymbol'}{\inputSymbol'}
        \, ,
    \end{equation*}
    and thus we can use \Cref{thm:mean_convergence} to obtain that the~above expactation converges as long as the sequence of integrands is uniformly integrable.
    Noting that we can upper bound by $\max_{c \in [\spatialDimension} \max_{z \in \projectionIndeces_{\indexSet}} \E | \indexedActivity{\depthSymbol - 1}{\sequenceVariable, c 1}{z} |^8$ by H{\" o}lder's inequality and exchangeability, uniform integrability can be obtained by \Cref{lem:sup_ui}.
\end{proof}

\begin{lemma}\label{lem:logit_third_moments}
    Under the~assumptions of \Cref{thm:gp_convergence_sqrt}, $\E | \summandLogit{\sequenceVariable, 1} |^3 = \littleO(\sqrt{\layerDimensionN})$.
\end{lemma}

\begin{proof}
    Using H{\" o}lder's inequality, it is sufficient to show $\limsup_{\sequenceVariable} \E | \summandLogit{\sequenceVariable, 1} |^4 < \infty$.
    Setting $R_{\sequenceVariable, j}^{\headIndex}(\inputSymbol) \coloneqq \query{\sequenceVariable, \cdot j}{\depthSymbol \headIndex}(\inputSymbol) (\key{\sequenceVariable, \cdot j}{\depthSymbol \headIndex}(\inputSymbol))^\top$
    \begin{align*}
        \E | \summandLogit{\sequenceVariable, 1} |^4
        =
        \sum_{\substack{(\inputSymbol, \headIndex) , (\inputSymbol', \headIndex')}}
        \!\!\!\!
            (\projectionCoefficientsLogit^{\inputSymbol, \headIndex})^\top
            \E \biggl[
                R_{\sequenceVariable, 1}^{\headIndex}(\inputSymbol)
                R_{\sequenceVariable, 1}^{\headIndex}(\inputSymbol)^\top
                \projectionCoefficientsLogit^{\inputSymbol, \headIndex}
                (\projectionCoefficientsLogit^{\inputSymbol', \headIndex'})^\top
                R_{\sequenceVariable, 1}^{\headIndex'}(\inputSymbol')
                R_{\sequenceVariable, 1}^{\headIndex'}(\inputSymbol')^\top
            \biggr]
            \projectionCoefficientsLogit^{\inputSymbol', \headIndex'}
        \, ,
    \end{align*}
    analogously to the~proof of \Cref{lem:logit_square_moments}.
    Substituting for the~individual terms and using H{\" o}lder's inequality, we we can see that each of the~terms in the~above sum can be itself decomposed into a~sum over $(\layerDimension{\depthSymbol - 1})^8$ terms that are up to a constant upper bounded by 
    \begin{equation*}
        \max_{a \in [\spatialDimension]} \max_{z \in \{ \inputSymbol, \inputSymbol' \} } 
        \E | \indexedActivity{\depthSymbol - 1}{\sequenceVariable, a 1}{z} |^8 
        \, ,
    \end{equation*}
    which means we can conclude this proof by bounding this quantity by a constant independent of $\sequenceVariable$ by \Cref{lem:mmnt_propagation}.
\end{proof}

\subsection{NTK convergence proof}\label{sect:ntk_proofs}

We need to prove convergence of the attention NTK at initialisation, i.e., for any $a, b \in [\spatialDimension]$, $i, j \in \natnum$, and $\inputSymbol, \inputSymbol' \in \indexSet$ 
\begin{align}
    \frac{
        \partial \indexedActivation{\depthSymbol}{\sequenceVariable, \fIndexA i}{\inputSymbol}
    }{
        \partial \params_{\sequenceVariable}^{\leq \depthSymbol}
    }
        \frac{
            \partial \indexedActivation{\depthSymbol}{\sequenceVariable, \fIndexB j}{\inputSymbol'}
        }{
            \partial \params_{\sequenceVariable}^{\leq \depthSymbol}
        }^\top
    \convergeProb
    \delta_{i=j}
    \ntkf{\fIndexA \fIndexB}{\depthSymbol}{\inputSymbol}{\inputSymbol'}
    \, ,
\end{align}
where $\params_{\sequenceVariable}^{\leq \depthSymbol}$ is the collection of trainable parameters in the first $\depthSymbol$ layers, as $\sequenceVariable \to \infty$.
We will further use $\params_{\sequenceVariable}^{\depthSymbol}$ to refer to the trainable parameters of the $\depthSymbol$\textsuperscript{th} layer; e.g., for the attention layer $\params_{\sequenceVariable}^{\depthSymbol} = \{ \uW_{\sequenceVariable}^{\depthSymbol} \} \cup \bigcup_{\headIndex = 1}^{\headDimension} \{ \uW_{\sequenceVariable}^{\depthSymbol \headIndex, \querySymbol} , \uW_{\sequenceVariable}^{\depthSymbol \headIndex, \keySymbol}, \uW_{\sequenceVariable}^{\depthSymbol \headIndex, \valueSymbol} \}$.

Note that
\begin{align}
    \frac{
        \partial \indexedActivation{\depthSymbol}{\sequenceVariable, \fIndexA i}{\inputSymbol}
    }{
        \partial \params_{\sequenceVariable}^{\leq \depthSymbol}
    }
        \frac{
            \partial \indexedActivation{\depthSymbol}{\sequenceVariable, \fIndexB j}{\inputSymbol'}
        }{
            \partial \params_{\sequenceVariable}^{\leq \depthSymbol}
        }^\top
    =
    \underbrace{
    \frac{
        \partial \indexedActivation{\depthSymbol}{\sequenceVariable, \fIndexA i}{\inputSymbol}
    }{
        \partial \params_{\sequenceVariable}^{\depthSymbol}
    }
        \frac{
            \partial \indexedActivation{\depthSymbol}{\sequenceVariable, \fIndexB j}{\inputSymbol'}
        }{
            \partial \params_{\sequenceVariable}^{\depthSymbol}
        }^\top
    }_{\text{direct}}
    +
    \underbrace{
    \frac{
        \partial \indexedActivation{\depthSymbol}{\sequenceVariable, \fIndexA i}{\inputSymbol}
    }{
        \partial \indexedActivity{\depthSymbol - 1}{\sequenceVariable}{\inputSymbol}
    }
    \frac{
        \partial \indexedActivity{\depthSymbol - 1}{\sequenceVariable}{\inputSymbol}
    }{
        \partial \params_{\sequenceVariable}^{< \depthSymbol}
    }
        \frac{
            \partial \indexedActivity{\depthSymbol - 1}{\sequenceVariable}{\inputSymbol'}
        }{
            \partial \params_{\sequenceVariable}^{< \depthSymbol}
        }^\top
        \frac{
            \partial \indexedActivation{\depthSymbol}{\sequenceVariable, \fIndexB j}{\inputSymbol'}
        }{
            \partial \indexedActivity{\depthSymbol - 1}{\sequenceVariable}{\inputSymbol'}
        }^\top
    }_{\text{indirect}}
    \, ,
\end{align}
where the \emph{direct} part corresponds to the contribution due to gradient w.r.t.\ the parameters of the $\depthSymbol$\textsuperscript{th} layer itself, and the \emph{indirect} part is due to effect of the $\depthSymbol$\textsuperscript{th} layer on the contribution due to the parameters of preceding layers.
The next two sections show convergence of each of these terms to a constant in probability, implying the desired result:

\begin{theorem}[NTK convergence]\label{thm:ntk_convergence}
    Under the assumptions of \Cref{thm:gp_convergence_sqrt} (including those stated at the beginning of \Cref{sect:proofs}), for any $\fIndexA, \fIndexB \in [\spatialDimension]$, and $\inputSymbol, \inputSymbol' \in \indexSet$
    \begin{align*}
        \frac{
            \partial \indexedActivation{\depthSymbol}{\sequenceVariable, \fIndexA i}{\inputSymbol}
        }{
            \partial \params_{\sequenceVariable}^{\leq \depthSymbol}
        }
            \frac{
                \partial \indexedActivation{\depthSymbol}{\sequenceVariable, \fIndexB j}{\inputSymbol'}
            }{
                \partial \params_{\sequenceVariable}^{\leq \depthSymbol}
            }^\top
        \convergeProb
        \delta_{i=j}
        \ntkf{\fIndexA \fIndexB}{\depthSymbol}{\inputSymbol}{\inputSymbol'}
        \, ,
    \end{align*}
    where
    \begin{align}
        \ntkf{\fIndexA \fIndexB}{\depthSymbol}{\inputSymbol}{\inputSymbol'}
        =
        &
        2 
        \kernelf{\fIndexA \fIndexB}{\depthSymbol}{\inputSymbol}{\inputSymbol'}
        +
        \OVVar
        \sum_{\substack{\gIndexA, \gIndexB}}^{\spatialDimension}
            \ntktildef{\gIndexA \gIndexB}{\depthSymbol}{\inputSymbol}{\inputSymbol'}
            \E [
                \tildeLogit{\fIndexA \gIndexA}{\depthSymbol 1} (\inputSymbol)
                \tildeLogit{\fIndexB \gIndexB}{\depthSymbol 1} (\inputSymbol')
            ]
        +
        \nonumber
        \\
        &
        \delta_{\scaling = \frac{1}{2}}
        \OVVar
        \QKVar
        (
            2\kerntildef{\fIndexA \fIndexB}{\depthSymbol}{\inputSymbol}{\inputSymbol'}
            +
            \ntktilde_{\fIndexA \fIndexB}^{\depthSymbol}(\inputSymbol, \inputSymbol')
        )
        \sum_{\substack{c_1, c_2 \\ d_1 , d_2}}^{\spatialDimension}
            \kerntildef{c_1 c_2}{\depthSymbol}{\inputSymbol}{\inputSymbol'}
            \kerntildef{d_1 d_2}{
            \depthSymbol}{\inputSymbol}{\inputSymbol'}
            \E \left[
                \frac{
                    \partial
                    \tildeLogitN{\fIndexA c_1}{\depthSymbol 1}{\inputSymbol}
                }{
                    \partial
                    \logitSymbol_{\fIndexA d_1}^{\depthSymbol 1}(\inputSymbol)
                }
                \frac{
                    \partial
                    \tildeLogitN{\fIndexB c_2}{\depthSymbol 1}{\inputSymbol'}
                }{
                    \partial
                    \logitSymbol_{\fIndexB d_2}^{\depthSymbol 1}(\inputSymbol')
                }
            \right]
        +
        \nonumber
        \\
        &
        \delta_{\substack{\scaling = \frac{1}{2}}}
        \OVVar
        \QKVar
        \kerntildef{\fIndexA \fIndexB}{\depthSymbol}{\inputSymbol}{\inputSymbol'}
        \sum_{\substack{c_1 , c_2 \\ d_1, d_2}}^{\depthSymbol}
            \kerntildef{c_1 c_2}{\depthSymbol}{\inputSymbol}{\inputSymbol'}
            \ntktilde_{d_1 d_2}^{\depthSymbol}(\inputSymbol, \inputSymbol')
            \E \left[
                \frac{
                    \partial
                    \tildeLogit{\fIndexA c_1}{\depthSymbol 1} (\inputSymbol)
                }{
                    \partial 
                    \logitSymbol_{\fIndexA d_1}^{\depthSymbol 1} (\inputSymbol)
                }
                \frac{
                    \partial
                    \tildeLogit{\fIndexB c_2}{\depthSymbol 1} (\inputSymbol')
                }{
                    \partial 
                    \logitSymbol_{\fIndexB d_2}^{\depthSymbol 1} (\inputSymbol')
                }
            \right]
        \, .
    \end{align}
\end{theorem}

\Cref{thm:ntk_convergence} will be proven in the following two subsections.

\subsubsection{Direct contribution}

The direct contribution of an attention layer can be expanded as
\begin{align*}
    \frac{
        \partial \indexedActivation{\depthSymbol}{\sequenceVariable, \fIndexA i}{\inputSymbol}
    }{
        \partial \params_{\sequenceVariable}^{\depthSymbol}
    }
        \frac{
            \partial \indexedActivation{\depthSymbol}{\sequenceVariable, \fIndexB j}{\inputSymbol'}
        }{
            \partial \params_{\sequenceVariable}^{\depthSymbol}
        }^\top
    =
    &\frac{
        \partial \indexedActivation{\depthSymbol}{\sequenceVariable, \fIndexA i}{\inputSymbol}
    }{
        \partial \uW_\sequenceVariable^{\depthSymbol, \outputSymbol}
    }
    \frac{
        \partial \indexedActivation{\depthSymbol}{\sequenceVariable, \fIndexB j}{\inputSymbol'}
    }{
        \partial \uW_\sequenceVariable^{\depthSymbol, \outputSymbol}
    }^\top
    +
    \\
    &\sum_{\headIndex = 1}^{\headDimension}
        \frac{
            \partial \indexedActivation{\depthSymbol}{\sequenceVariable, \fIndexA i}{\inputSymbol}
        }{
            \partial \uW_\sequenceVariable^{\depthSymbol\headIndex, \valueSymbol}
        }
        \frac{
            \partial \indexedActivation{\depthSymbol}{\sequenceVariable, \fIndexB j}{\inputSymbol'}
        }{
            \partial \uW_\sequenceVariable^{\depthSymbol\headIndex, \valueSymbol}
        }^\top
        +
        \frac{
            \partial \indexedActivation{\depthSymbol}{\sequenceVariable, \fIndexA i}{\inputSymbol}
        }{
            \partial \uW_\sequenceVariable^{\depthSymbol\headIndex, \querySymbol}
        }
        \frac{
            \partial \indexedActivation{\depthSymbol}{\sequenceVariable, \fIndexB j}{\inputSymbol'}
        }{
            \partial \uW_\sequenceVariable^{\depthSymbol\headIndex, \querySymbol}
        }^\top
        +
        \frac{
            \partial \indexedActivation{\depthSymbol}{\sequenceVariable, \fIndexA i}{\inputSymbol}
        }{
            \partial \uW_\sequenceVariable^{\depthSymbol\headIndex, \keySymbol}
        }
        \frac{
            \partial \indexedActivation{\depthSymbol}{\sequenceVariable, \fIndexB j}{\inputSymbol'}
        }{
            \partial \uW_\sequenceVariable^{\depthSymbol\headIndex, \keySymbol}
        }^\top
    \, .
\end{align*}
We prove convergence of each of these terms next.

\begin{lemma}\label{lem:wo_ntk}
$
    \frac{
        \partial \indexedActivation{\depthSymbol}{\sequenceVariable, \fIndexA i}{\inputSymbol}
    }{
        \partial \uW_\sequenceVariable^{\depthSymbol, \outputSymbol}
    }
    \frac{
        \partial \indexedActivation{\depthSymbol}{\sequenceVariable, \fIndexB j}{\inputSymbol'}
    }{
        \partial \uW_\sequenceVariable^{\depthSymbol, \outputSymbol}
    }^\top
    \convergeProb
    \delta_{i = j}
    \kernelf{a b}{\depthSymbol}{\inputSymbol}{\inputSymbol'}
    \, .
$
\end{lemma}

\begin{proof}[Proof of \Cref{lem:wo_ntk}]
Observe
\begin{align*}
    \frac{
        \partial \indexedActivation{\depthSymbol}{\sequenceVariable, \fIndexA i}{\inputSymbol}
    }{
        \partial \uW_\sequenceVariable^{\depthSymbol, \outputSymbol}
    }
    \frac{
        \partial \indexedActivation{\depthSymbol}{\sequenceVariable, \fIndexB j}{\inputSymbol'}
    }{
        \partial \uW_\sequenceVariable^{\depthSymbol, \outputSymbol}
    }^\top
    &=
    \delta_{i = j}
    \sum_{\headIndex = 1}^{\headDimension}
    \sum_{k = 1}^{\valueDimension}
        \frac{\outVar}{\headDimension \valueDimension}
        \indexedActivation{\depthSymbol\headIndex}{\sequenceVariable, \fIndexA k}{\inputSymbol}
        \indexedActivation{\depthSymbol\headIndex}{\sequenceVariable, \fIndexB k}{\inputSymbol'}
    \\
    &=
    \delta_{i = j}
    \frac{\outVar}{\headDimension}
    \sum_{\headIndex = 1}^{\headDimension}
    \sum_{c_1, c_2 = 1}^{\spatialDimension}
        \tildeLogitN{\sequenceVariable, \fIndexA c_1}{\depthSymbol \headIndex}{\inputSymbol}
        \tildeLogitN{\sequenceVariable, \fIndexB c_2}{\depthSymbol \headIndex}{\inputSymbol'}
        \frac{
            \langle
                \val{\sequenceVariable, c_1 \cdot}{\depthSymbol\headIndex}(\inputSymbol) 
                ,
                \val{\sequenceVariable, c_2 \cdot}{\depthSymbol\headIndex}(\inputSymbol')
            \rangle
        }{
            \valueDimension
        }
    \, .
\end{align*}
Since $\spatialDimension$ is fixed, we can focus on an arbitrary pair $c_1, c_2 \in [\spatialDimension]$.
Notice that by the continuous mapping theorem and \Cref{lem:inner_prod_converge,lem:slutsky}, the individual summands converge in distribution
$$
\tildeLogitN{\sequenceVariable, \fIndexA c_1}{\depthSymbol\headIndex}{\inputSymbol}
\tildeLogitN{\sequenceVariable, \fIndexB c_2}{\depthSymbol\headIndex}{\inputSymbol'}
\frac{
    \langle
        \val{\sequenceVariable, c_1 \cdot}{\depthSymbol\headIndex}(\inputSymbol) 
        ,
        \val{\sequenceVariable, c_2 \cdot}{\depthSymbol\headIndex}(\inputSymbol')
    \rangle
}{
    \valueDimension
}
\convergeDist
\valueVar
\tildeLogitN{\fIndexA c_1}{\depthSymbol\headIndex}{\inputSymbol}
\tildeLogitN{\fIndexB c_2}{\depthSymbol\headIndex}{\inputSymbol'}
\kerntildef{c_1 c_2}{\depthSymbol}{\inputSymbol}{\inputSymbol'}
\, ,
$$
where $\tildeLogit{}{\depthSymbol
\headIndex}$ follows the $\softmax_\#$ pushforward of the GP distribution of $\logitSymbol^{\depthSymbol}$ described in \Cref{thm:gp_convergence_sqrt} if $\scaling = \frac{1}{2}$, or $\tildeLogitN{}{\depthSymbol}{\inputSymbol} = \softmax (\queryStd \keyStd \kerntildef{}{\depthSymbol}{\inputSymbol}{\inputSymbol})$ a.s.\ if $\scaling = 1$ \citep[appendix A]{yang2019v2}.
The desired result could thus be established by application of \Cref{lem:wlln_exch}, averaging over the $\headIndex$ index, if its assumptions hold.

Starting with the exchangeability assumption, note that if we condition on $\indexedActivity{\depthSymbol - 1}{\sequenceVariable}{\inputSymbol}, \indexedActivity{\depthSymbol - 1}{\sequenceVariable}{\inputSymbol'}$, the individual terms are i.i.d.\ because the parameters of individual heads are i.i.d.
Since $\{\tildeLogitN{a c_1}{\depthSymbol \headIndex}{\inputSymbol} \tildeLogitN{b c_2}{\depthSymbol\headIndex}{\inputSymbol'}\}_{\headIndex \geq 1}$ are also i.i.d.\ (see \Cref{thm:gp_convergence_sqrt} for $\scaling  = \frac{1}{2}$, and constancy under $\scaling = 1$), it is also clear that the $\E [ \genericRV_{*, 1} \genericRV_{*, 2} ] = (\E [\genericRV_{*, 1}])^2$ is satisfied.
All that remains is to show $\limsup_{\rowIndex \to \infty} \E | \genericRV_{\rowIndex, 1} |^{2 + \varepsilon} < \infty$, and where we will use $\varepsilon = 2$ for convenience.
By H{\" o}lder's inequality
\begin{align*}
    \E \left\{
        \left[
            \tildeLogitN{\sequenceVariable, \fIndexA c_1}{\depthSymbol 1}{\inputSymbol}
            \tildeLogitN{\sequenceVariable, \fIndexB c_2}{\depthSymbol 1}{\inputSymbol'}
            \vphantom{
                \frac{
                    \langle
                        \val{\sequenceVariable, c_1 \cdot}{\depthSymbol 1}(\inputSymbol) 
                        ,
                        \val{\sequenceVariable, c_2 \cdot}{\depthSymbol 1}(\inputSymbol')
                    \rangle
                }{
                    \valueDimension
                }
            }
            \right.\right.
            &
            \left.\left.
            \frac{
                \langle
                    \val{\sequenceVariable, c_1 \cdot}{\depthSymbol 1}(\inputSymbol) 
                    ,
                    \val{\sequenceVariable, c_2 \cdot}{\depthSymbol 1}(\inputSymbol')
                \rangle
            }{
                \valueDimension
            }
        \right]^4
    \right\}
    \\
    &\lesssim
    \poly \biggl(
        \max_{c, c' \in [\spatialDimension], z \in \{\inputSymbol, \inputSymbol'\}}
            \E
                |
                    \tildeLogitN{\sequenceVariable, c c'}{\depthSymbol 1}{z} 
                |^{16} 
        ,
        \max_{c \in [\spatialDimension], z \in \{\inputSymbol, \inputSymbol'\}}
            \E |
                    \indexedActivity{\depthSymbol -1}{\sequenceVariable, c 1}{z}
            |^{16}
    \biggr)
    \, ,
\end{align*}
where we used the assumed exchangeability of $\indexedActivity{\depthSymbol-1}{\sequenceVariable}{z}$ over its columns.
Application of \Cref{lem:mmnt_propagation} implies that the above can be bounded by a constant independent of $\sequenceVariable$, implying all assumptions of \Cref{lem:wlln_exch} are satisfied.
\end{proof}

\begin{lemma}\label{lem:wv_ntk}
$
    \sum_{\headIndex = 1}^{\headDimension}
        \frac{
            \partial \indexedActivation{\depthSymbol}{\sequenceVariable, \fIndexA i}{\inputSymbol}
        }{
            \partial \uW_\sequenceVariable^{\depthSymbol\headIndex, \valueSymbol}
        }
        \frac{
            \partial \indexedActivation{\depthSymbol}{\sequenceVariable, \fIndexB j}{\inputSymbol'}
        }{
            \partial \uW_\sequenceVariable^{\depthSymbol\headIndex, \valueSymbol}
        }^\top
    \convergeProb
    \delta_{i = j}
    \kernelf{a b}{\depthSymbol}{\inputSymbol}{\inputSymbol'}
    \, .
$
\end{lemma}

\begin{proof}[Proof of \Cref{lem:wv_ntk}]
    Note that
    \begin{align*}
        \sum_{\headIndex = 1}^{\headDimension}
            \frac{
                \partial \indexedActivation{\depthSymbol}{\sequenceVariable, \fIndexA i}{\inputSymbol}
            }{
                \partial \uW_\sequenceVariable^{\depthSymbol\headIndex, \valueSymbol}
            }
            \frac{
                \partial \indexedActivation{\depthSymbol}{\sequenceVariable, \fIndexB j}{\inputSymbol'}
            }{
                \partial \uW_\sequenceVariable^{\depthSymbol\headIndex, \valueSymbol}
            }^\top
        &=
        \frac{\outVar}{\headDimension \valueDimension}
        \sum_{\headIndex = 1}^{\headDimension}
        \sum_{k = 1}^{\valueDimension}
            \uW_{\sequenceVariable, k i}^{\depthSymbol\headIndex, \outputSymbol}
            \uW_{\sequenceVariable, k j}^{\depthSymbol\headIndex, \outputSymbol}
            \frac{\valueVar}{\layerDimension{\depthSymbol - 1}}
            \left\langle
                \tildeLogitN{\sequenceVariable, \fIndexA \cdot}{\depthSymbol\headIndex}{\inputSymbol}
                \indexedActivity{\depthSymbol - 1}{}{\inputSymbol}
                ,
                \tildeLogitN{\sequenceVariable, \fIndexB \cdot}{\depthSymbol\headIndex}{\inputSymbol}
                \indexedActivity{\depthSymbol - 1}{}{\inputSymbol'}
            \right\rangle
        \\
        &=
        \frac{\outVar\valueVar}{\headDimension \valueDimension}
        \sum_{\headIndex, k}
        \sum_{c_1, c_2 = 1}^{\spatialDimension}
            \uW_{\sequenceVariable, k i}^{\depthSymbol\headIndex, \outputSymbol}
            \uW_{\sequenceVariable, k j}^{\depthSymbol\headIndex, \outputSymbol}
            \tildeLogitN{\sequenceVariable, \fIndexA c_1}{\depthSymbol\headIndex}{\inputSymbol}
            \tildeLogitN{\sequenceVariable, \fIndexB c_2}{\depthSymbol\headIndex}{\inputSymbol'}
            \frac{
                \langle
                    \indexedActivity{\depthSymbol - 1}{c_1 \cdot}{\inputSymbol}
                    ,
                    \indexedActivity{\depthSymbol - 1}{c_2 \cdot}{\inputSymbol'}
                \rangle
            }{
                \layerDimension{\depthSymbol - 1}
            }
        \, .
    \end{align*}
    Since $\spatialDimension$ is fixed, we can focus on an arbitrary $c_1, c_2 \in [\spatialDimension]$.
    Notice that by the assumed independence of the entries of $\uW_{\sequenceVariable}^{\depthSymbol\headIndex, \outputSymbol}$, the continuous mapping theorem and \Cref{lem:inner_prod_converge,lem:slutsky}, the individual summands converge in distribution
    $$
        \uW_{\sequenceVariable, k i}^{\depthSymbol\headIndex, \outputSymbol}
        \uW_{\sequenceVariable, k j}^{\depthSymbol\headIndex, \outputSymbol}
        \tildeLogitN{\sequenceVariable, \fIndexA c_1}{\depthSymbol\headIndex}{\inputSymbol}
        \tildeLogitN{\sequenceVariable, \fIndexB c_2}{\depthSymbol\headIndex}{\inputSymbol'}
        \frac{
            \langle
                \indexedActivity{\depthSymbol - 1}{c_1 \cdot}{\inputSymbol}
                ,
                \indexedActivity{\depthSymbol - 1}{c_2 \cdot}{\inputSymbol'}
            \rangle
        }{
            \layerDimension{\depthSymbol - 1}
        }
        \convergeDist
        \uW_{\sequenceVariable, k i}^{\depthSymbol\headIndex, \outputSymbol}
        \uW_{\sequenceVariable, k j}^{\depthSymbol\headIndex, \outputSymbol}
        \tildeLogitN{\fIndexA c_1}{\depthSymbol\headIndex}{\inputSymbol}
        \tildeLogitN{\fIndexB c_2}{\depthSymbol\headIndex}{\inputSymbol'}
        \kerntildef{c_1 c_2}{\depthSymbol}{\inputSymbol}{\inputSymbol'}
        \, ,
    $$
    with the distribution of
    $\tildeLogitN{}{\depthSymbol\headIndex}{\inputSymbol}$
    as in the proof of \Cref{lem:wo_ntk}.
    The desired result can thus again be obtained by applying \Cref{lem:wlln_exch}, averaging over $\headIndex$ and $k$, if its assumptions hold.
    As $\E [\uW_{\sequenceVariable , k i}^{\depthSymbol\headIndex, \outputSymbol} \uW_{\sequenceVariable, k j}^{\depthSymbol\headIndex, \outputSymbol}] = \delta_{i = j}$ and $\E |\uW_{\sequenceVariable , k i}^{\depthSymbol\headIndex, \outputSymbol}|^t < \infty$ for any $t \geq 1$, the same argument as in \Cref{lem:wo_ntk} applies.
\end{proof}

\begin{lemma}\label{lem:wqk_ntk}
$
    \sum_{\headIndex = 1}^{\headDimension}
        \frac{
            \partial \indexedActivation{\depthSymbol}{\sequenceVariable, \fIndexA i}{\inputSymbol}
        }{
            \partial \uW_\sequenceVariable^{\depthSymbol\headIndex, \querySymbol}
        }
        \frac{
            \partial \indexedActivation{\depthSymbol}{\sequenceVariable, \fIndexB j}{\inputSymbol'}
        }{
            \partial \uW_\sequenceVariable^{\depthSymbol\headIndex, \querySymbol}
        }^\top
        +
        \frac{
            \partial \indexedActivation{\depthSymbol}{\sequenceVariable, \fIndexA i}{\inputSymbol}
        }{
            \partial \uW_\sequenceVariable^{\depthSymbol\headIndex, \keySymbol}
        }
        \frac{
            \partial \indexedActivation{\depthSymbol}{\sequenceVariable, \fIndexB j}{\inputSymbol'}
        }{
            \partial \uW_\sequenceVariable^{\depthSymbol\headIndex, \keySymbol}
        }^\top
    $
    converges in probability to
    $$
    \delta_{i = j}
    \delta_{\scaling = \frac{1}{2}}
    2
    \OVVar
    \QKVar
    \kerntildef{a b}{\depthSymbol}{\inputSymbol}{\inputSymbol'}
    \sum_{\substack{c_1, c_2 \\ d_1 , d_2}}^{\spatialDimension}
        \kerntildef{c_1 c_2}{\depthSymbol}{\inputSymbol}{\inputSymbol'}
        \kerntildef{d_1 d_2}{
        \depthSymbol}{\inputSymbol}{\inputSymbol'}
        \E \left[
            \frac{
                \partial
                \tildeLogitN{\fIndexA c_1}{\depthSymbol 1}{\inputSymbol}
            }{
                \partial
                \logitSymbol_{\fIndexA d_1}^{\depthSymbol 1}(\inputSymbol)
            }
            \frac{
                \partial
                \tildeLogitN{\fIndexB c_2}{\depthSymbol 1}{\inputSymbol'}
            }{
                \partial
                \logitSymbol_{\fIndexB d_2}^{\depthSymbol 1}(\inputSymbol')
            }
        \right]
    \, .
    $$
\end{lemma}

\begin{proof}[Proof of \Cref{lem:wqk_ntk}]
    By symmetry, it is sufficient to prove convergence for the gradients w.r.t.\ $\uW_\sequenceVariable^{\depthSymbol\headIndex, \keySymbol}$.
    Observe
    \begin{align*}
        &\sum_{\headIndex = 1}^{\headDimension}
            \frac{
                \partial \indexedActivation{\depthSymbol}{\sequenceVariable, \fIndexA i}{\inputSymbol}
            }{
                \partial \uW_\sequenceVariable^{\depthSymbol\headIndex, \keySymbol}
            }
            \frac{
                \partial \indexedActivation{\depthSymbol}{\sequenceVariable, \fIndexB j}{\inputSymbol'}
            }{
                \partial \uW_\sequenceVariable^{\depthSymbol\headIndex, \keySymbol}
            }^\top
        \\
        &=
        \frac{\outVar}{\headDimension \valueDimension}
        \sum_{\headIndex = 1}^{\headDimension}
        \sum_{k_1, k_2 = 1}^{\valueDimension}
        \sum_{\substack{c_1, c_2 \\ d_1, d_2}}^{\spatialDimension}
            \uW_{\sequenceVariable, k_1 i}^{\depthSymbol\headIndex, \outputSymbol}
            \uW_{\sequenceVariable, k_2 j}^{\depthSymbol\headIndex, \outputSymbol}
            \val{\sequenceVariable, c_1 k_1}{\depthSymbol\headIndex} (\inputSymbol)
            \val{\sequenceVariable, c_2 k_2}{\depthSymbol\headIndex} (\inputSymbol')
            \frac{
                \partial
                \tildeLogitN{\sequenceVariable, \fIndexA c_1}{\depthSymbol\headIndex}{\inputSymbol}
            }{
                \partial
                \logitSymbol_{\sequenceVariable, \fIndexA d_1}^{\depthSymbol\headIndex}(\inputSymbol)
            }
            \frac{
                \partial
                \tildeLogitN{\sequenceVariable, \fIndexB c_2}{\depthSymbol\headIndex}{\inputSymbol'}
            }{
                \partial
                \logitSymbol_{\sequenceVariable, \fIndexB d_2}^{\depthSymbol\headIndex}(\inputSymbol')
            }
            \frac{
                \partial
                \logitSymbol_{\sequenceVariable, \fIndexA d_1}^{\depthSymbol\headIndex}(\inputSymbol)
            }{
                \partial \uW_\sequenceVariable^{\depthSymbol\headIndex, \keySymbol}
            }
            \frac{
                \partial
                \logitSymbol_{\sequenceVariable, \fIndexB d_2}^{\depthSymbol\headIndex}(\inputSymbol')
            }{
                \partial \uW_\sequenceVariable^{\depthSymbol\headIndex, \keySymbol}
            }^\top
        \, .
    \end{align*}
    Since $\spatialDimension$ is fixed, we can focus on arbitrary $c_1, c_2, d_1, d_2 \in [\spatialDimension]$.
    Rewriting the r.h.s.\ above for one such choice, we obtain
    \begin{align*}
        \frac{\outVar \keyVar}{\headDimension \valueDimension}
        \! \!
        \sum_{\headIndex , k_1, k_2}
            \! \!
            \uW_{\sequenceVariable, k_1 i}^{\depthSymbol\headIndex, \outputSymbol}
            \uW_{\sequenceVariable, k_2 j}^{\depthSymbol\headIndex, \outputSymbol}
            \val{\sequenceVariable, c_1 k_1}{\depthSymbol\headIndex} (\inputSymbol)
            \val{\sequenceVariable, c_2 k_2}{\depthSymbol\headIndex} (\inputSymbol')
            \frac{
                \partial
                \tildeLogitN{\sequenceVariable, \fIndexA c_1}{\depthSymbol\headIndex}{\inputSymbol}
            }{
                \partial
                \logitSymbol_{\sequenceVariable, \fIndexA d_1}^{\depthSymbol\headIndex}(\inputSymbol)
            }
            \frac{
                \partial
                \tildeLogitN{\sequenceVariable, \fIndexB c_2}{\depthSymbol\headIndex}{\inputSymbol'}
            }{
                \partial
                \logitSymbol_{\sequenceVariable, \fIndexB d_2}^{\depthSymbol\headIndex}(\inputSymbol')
            }
            \frac{
                \langle
                    \query{\sequenceVariable, \fIndexA \cdot}{\depthSymbol\headIndex} (\inputSymbol)
                    ,
                    \query{\sequenceVariable, \fIndexB \cdot}{\depthSymbol\headIndex} (\inputSymbol')
                \rangle
            }{
                (\logitDimension)^{2\scaling}
            }
            \frac{
                \langle
                    \indexedActivity{\depthSymbol - 1}{d_1 \cdot}{\inputSymbol}
                    ,
                    \indexedActivity{\depthSymbol - 1}{d_2 \cdot}{\inputSymbol'}
                \rangle
            }{
                \layerDimension{\depthSymbol - 1}
            }
        \, .
    \end{align*}
    Noting that 
    $
            \langle
                \indexedActivity{\depthSymbol - 1}{d_1 \cdot}{\inputSymbol}
                ,
                \indexedActivity{\depthSymbol - 1}{d_2 \cdot}{\inputSymbol'}
            \rangle
        /
            \layerDimension{\depthSymbol - 1}
    $
    only depends on the spatial dimension indices $d_1$ and $d_2$, we can use \Cref{lem:inner_prod_converge} to establish it converges in probability to $\kerntildef{d_1 d_2}{\depthSymbol}{\inputSymbol}{\inputSymbol'}$, implying that we only need to prove that the rest of the terms in the above sum also converges in probability.
    Let
    \begin{align*}
        \generalMean_{\sequenceVariable}
        =
        \frac{1}{\headDimension \valueDimension}
        \sum_{\headIndex, k_1, k_2}
            \uW_{\sequenceVariable, k_1 i}^{\depthSymbol\headIndex, \outputSymbol}
            \uW_{\sequenceVariable, k_2 j}^{\depthSymbol\headIndex, \outputSymbol}
            \val{\sequenceVariable, c_1 k_1}{\depthSymbol\headIndex} (\inputSymbol)
            \val{\sequenceVariable, c_2 k_2}{\depthSymbol\headIndex} (\inputSymbol')
            \frac{
                \partial
                \tildeLogitN{\sequenceVariable, \fIndexA c_1}{\depthSymbol\headIndex}{\inputSymbol}
            }{
                \partial
                \logitSymbol_{\sequenceVariable, \fIndexA d_1}^{\depthSymbol\headIndex}(\inputSymbol)
            }
            \frac{
                \partial
                \tildeLogitN{\sequenceVariable, \fIndexB c_2}{\depthSymbol\headIndex}{\inputSymbol'}
            }{
                \partial
                \logitSymbol_{\sequenceVariable, \fIndexB d_2}^{\depthSymbol\headIndex}(\inputSymbol')
            }
            \frac{
                \langle
                    \query{\sequenceVariable, \fIndexA \cdot}{\depthSymbol\headIndex} (\inputSymbol)
                    ,
                    \query{\sequenceVariable, \fIndexB \cdot}{\depthSymbol\headIndex} (\inputSymbol')
                \rangle
            }{
                (\logitDimension)^{2\scaling}
            }
        \, ,
    \end{align*}
    and note that 
    $
        \E [\generalMean_{\sequenceVariable}]
        =
        \delta_{i = j} 
        \E \left[
            \frac{
                \langle
                    \query{\sequenceVariable, \fIndexA \cdot}{\depthSymbol 1} (\inputSymbol)
                    ,
                    \query{\sequenceVariable, \fIndexB \cdot}{\depthSymbol 1} (\inputSymbol')
                \rangle
            }{
                (\logitDimension)^{2\scaling}
            }
            \frac{
                \langle
                    \indexedActivity{\depthSymbol - 1}{\sequenceVariable, c_1 \cdot}{\inputSymbol}
                    ,
                    \indexedActivity{\depthSymbol - 1}{\sequenceVariable, c_2 \cdot}{\inputSymbol'}
                \rangle
            }{
                \layerDimension{\depthSymbol - 1}
            }
            \frac{
                \partial
                \tildeLogitN{\sequenceVariable, \fIndexA c_1}{\depthSymbol 1}{\inputSymbol}
            }{
                \partial
                \logitSymbol_{\sequenceVariable, \fIndexA d_1}^{\depthSymbol 1}(\inputSymbol)
            }
            \frac{
                \partial
                \tildeLogitN{\sequenceVariable, \fIndexB c_2}{\depthSymbol 1}{\inputSymbol'}
            }{
                \partial
                \logitSymbol_{\sequenceVariable, \fIndexB d_2}^{\depthSymbol 1}(\inputSymbol')
            }
        \right]
    $
    by exchangeability.
    This suggests that the required result could be obtained using the Chebyshev's inequality
    \begin{align*}
        \Prob (
            | \generalMean_{\sequenceVariable} - \E \generalMean_{n} |
            \geq
            \delta
        )
        \leq
        \frac{\E [ \generalMean_{\sequenceVariable}^2 ] - \{ \E [ \generalMean_{\sequenceVariable} ] \}^2 }{\delta^2}
        \, ,
    \end{align*}
    if $\E [ \generalMean_{\sequenceVariable} ]$ converges to the desired limit.
    To establish this convergence, observe
    \begin{align}\label{eq:wqk_weak_limit_mean}
        \delta_{i = j}
        \frac{
            \langle
                \query{\sequenceVariable, \fIndexA \cdot}{\depthSymbol 1} (\inputSymbol)
                ,
                \query{\sequenceVariable, \fIndexB \cdot}{\depthSymbol 1} (\inputSymbol')
            \rangle
        }{
            (\logitDimension)^{2\scaling}
        }
        \frac{
            \langle
                \indexedActivity{\depthSymbol - 1}{\sequenceVariable, c_1 \cdot}{\inputSymbol}
                ,
                \indexedActivity{\depthSymbol - 1}{\sequenceVariable, c_2 \cdot}{\inputSymbol'}
            \rangle
        }{
            \layerDimension{\depthSymbol - 1}
        }
        &\frac{
            \partial
            \tildeLogitN{\sequenceVariable, \fIndexA c_1}{\depthSymbol 1}{\inputSymbol}
        }{
            \partial
            \logitSymbol_{\sequenceVariable, \fIndexA d_1}^{\depthSymbol 1}(\inputSymbol)
        }
        \frac{
            \partial
            \tildeLogitN{\sequenceVariable, \fIndexB c_2}{\depthSymbol 1}{\inputSymbol'}
        }{
            \partial
            \logitSymbol_{\sequenceVariable, \fIndexB d_2}^{\depthSymbol 1}(\inputSymbol')
        }
        \nonumber
        \\
        \convergeDist
        \delta_{i = j}
        \delta_{\scaling = \frac{1}{2}}
        \queryVar
        \kerntildef{a b}{\depthSymbol}{\inputSymbol}{\inputSymbol'}
        \kerntildef{c_1 c_2}{\depthSymbol}{\inputSymbol}{\inputSymbol'}
        &\frac{
            \partial
            \tildeLogitN{\fIndexA c_1}{\depthSymbol 1}{\inputSymbol}
        }{
            \partial
            \logitSymbol_{\fIndexA d_1}^{\depthSymbol 1}(\inputSymbol)
        }
        \frac{
            \partial
            \tildeLogitN{\fIndexB c_2}{\depthSymbol 1}{\inputSymbol'}
        }{
            \partial
            \logitSymbol_{\fIndexB d_2}^{\depthSymbol 1}(\inputSymbol')
        }
        \, ,
    \end{align}
    since the first two terms converge in probability (\Cref{lem:inner_prod_converge}), and the last converges in distribution by \Cref{thm:gp_convergence_sqrt} and the continuous mapping theorem, implying that the product of all three thus converges in distribution by \Cref{lem:slutsky}.
    Convergence of $\E [\generalMean_{\sequenceVariable}]$ could thus be obtained by establishing uniform integrability of the $(\generalMean_{\sequenceVariable})_{\sequenceVariable \geq 1}$ sequence (\Cref{thm:mean_convergence}).
    
    By \Cref{lem:sup_ui}, uniform integrability of $\generalMean_{\sequenceVariable}$ can be established by showing $\E [ \generalMean_{\sequenceVariable}^2 ] \to \{ \E [ \generalMean_{*} ] \}^2$ which would also imply $\generalMean_{\sequenceVariable} \convergeProb \E [\generalMean_{*}]$ by the above Chebyshev's inequality.
    For the rest of this proof, we drop the $\inputSymbol, \inputSymbol'$ from our equations for brevity; this allows us to write
    \begin{align*}
        \E [\generalMean_{\sequenceVariable}^2]
        =
        \frac{1}{(\headDimension \valueDimension)^2}
        \sum_{\substack{\headIndex_1, \headIndex_2 \\ k_1, k_2, k_3, k_4}}
            \E \left[
                \prod_{t = 0}^{1}
                \uW_{\sequenceVariable, k_{2t + 1} i}^{\depthSymbol\headIndex_{t + 1}, \outputSymbol}
                \uW_{\sequenceVariable, k_{2t + 2} j}^{\depthSymbol\headIndex_{t + 1}, \outputSymbol}
                \val{\sequenceVariable, c_1 k_{2t + 1}}{\depthSymbol\headIndex_{t + 1}}
                \val{\sequenceVariable, c_2 k_{2t + 2}}{\depthSymbol\headIndex_{t + 1}}
                \frac{
                    \langle
                        \query{\sequenceVariable, \fIndexA \cdot}{\depthSymbol\headIndex_{t + 1}}
                        ,
                        \query{\sequenceVariable, \fIndexB \cdot}{\depthSymbol\headIndex_{t+1}}
                    \rangle
                }{
                    (\logitDimension)^{2\scaling}
                }
                \frac{
                    \partial
                    \tildeLogit{\sequenceVariable, \fIndexA c_1}{\depthSymbol\headIndex_{t + 1}}
                }{
                    \partial
                    \logitSymbol_{\sequenceVariable, \fIndexA d_1}^{\depthSymbol\headIndex_{t + 1}}
                }
                \frac{
                    \partial
                    \tildeLogit{\sequenceVariable, \fIndexB c_2}{\depthSymbol\headIndex_{t + 1}}
                }{
                    \partial
                    \logitSymbol_{\sequenceVariable, \fIndexB d_2}^{\depthSymbol\headIndex_{t + 1}}
                }
            \right]
        \, .
    \end{align*}
    From above, we can restrict our attention to groups of terms that include at least $\mathcal{O}((\headDimension \valueDimension)^2)$ of the summands as long as the expectation of the square of each term can be bounded by a constant independent of the $\headIndex, k$ and $\sequenceVariable$ indices.
    \begin{align}\label{eq:wqk_weak_limit_bound}
        \E \left\{
            \left[
                \prod_{t = 0}^{1}
                \right.\right.
                &
                \left.\left.
                \uW_{\sequenceVariable, k_{2t + 1} i}^{\depthSymbol\headIndex_{t + 1}, \outputSymbol}
                \uW_{\sequenceVariable, k_{2t + 2} j}^{\depthSymbol\headIndex_{t + 1}, \outputSymbol}
                \val{\sequenceVariable, c_1 k_{2t + 1}}{\depthSymbol\headIndex_{t + 1}}
                \val{\sequenceVariable, c_2 k_{2t + 2}}{\depthSymbol\headIndex_{t + 1}}
                \frac{
                    \langle
                        \query{\sequenceVariable, \fIndexA \cdot}{\depthSymbol\headIndex_{t + 1}}
                        ,
                        \query{\sequenceVariable, \fIndexB \cdot}{\depthSymbol\headIndex_{t+1}}
                    \rangle
                }{
                    (\logitDimension)^{2\scaling}
                }
                \frac{
                    \partial
                    \tildeLogit{\sequenceVariable, \fIndexA c_1}{\depthSymbol\headIndex_{t + 1}}
                }{
                    \partial
                    \logitSymbol_{\sequenceVariable, \fIndexA d_1}^{\depthSymbol\headIndex_{t + 1}}
                }
                \frac{
                    \partial
                    \tildeLogit{\sequenceVariable, \fIndexB c_2}{\depthSymbol\headIndex_{t + 1}}
                }{
                    \partial
                    \logitSymbol_{\sequenceVariable, \fIndexB d_2}^{\depthSymbol\headIndex_{t + 1}}
                }
            \right]^2
        \right\}
        \nonumber
        \\
        &\lesssim
        \E \left\{
            \left[
                \prod_{t = 0}^{1}
                \val{\sequenceVariable, c_1 k_{2t + 1}}{\depthSymbol\headIndex_{t + 1}}
                \val{\sequenceVariable, c_2 k_{2t + 2}}{\depthSymbol\headIndex_{t + 1}}
                \frac{
                    \langle
                        \query{\sequenceVariable, \fIndexA \cdot}{\depthSymbol\headIndex_{t + 1}}
                        ,
                        \query{\sequenceVariable, \fIndexB \cdot}{\depthSymbol\headIndex_{t+1}}
                    \rangle
                }{
                    (\logitDimension)^{2\scaling}
                }
            \right]^4
        \right\}
        \lesssim
        \poly \left(
            \max_{c \in [\spatialDimension] , z \in \{ \inputSymbol, \inputSymbol' \}}
                \E | \indexedActivity{\depthSymbol - 1}{\sequenceVariable, c 1}{z} |^{16}
        \right)
        \, ,
    \end{align}
    by H{\" o}lder's inequality and exchangeability.
    Application of \Cref{lem:mmnt_propagation} allows us to bound the above r.h.s.\ by a constant independent of $\headIndex, k$ and $\sequenceVariable$ as desired.
    
    We can thus only focus on the terms for which $\headIndex_1 \neq \headIndex_2$.
    Among these, the only ones with non-zero expectation are those where $i = j$, $k_1 = k_2$, and $k_3 = k_4$, contributing to $\E [ \generalMean_{\sequenceVariable}^2 ]$ by
    \begin{align}\label{eq:wqk_square_simplif}
        \delta_{i = j}
        \frac{\outVar \valueVar}{(\headDimension \valueDimension)^2}
        \sum_{\substack{h_1 h_2 \\ k_1, k_2}}
            \E \left[
                \left(
                    \frac{
                        \langle
                            \activitySymbol_{\sequenceVariable, c_1 \cdot}^{\depthSymbol - 1}
                            ,
                            \activitySymbol_{\sequenceVariable, c_2 \cdot}^{\depthSymbol - 1}
                        \rangle
                    }{
                        \layerDimension{\depthSymbol - 1}
                    }
                \right)^2
                \frac{
                    \langle
                        \query{\sequenceVariable, \fIndexA \cdot}{\depthSymbol 1}
                        ,
                        \query{\sequenceVariable, \fIndexB \cdot}{\depthSymbol 1}
                    \rangle
                }{
                    (\logitDimension)^{2\scaling}
                }
                \frac{
                    \langle
                        \query{\sequenceVariable, \fIndexA \cdot}{\depthSymbol 2}
                        ,
                        \query{\sequenceVariable, \fIndexB \cdot}{\depthSymbol 2}
                    \rangle
                }{
                    (\logitDimension)^{2\scaling}
                }
                \frac{
                    \partial
                    \tildeLogit{\sequenceVariable, \fIndexA c_1}{\depthSymbol 1}
                }{
                    \partial
                    \logitSymbol_{\sequenceVariable, \fIndexA d_1}^{\depthSymbol 1}
                }
                \frac{
                    \partial
                    \tildeLogit{\sequenceVariable, \fIndexB c_2}{\depthSymbol 1}
                }{
                    \partial
                    \logitSymbol_{\sequenceVariable, \fIndexB d_2}^{\depthSymbol 1}
                }
                \frac{
                    \partial
                    \tildeLogit{\sequenceVariable, \fIndexA c_1}{\depthSymbol 2}
                }{
                    \partial
                    \logitSymbol_{\sequenceVariable, \fIndexA d_1}^{\depthSymbol 2}
                }
                \frac{
                    \partial
                    \tildeLogit{\sequenceVariable, \fIndexB c_2}{\depthSymbol 2}
                }{
                    \partial
                    \logitSymbol_{\sequenceVariable, \fIndexB d_2}^{\depthSymbol 2}
                }
            \right]
    \end{align}
    by exchangeability.
    Noting that the sum cancels out with the $\headDimension \valueDimension$ terms, we see that the limit of $\E [ \generalMean_{\sequenceVariable}^2 ]$ will be identical to that of \Cref{eq:wqk_square_simplif}.
    Applying \Cref{lem:inner_prod_converge,lem:slutsky}, \Cref{thm:gp_convergence_sqrt} (resp.\ the result by \citet[appendix A]{yang2019v2} if $\scaling = 1$), and the continuous mapping theorem
    \begin{align}\label{eq:wqk_weak_limit_square}
        \delta_{i = j}
        &\left(
            \frac{
                \langle
                    \activitySymbol_{\sequenceVariable, c_1 \cdot}^{\depthSymbol - 1}
                    ,
                    \activitySymbol_{\sequenceVariable, c_2 \cdot}^{\depthSymbol - 1}
                \rangle
            }{
                \layerDimension{\depthSymbol - 1}
            }
        \right)^2
        \frac{
            \langle
                \query{\sequenceVariable, \fIndexA \cdot}{\depthSymbol 1}
                ,
                \query{\sequenceVariable, \fIndexB \cdot}{\depthSymbol 1}
            \rangle
        }{
            (\logitDimension)^{2\scaling}
        }
        \frac{
            \langle
                \query{\sequenceVariable, \fIndexA \cdot}{\depthSymbol 2}
                ,
                \query{\sequenceVariable, \fIndexB \cdot}{\depthSymbol 2}
            \rangle
        }{
            (\logitDimension)^{2\scaling}
        }
        \frac{
            \partial
            \tildeLogit{\sequenceVariable, \fIndexA c_1}{\depthSymbol 1}
        }{
            \partial
            \logitSymbol_{\sequenceVariable, \fIndexA d_1}^{\depthSymbol 1}
        }
        \frac{
            \partial
            \tildeLogit{\sequenceVariable, \fIndexB c_2}{\depthSymbol 1}
        }{
            \partial
            \logitSymbol_{\sequenceVariable, \fIndexB d_2}^{\depthSymbol 1}
        }
        \frac{
            \partial
            \tildeLogit{\sequenceVariable, \fIndexA c_1}{\depthSymbol 2}
        }{
            \partial
            \logitSymbol_{\sequenceVariable, \fIndexA d_1}^{\depthSymbol 2}
        }
        \frac{
            \partial
            \tildeLogit{\sequenceVariable, \fIndexB c_2}{\depthSymbol 2}
        }{
            \partial
            \logitSymbol_{\sequenceVariable, \fIndexB d_2}^{\depthSymbol 2}
        }
        \nonumber
        \\
        &\convergeDist
        \delta_{i = j}
        \delta_{\scaling = \frac{1}{2}}
        \queryStd^4
        [\kerntildef{a b}{\depthSymbol}{\inputSymbol}{\inputSymbol'}]^2
        [\kerntildef{c_1 c_2}{\depthSymbol}{\inputSymbol}{\inputSymbol'}]^2
        \frac{
            \partial
            \tildeLogit{\fIndexA c_1}{\depthSymbol 1}
        }{
            \partial
            \logitSymbol_{\fIndexA d_1}^{\depthSymbol 1}
        }
        \frac{
            \partial
            \tildeLogit{\fIndexB c_2}{\depthSymbol 1}
        }{
            \partial
            \logitSymbol_{\fIndexB d_2}^{\depthSymbol 1}
        }
        \frac{
            \partial
            \tildeLogit{\fIndexA c_1}{\depthSymbol 2}
        }{
            \partial
            \logitSymbol_{\fIndexA d_1}^{\depthSymbol 2}
        }
        \frac{
            \partial
            \tildeLogit{\fIndexB c_2}{\depthSymbol 2}
        }{
            \partial
            \logitSymbol_{\fIndexB d_2}^{\depthSymbol 2}
        }
        \, ,
    \end{align}
    where $\frac{\partial \tildeLogit{}{\depthSymbol \headIndex}}{\partial \logitSymbol^{\depthSymbol \headIndex}}$
    follows the $(\nabla \softmax)_{\#}$ pushforward of the GP distribution of $\logitSymbol^{\depthSymbol}$ described in \Cref{thm:gp_convergence_sqrt} if $\scaling = \frac{1}{2}$, and is a.s.\ constant if $\scaling = 1$ as the limit $\tildeLogit{}{\depthSymbol \headIndex}$ is a.s.\ constant \citep[appendix A]{yang2019v2}, both by the assumed continuity of $\nabla \softmax$.
    
    Finally, because \Cref{eq:wqk_weak_limit_bound} establishes uniform integrability, and $\frac{\partial \tildeLogit{}{\depthSymbol 1}}{\partial \logitSymbol^{\depthSymbol 1}}$ is independent of $\frac{\partial \tildeLogit{}{\depthSymbol 2}}{\partial \logitSymbol^{\depthSymbol 2}}$ by \Cref{thm:gp_convergence_sqrt}, we can combine \Cref{eq:wqk_weak_limit_mean,eq:wqk_weak_limit_square} with \Cref{thm:mean_convergence} to conclude that both $\E [ \generalMean_{\sequenceVariable}^2 ]$ and $\{ \E [ \generalMean_{\sequenceVariable} ] \}^2$ converge to the same limit.
\end{proof}

\subsubsection{Indirect contribution}

The indirect contribution of an attention layer can be expanded as
\begin{align*}
    \frac{
        \partial \indexedActivation{\depthSymbol}{\sequenceVariable, \fIndexA i}{\inputSymbol}
    }{
        \partial \indexedActivity{\depthSymbol - 1}{\sequenceVariable}{\inputSymbol}
    }
    \frac{
        \partial \indexedActivity{\depthSymbol - 1}{\sequenceVariable}{\inputSymbol}
    }{
        \partial \params_{\sequenceVariable}^{< \depthSymbol}
    }
        \frac{
            \partial \indexedActivity{\depthSymbol - 1}{\sequenceVariable}{\inputSymbol'}
        }{
            \partial \params_{\sequenceVariable}^{< \depthSymbol}
        }^\top
        \frac{
            \partial \indexedActivation{\depthSymbol}{\sequenceVariable, \fIndexB j}{\inputSymbol'}
        }{
            \partial \indexedActivity{\depthSymbol - 1}{\sequenceVariable}{\inputSymbol'}
        }^\top
    =
    \sum_{\gIndexA, \gIndexB = 1}^{\spatialDimension}
    \sum_{\gIndexZA, \gIndexZB = 1}^{\layerDimension{\depthSymbol - 1}}
        \ntkhatf{
            \gIndexA \gIndexZA , \gIndexB \gIndexZB
        }{\depthSymbol}{\inputSymbol}{\inputSymbol'}
        \frac{
            \partial \indexedActivation{\depthSymbol}{\sequenceVariable, \fIndexA i}{\inputSymbol}
        }{
            \partial \indexedActivity{\depthSymbol - 1}{\sequenceVariable, \gIndexA \gIndexZA}{\inputSymbol}
        }
        \frac{
            \partial \indexedActivation{\depthSymbol}{\sequenceVariable, \fIndexB j}{\inputSymbol'}
        }{
            \partial \indexedActivity{\depthSymbol - 1}{\sequenceVariable, \gIndexB \gIndexZB}{\inputSymbol'}
        }
    \, ,
\end{align*}
where\footnote{$\ntkhat$ should technically also be subscripted with $\sequenceVariable$ as all other variables dependent on the $\params_{\sequenceVariable}^{\leq \depthSymbol}$; we make an exception here and omit this from our notation as the number of subscripts of $\ntkhat$ is already high.}
\begin{align}\label{eq:ntk_hat}
    \ntkhatf{
        \gIndexA \gIndexZA , \gIndexB \gIndexZB
    }{\depthSymbol}{\inputSymbol}{\inputSymbol'}
    \coloneqq
    \left\langle
        \frac{
            \partial \indexedActivity{\depthSymbol - 1}{\sequenceVariable, \gIndexA \gIndexZA}{\inputSymbol}
        }{
            \partial \params_{\sequenceVariable}^{< \depthSymbol}
        }
        ,
        \frac{
            \partial \indexedActivity{\depthSymbol - 1}{\sequenceVariable, \gIndexB \gIndexZB}{\inputSymbol'}
        }{
            \partial \params_{\sequenceVariable}^{< \depthSymbol}
        }
    \right\rangle
    \, ,
\end{align}
which we know converges a.s., and thus also in probability, to $\delta_{\gIndexZA = \gIndexZB} \ntktildef{\gIndexA \gIndexB}{\depthSymbol}{\inputSymbol}{\inputSymbol'}$ for architectures without attention layers \citep{yang2019v2}.
Expanding the indirect contribution further
\begin{align*}
    &\sum_{\gIndexA, \gIndexB}
    \sum_{\gIndexZA, \gIndexZB}
        \ntkhatf{
            \gIndexA \gIndexZA , \gIndexB \gIndexZB
        }{\depthSymbol}{\inputSymbol}{\inputSymbol'}
        \frac{
            \partial \indexedActivation{\depthSymbol}{\sequenceVariable, \fIndexA i}{\inputSymbol}
        }{
            \partial \indexedActivity{\depthSymbol - 1}{\sequenceVariable, \gIndexA \gIndexZA}{\inputSymbol}
        }
        \frac{
            \partial \indexedActivation{\depthSymbol}{\sequenceVariable, \fIndexB j}{\inputSymbol'}
        }{
            \partial \indexedActivity{\depthSymbol - 1}{\sequenceVariable, \gIndexB \gIndexZB}{\inputSymbol'}
        }
    \\
    &=
    \sum_{\gIndexA, \gIndexB}
    \sum_{\gIndexZA, \gIndexZB}
        \ntkhatf{
            \gIndexA \gIndexZA , \gIndexB \gIndexZB
        }{\depthSymbol}{\inputSymbol}{\inputSymbol'}
        \frac{\outVar}{\headDimension \valueDimension}
        \sum_{\headIndex_1, \headIndex_2 = 1}^{\headDimension}
        \sum_{k_1, k_2 = 1}^{\valueDimension}
            \uW_{\sequenceVariable , k_1 i}^{\depthSymbol \headIndex_1 , \outputSymbol}
            \uW_{\sequenceVariable , k_2 j}^{\depthSymbol \headIndex_2 , \outputSymbol}
            \frac{
                \partial
                \tildeLogitN{\sequenceVariable , \fIndexA \cdot}{\depthSymbol \headIndex_1}{\inputSymbol}
                \val{\sequenceVariable, \cdot k_1}{\depthSymbol \headIndex_1} (\inputSymbol)
            }{
                \partial \indexedActivity{\depthSymbol - 1}{\sequenceVariable, \gIndexA \gIndexZA}{\inputSymbol}
            }
            \frac{
                \partial
                \tildeLogitN{\sequenceVariable , \fIndexB \cdot}{\depthSymbol \headIndex_1}{\inputSymbol}
                \val{\sequenceVariable, \cdot k_2}{\depthSymbol \headIndex_2} (\inputSymbol')
            }{
                \partial \indexedActivity{\depthSymbol - 1}{\sequenceVariable, \gIndexB \gIndexZB}{\inputSymbol'}
            }
    \, .
\end{align*}
In the rest of this section, we drop the $\inputSymbol, \inputSymbol'$ from most of our equations so as to reduce the number of multi-line expressions.
Continuing with the inner sum from above we obtain
\begin{align*}
    &\frac{\outVar}{\headDimension \valueDimension}
    \sum_{\substack{\headIndex_1, \headIndex_2\\ k_1, k_2}}
        \uW_{\sequenceVariable , k_1 i}^{\depthSymbol \headIndex_1 , \outputSymbol}
        \uW_{\sequenceVariable , k_2 j}^{\depthSymbol \headIndex_2 , \outputSymbol}
        \frac{
            \partial
            \tildeLogit{\sequenceVariable , \fIndexA \cdot}{\depthSymbol \headIndex_1}
            \val{\sequenceVariable, \cdot k_1}{\depthSymbol \headIndex_1}
        }{
            \partial \activitySymbol_{\sequenceVariable, \gIndexA \gIndexZA}^{\depthSymbol - 1}
        }
        \frac{
            \partial
            \tildeLogit{\sequenceVariable , \fIndexB \cdot}{\depthSymbol \headIndex_1}
            \val{\sequenceVariable, \cdot k_2}{\depthSymbol \headIndex_2}
        }{
            \partial \activitySymbol_{\sequenceVariable, \gIndexB \gIndexZB}^{\depthSymbol - 1}
        }
    \\
    &=
    \frac{\outVar}{\headDimension \valueDimension}
    \sum_{\substack{\headIndex_1, \headIndex_2\\ k_1, k_2}}
        \uW_{\sequenceVariable , k_1 i}^{\depthSymbol \headIndex_1 , \outputSymbol}
        \uW_{\sequenceVariable , k_2 j}^{\depthSymbol \headIndex_2 , \outputSymbol}
        \begin{aligned}[t]
            &\biggl(
                \sum_{c_1=1}^{\spatialDimension}
                    \tildeLogit{\sequenceVariable, \fIndexA c_1}{\depthSymbol\headIndex_1}
                    \frac{
                        \partial
                        \val{\sequenceVariable, c_1 k_1}{\depthSymbol \headIndex_1}
                    }{
                        \partial \activitySymbol_{\sequenceVariable, \gIndexA \gIndexZA}^{\depthSymbol - 1}
                    }
                    +
                    \val{\sequenceVariable, c_1 k_1}{\depthSymbol \headIndex_1}
                    \sum_{d_1 = 1}^{\spatialDimension}
                        \frac{
                            \partial
                            \tildeLogit{\sequenceVariable, \fIndexA c_1}{\depthSymbol\headIndex_1}
                        }{
                            \partial 
                            \logitSymbol_{\sequenceVariable, \fIndexA d_1}^{\depthSymbol\headIndex_1}
                        }
                        \frac{
                            \partial 
                            \logitSymbol_{\sequenceVariable, \fIndexA d_1}^{\depthSymbol\headIndex_1}
                        }{
                            \partial \activitySymbol_{\sequenceVariable, \gIndexA \gIndexZA}^{\depthSymbol - 1}
                        }
            \biggr)
            \cdot
            \\
            &\biggl(
                \sum_{c_2=1}^{\spatialDimension}
                    \tildeLogit{\sequenceVariable, \fIndexB c_2}{\depthSymbol\headIndex_2}
                    \frac{
                        \partial
                        \val{\sequenceVariable, c_2 k_2}{\depthSymbol \headIndex_2}
                    }{
                        \partial \activitySymbol_{\sequenceVariable, \gIndexB \gIndexZB}^{\depthSymbol - 1}
                    }
                    +
                    \val{\sequenceVariable, c_2 k_2}{\depthSymbol \headIndex_2}
                    \sum_{d_2 = 1}^{\spatialDimension}
                        \frac{
                            \partial
                            \tildeLogit{\sequenceVariable, \fIndexB c_2}{\depthSymbol\headIndex_2}
                        }{
                            \partial 
                            \logitSymbol_{\sequenceVariable, \fIndexB d_2}^{\depthSymbol\headIndex_2}
                        }
                        \frac{
                            \partial 
                            \logitSymbol_{\sequenceVariable, \fIndexB d_2}^{\depthSymbol\headIndex_2}
                        }{
                            \partial \activitySymbol_{\sequenceVariable, \gIndexB \gIndexZB}^{\depthSymbol - 1}
                        }
            \biggr)
            \, ,
        \end{aligned}
\end{align*}
which gives us four sums after multiplying out the terms inside the parenthesis, for each of which we prove convergence separately.
Since the spatial dimension $\spatialDimension$ does not change with $\sequenceVariable$, we will restrict our attention to an arbitrary fixed choice of $\gIndexA, \gIndexB, c_1, c_2, d_1, d_2 \in [\spatialDimension]$ throughout.

\begin{lemma}\label{lem:gg_ntk}
    $
    \frac{\outVar}{\headDimension \valueDimension}
    \sum_{\gIndexZA, \gIndexZB}
    \sum_{\substack{\headIndex_1, \headIndex_2\\ k_1, k_2}}
        \ntkhat_{
            \gIndexA \gIndexZA , \gIndexB \gIndexZB
        }^{\depthSymbol}
        \uW_{\sequenceVariable , k_1 i}^{\depthSymbol \headIndex_1 , \outputSymbol}
        \uW_{\sequenceVariable , k_2 j}^{\depthSymbol \headIndex_2 , \outputSymbol}
        \tildeLogit{\sequenceVariable, \fIndexA c_1}{\depthSymbol\headIndex_1}
        \tildeLogit{\sequenceVariable, \fIndexB c_2}{\depthSymbol\headIndex_2}
        \frac{
            \partial
            \val{\sequenceVariable, c_1 k_1}{\depthSymbol \headIndex_1}
        }{
            \partial \activitySymbol_{\sequenceVariable, \gIndexA \gIndexZA}^{\depthSymbol - 1}
        }
        \frac{
            \partial
            \val{\sequenceVariable, c_2 k_2}{\depthSymbol \headIndex_2}
        }{
            \partial \activitySymbol_{\sequenceVariable, \gIndexB \gIndexZB}^{\depthSymbol - 1}
        }
    $
    converges in probability to
    \begin{align*}
        \delta_{i = j}
        \delta_{\substack{c_1 = \gIndexA \\ c_2 = \gIndexB}}
        \OVVar
        \ntktildef{\gIndexA \gIndexB}{\depthSymbol}{\inputSymbol}{\inputSymbol'}
        \E [
            \tildeLogit{\fIndexA c_1}{\depthSymbol 1} (\inputSymbol)
            \tildeLogit{\fIndexB c_2}{\depthSymbol 1} (\inputSymbol')
        ]
        \, ,
    \end{align*}
\end{lemma}

\begin{proof}[Proof of \Cref{lem:gg_ntk}]
    Note that
    \begin{align*}
        &\frac{1}{\headDimension \valueDimension}
        \sum_{\gIndexZA, \gIndexZB}
        \sum_{\substack{\headIndex_1, \headIndex_2\\ k_1, k_2}}
            \ntkhat_{
                \gIndexA \gIndexZA , \gIndexB \gIndexZB
            }^{\depthSymbol}
            \uW_{\sequenceVariable , k_1 i}^{\depthSymbol \headIndex_1 , \outputSymbol}
            \uW_{\sequenceVariable , k_2 j}^{\depthSymbol \headIndex_2 , \outputSymbol}
            \tildeLogit{\sequenceVariable, \fIndexA c_1}{\depthSymbol\headIndex_1}
            \tildeLogit{\sequenceVariable, \fIndexB c_2}{\depthSymbol\headIndex_2}
            \frac{
                \partial
                \val{\sequenceVariable, c_1 k_1}{\depthSymbol \headIndex_1}
            }{
                \partial \activitySymbol_{\sequenceVariable, \gIndexA \gIndexZA}^{\depthSymbol - 1}
            }
            \frac{
                \partial
                \val{\sequenceVariable, c_2 k_2}{\depthSymbol \headIndex_2}
            }{
                \partial \activitySymbol_{\sequenceVariable, \gIndexB \gIndexZB}^{\depthSymbol - 1}
            }
        \\
        &=
        \delta_{\substack{c_1 = \gIndexA \\ c_2 = \gIndexB}}
        \frac{\valueVar}{\headDimension \valueDimension \layerDimension{\depthSymbol - 1}}
        \sum_{\gIndexZA, \gIndexZB}
        \sum_{\substack{\headIndex_1, \headIndex_2\\ k_1, k_2}}
            \ntkhat_{
                \gIndexA \gIndexZA , \gIndexB \gIndexZB
            }^{\depthSymbol}
            \uW_{\sequenceVariable , k_1 i}^{\depthSymbol \headIndex_1 , \outputSymbol}
            \uW_{\sequenceVariable , k_2 j}^{\depthSymbol \headIndex_2 , \outputSymbol}
            \uW_{\sequenceVariable, \gIndexZA k_1}^{\depthSymbol \headIndex_1, \valueSymbol}
            \uW_{\sequenceVariable, \gIndexZB k_2}^{\depthSymbol \headIndex_2, \valueSymbol}
            \tildeLogit{\sequenceVariable, \fIndexA c_1}{\depthSymbol\headIndex_1}
            \tildeLogit{\sequenceVariable, \fIndexB c_2}{\depthSymbol\headIndex_2}
        \coloneqq
        \delta_{\substack{c_1 = \gIndexA \\ c_2 = \gIndexB}}
        \valueVar \,
        \generalMean_{\sequenceVariable}
        \, .
    \end{align*}
    Further, 
    $
        \E [\generalMean_{\sequenceVariable}]
        =
        \E [
            \ntkhat_{
                \gIndexA 1 , \gIndexB 1
            }^{\depthSymbol}
            \tildeLogit{\sequenceVariable, \fIndexA c_1}{\depthSymbol 1}
            \tildeLogit{\sequenceVariable, \fIndexB c_2}{\depthSymbol 1}
        ]
    $
    by exchangeability.
    As in the proof of \Cref{lem:wqk_ntk}, the desired result could thus be obtained by an application of Chebyshev's inequality, $\Prob(|\generalMean_{\sequenceVariable} - \E \generalMean_{\sequenceVariable}| \geq \delta) \leq \delta^{-2} [ \E [\generalMean_{\sequenceVariable}^2] - \{ \E [ \generalMean_{\sequenceVariable} ] \}^2 ] $, if $\E [ \generalMean_{\sequenceVariable} ]$ converges to the desired limit and $|\! \E [\generalMean_{\sequenceVariable}^2] - \{ \E [ \generalMean_{\sequenceVariable} ] \}^2| \to 0$ as $\sequenceVariable \to \infty$.
    
    To establish convergence of the mean, first note that
    $
        \ntkhat_{
            \gIndexA 1 \gIndexB 1
        }^{\depthSymbol}
        \tildeLogit{\sequenceVariable, \fIndexA c_1}{\depthSymbol 1}
        \tildeLogit{\sequenceVariable, \fIndexB c_2}{\depthSymbol 1}
        \convergeDist
        \ntktilde_{
            \gIndexA \gIndexB
        }^{\depthSymbol}
        \tildeLogit{\fIndexA c_1}{\depthSymbol 1}
        \tildeLogit{\fIndexB c_2}{\depthSymbol 1}
    $
    by \Cref{thm:gp_convergence_sqrt}, the continuous mapping theorem,
    $
        \ntkhat_{
            \gIndexA 1 \gIndexB 1
        }^{\depthSymbol}
        \convergeProb
        \ntktilde_{
            \gIndexA \gIndexB
        }^{\depthSymbol}
    $ \citep{yang2019v2}, and \Cref{lem:slutsky}.
    Inspecting \Cref{thm:mean_convergence} and \Cref{lem:sup_ui}, we see it is sufficient to show $\E [ \generalMean_{\sequenceVariable}^2 ] \to \{ \E [ \generalMean_{*} ] \}^2$ to establish both convergence of the mean, and $\generalMean_{\sequenceVariable} \convergeProb \E [ \generalMean_{*} ]$.
    We thus turn to $\E [\generalMean_{\sequenceVariable}^2]$
    \begin{align*}
        \frac{1}{(
            \headDimension \valueDimension \layerDimension{\depthSymbol - 1}
        )^2}
        \sum_{\substack{\gIndexZA_1, \gIndexZB_1\\ \gIndexZA_2, \gIndexZB_2}}
        \sum_{\substack{\headIndex_1, \headIndex_2, \headIndex_3, \headIndex_4\\ k_1, k_2, k_3, k_4}}
            \! \! \! \!
            \E \left[
                \prod_{t=0}^1
                    \ntkhat_{
                        \gIndexA \gIndexZA_{t+1} ,
                        \gIndexB \gIndexZB_{t+1}
                    }^{\depthSymbol}
                    \uW_{\sequenceVariable , k_{2t + 1} i}^{\depthSymbol \headIndex_{2t + 1} , \outputSymbol}
                    \uW_{\sequenceVariable , k_{2t + 2} j}^{\depthSymbol \headIndex_{2t + 2} , \outputSymbol}
                    \uW_{\sequenceVariable, \gIndexZA_{t + 1} k_{2t + 1}}^{\depthSymbol \headIndex_{2t + 1}, \valueSymbol}
                    \uW_{\sequenceVariable, \gIndexZB_{t + 1} k_{2t + 2}}^{\depthSymbol \headIndex_{2t + 2}, \valueSymbol}
                    \tildeLogit{\sequenceVariable, \fIndexA c_1}{\depthSymbol\headIndex_{2t + 1}}
                    \tildeLogit{\sequenceVariable, \fIndexB c_2}{\depthSymbol\headIndex_{2t + 2}}
            \right]
        .
    \end{align*}
    We can thus restrict our attention to groups of terms that include at least $\mathcal{O}((\headDimension \valueDimension \layerDimension{\depthSymbol - 1})^2)$ of the summands, as long as the expectation of the square of each term can be bounded by a constant independent of the $\headIndex, k, \gIndexZA, \gIndexZB$ and $\sequenceVariable$ indices.
    Observe that
    \begin{align}\label{eq:gg_bound}
        \E \left\{
            \left[
            \prod_{t=0}^1
                    \ntkhat_{
                        \gIndexA \gIndexZA_{t+1} ,
                        \gIndexB \gIndexZB_{t+1}
                    }^{\depthSymbol}
                    \uW_{\sequenceVariable , k_{2t + 1} i}^{\depthSymbol \headIndex_{2t + 1} , \outputSymbol}
                    \right.\right.
                    &\left.\left.
                    \vphantom{\prod_{t=0}^1}
                    \uW_{\sequenceVariable , k_{2t + 2} j}^{\depthSymbol \headIndex_{2t + 2} , \outputSymbol}
                    \uW_{\sequenceVariable, \gIndexZA_{t + 1} k_1}^{\depthSymbol \headIndex_{2t + 1}, \valueSymbol}
                    \uW_{\sequenceVariable, \gIndexZB_{t + 1} k_2}^{\depthSymbol \headIndex_{2t + 2}, \valueSymbol}
                    \tildeLogit{\sequenceVariable, \fIndexA c_1}{\depthSymbol\headIndex_{2t + 1}}
                    \tildeLogit{\sequenceVariable, \fIndexB c_2}{\depthSymbol\headIndex_{2t + 2}}
            \right]^2
        \right\}
        \nonumber
        \\
        &\lesssim
        \poly \left(
            \max_{\substack{
                \gIndexA, \gIndexB \in [\spatialDimension],
                \gIndexZA, \gIndexZB \in \{ 1, 2 \} \\
                z, z' \in \{ \inputSymbol, \inputSymbol' \}
            }}
                \E [
                    \ntkhat_{
                        \gIndexA \gIndexZA , \,
                        \gIndexB \gIndexZB
                    }^{\depthSymbol} (z, z')^4
                ]
            \, ,
            \max_{\substack{
                c, c' \in [\spatialDimension]  \\
                z \in \{ \inputSymbol , \inputSymbol' \}
            }}
                \E [
                    \tildeLogit{\sequenceVariable, c , c'}{\depthSymbol 1} (z)^8
                ]
        \right)
        \, ,
    \end{align}
    and thus we can obtain the desired bound by applying \Cref{lem:mmnt_propagation}.
    We thus shift our attention to the terms that are not $\littleO( (\headDimension \valueDimension \layerDimension{\depthSymbol - 1})^2 )$ of which there are three types: (i)~$i = j$, $(\headIndex_1, k_1, \gIndexZA_1) = (\headIndex_2, k_2, \gIndexZB_2)$, and $(\headIndex_3, k_3, \gIndexZA_3) = (\headIndex_4, k_4, \gIndexZB_4)$; (ii)~$i = j$, $(\headIndex_1, k_1, \gIndexZA_1) = (\headIndex_4, k_4, \gIndexZB_4)$, and $(\headIndex_2, k_2, \gIndexZA_2) = (\headIndex_3, k_3, \gIndexZB_3)$; (iii)~$(\headIndex_1, k_1, \gIndexZA_1) = (\headIndex_3, k_3, \gIndexZA_3)$, and $(\headIndex_2, k_2, \gIndexZA_2) = (\headIndex_4, k_4, \gIndexZA_4)$.
    Hence the limit of $\E [\generalMean_{\sequenceVariable}^2 ]$ will up to a constant coincide with that of 
    \begin{align*}
        \E \left[
            \left(
                \ntkhat_{
                    \gIndexA 1 ,
                    \gIndexB 2
                }^{\depthSymbol}
                \tildeLogit{\sequenceVariable, \fIndexA \gIndexA}{\depthSymbol 1}
                \tildeLogit{\sequenceVariable, \fIndexB \gIndexB}{\depthSymbol 2}
            \right)^2
        \right]
        +
        \delta_{i = j}
        \E \left[
            \left(
                \ntkhat_{
                    \gIndexA 1 ,
                    \gIndexB 1
                }^{\depthSymbol}
                \ntkhat_{
                    \gIndexA 2 ,
                    \gIndexB 2
                }^{\depthSymbol}
                +
                \ntkhat_{
                    \gIndexA 1 ,
                    \gIndexB 2
                }^{\depthSymbol}
                \ntkhat_{
                    \gIndexA 2 ,
                    \gIndexB 1
                }^{\depthSymbol}
            \right)
            \tildeLogit{\sequenceVariable, \fIndexA \gIndexA}{\depthSymbol 1}
            \tildeLogit{\sequenceVariable, \fIndexB \gIndexB}{\depthSymbol 1}
            \tildeLogit{\sequenceVariable, \fIndexA \gIndexA}{\depthSymbol 2}
            \tildeLogit{\sequenceVariable, \fIndexB \gIndexB}{\depthSymbol 2}
        \right]
        \, ,
    \end{align*}
    by exchangeability.
    Noticing $\tildeLogit{\sequenceVariable}{\depthSymbol}$ converges in distribution by \Cref{thm:gp_convergence_sqrt} and the continuous mapping theorem, and the $\ntkhat^{\depthSymbol}$ converges in this distribution \citep{yang2019v2}, both integrands converge in distribution by the continuous mapping theorem and \Cref{lem:slutsky}.
    Since the $\tildeLogit{\sequenceVariable}{\depthSymbol \headIndex}$ corresponding to different heads are independent in the limit (\Cref{thm:gp_convergence_sqrt}), and the limit of  $\ntkhat_{\gIndexA \gIndexZA , \gIndexB \gIndexZB}^{\depthSymbol}$ is non-zero only if $\gIndexZA = \gIndexZB$
    \citep{yang2019v2}, application of \Cref{thm:mean_convergence} combined with the bound from \Cref{eq:gg_bound} concludes the proof.
\end{proof}

\begin{lemma}\label{lem:vv_ntk}
    $
        \frac{\outVar}{\headDimension \valueDimension}
        \sum_{\gIndexZA, \gIndexZB}
        \sum_{\substack{\headIndex_1, \headIndex_2\\ k_1, k_2}}
            \ntkhat_{
                \gIndexA \gIndexZA , \gIndexB \gIndexZB
            }^{\depthSymbol}
            \uW_{\sequenceVariable , k_1 i}^{\depthSymbol \headIndex_1 , \outputSymbol}
            \uW_{\sequenceVariable , k_2 j}^{\depthSymbol \headIndex_2 , \outputSymbol}
            \val{\sequenceVariable, c_1 k_1}{\depthSymbol \headIndex_1}
            \val{\sequenceVariable, c_2 k_2}{\depthSymbol \headIndex_2}
            \frac{
                \partial
                \tildeLogit{\sequenceVariable, \fIndexA c_1}{\depthSymbol\headIndex_1}
            }{
                \partial 
                \logitSymbol_{\sequenceVariable, \fIndexA d_1}^{\depthSymbol\headIndex_1}
            }
            \frac{
                \partial
                \tildeLogit{\sequenceVariable, \fIndexB c_2}{\depthSymbol\headIndex_2}
            }{
                \partial 
                \logitSymbol_{\sequenceVariable, \fIndexB d_2}^{\depthSymbol\headIndex_2}
            }
            \frac{
                \partial 
                \logitSymbol_{\sequenceVariable, \fIndexA d_1}^{\depthSymbol\headIndex_1}
            }{
                \partial \activitySymbol_{\sequenceVariable, \gIndexA \gIndexZA}^{\depthSymbol - 1}
            }
            \frac{
                \partial 
                \logitSymbol_{\sequenceVariable, \fIndexB d_2}^{\depthSymbol\headIndex_2}
            }{
                \partial \activitySymbol_{\sequenceVariable, \gIndexB \gIndexZB}^{\depthSymbol - 1}
            }
    $
    converges in probability to
    \begin{align*}
        \delta_{\substack{i = j \\ \scaling = \frac{1}{2}}}
        \OVVar
        \QKVar
        \ntktilde_{\gIndexA \gIndexB}^{\depthSymbol}(\inputSymbol, \inputSymbol')
        \kerntildef{c_1 c_2}{\depthSymbol}{\inputSymbol}{\inputSymbol'}
        \left( 
            \delta_{\substack{d_1 = \gIndexA \\ d_2 = \gIndexB}}
            \kerntildef{\fIndexA \fIndexB}{\depthSymbol}{\inputSymbol}{\inputSymbol'}
            +
            \delta_{\substack{\gIndexA = \fIndexA \\ \gIndexB = \fIndexB}}
            \kerntildef{d_1 d_2}{\depthSymbol}{\inputSymbol}{\inputSymbol'}
        \right)
        \E \left[
            \frac{
                \partial
                \tildeLogit{\fIndexA c_1}{\depthSymbol 1} (\inputSymbol)
            }{
                \partial 
                \logitSymbol_{\fIndexA d_1}^{\depthSymbol 1} (\inputSymbol)
            }
            \frac{
                \partial
                \tildeLogit{\fIndexB c_2}{\depthSymbol 1} (\inputSymbol')
            }{
                \partial 
                \logitSymbol_{\fIndexB d_2}^{\depthSymbol 1} (\inputSymbol')
            }
        \right]
        \, .
    \end{align*}
\end{lemma}

\begin{proof}[Proof of \Cref{lem:vv_ntk}]
    To make the notation more succinct, we define
    \begin{align}\label{eq:vv_grad_logit_summand}
        \frac{
            \partial 
            \logitSymbol_{\sequenceVariable, \fIndexA d_1}^{\depthSymbol\headIndex_1}
        }{
            \partial \activitySymbol_{\sequenceVariable, \gIndexA \gIndexZA}^{\depthSymbol - 1}
        }
        =
        \frac{1}{(\logitDimension)^{\scaling} \sqrt{\layerDimension{\depthSymbol - 1}}}
                \sum_{u_1 = 1}^{\logitDimension}
                    \underbrace{
                        \delta_{d_1 = \gIndexA}
                        \keyStd
                        \query{\fIndexA u_1}{\headIndex_1}
                            \uW_{\sequenceVariable, \gIndexZA u_1}^{\depthSymbol\headIndex_1 , \keySymbol}
                        +
                        \delta_{\gIndexA = \fIndexA}
                        \queryStd
                        \key{\sequenceVariable, d_1 u_1}{\depthSymbol \headIndex_1}
                            \uW_{\sequenceVariable, \gIndexZA u_1}^{\depthSymbol\headIndex_1 , \querySymbol}
                    }_{\coloneqq \Gamma_{\gIndexZA u_1}^{\headIndex_1}}
        \, .
    \end{align}
    which leads us to
    \begin{align*}
        \generalMean_{\sequenceVariable}
        =
        \frac{\outVar}{\headDimension \valueDimension \layerDimension{\depthSymbol - 1} (\logitDimension)^{2\scaling}}
        \sum_{\substack{\gIndexZA, \gIndexZB \\ \headIndex_1, \headIndex_2}}
        \sum_{\substack{k_1, k_2 \\ u_1 , u_2}}
            \ntkhat_{
                \gIndexA \gIndexZA , \gIndexB \gIndexZB
            }^{\depthSymbol}
            \uW_{\sequenceVariable , k_1 i}^{\depthSymbol \headIndex_1 , \outputSymbol}
            \uW_{\sequenceVariable , k_2 j}^{\depthSymbol \headIndex_2 , \outputSymbol}
            \val{\sequenceVariable, c_1 k_1}{\depthSymbol \headIndex_1}
            \val{\sequenceVariable, c_2 k_2}{\depthSymbol \headIndex_2}
            \frac{
                \partial
                \tildeLogit{\sequenceVariable, \fIndexA c_1}{\depthSymbol\headIndex_1}
            }{
                \partial 
                \logitSymbol_{\sequenceVariable, \fIndexA d_1}^{\depthSymbol\headIndex_1}
            }
            \frac{
                \partial
                \tildeLogit{\sequenceVariable, \fIndexB c_2}{\depthSymbol\headIndex_2}
            }{
                \partial 
                \logitSymbol_{\sequenceVariable, \fIndexB d_2}^{\depthSymbol\headIndex_2}
            }
            \Gamma_{\sequenceVariable, \gIndexZA u_1}^{\headIndex_1}
            \Gamma_{\sequenceVariable, \gIndexZB u_2}^{\headIndex_2}
        \, .
    \end{align*}
    Unlike in the proof of \Cref{lem:gg_ntk}, the mean 
    \begin{align}\label{eq:vv_mean}
        \E [
            \generalMean_{\sequenceVariable}
        ]
        =
        \delta_{i = j}
        \frac{\OVVar}{\layerDimension{\depthSymbol - 1} (\logitDimension)^{2\scaling}}
        \sum_{\gIndexZA, \gIndexZB}
        \sum_{\substack{u_1 , u_2}}
            \E \left[
                \ntkhat_{
                    \gIndexA \gIndexZA , \gIndexB \gIndexZB
                }^{\depthSymbol}
                \frac{
                    \langle
                        \activitySymbol_{\sequenceVariable, c_1 \cdot}^{\depthSymbol - 1}
                        ,
                        \activitySymbol_{\sequenceVariable, c_2 \cdot}^{\depthSymbol - 1}
                    \rangle
                }{
                    \layerDimension{\depthSymbol - 1}
                }
                \frac{
                    \partial
                    \tildeLogit{\sequenceVariable, \fIndexA c_1}{\depthSymbol 1}
                }{
                    \partial 
                    \logitSymbol_{\sequenceVariable, \fIndexA d_1}^{\depthSymbol 1}
                }
                \frac{
                    \partial
                    \tildeLogit{\sequenceVariable, \fIndexB c_2}{\depthSymbol 1}
                }{
                    \partial 
                    \logitSymbol_{\sequenceVariable, \fIndexB d_2}^{\depthSymbol 1}
                }
                \Gamma_{\sequenceVariable, \gIndexZA u_1}^{1}
                \Gamma_{\sequenceVariable, \gIndexZB u_2}^{1}
            \right]
        \, ,
    \end{align}
    only eliminates some of the sums.
    This issue can be resolved with the help of \Cref{lem:vv_subtask_convg}.
    
    \begin{lemma}\label{lem:vv_subtask_convg}
        The random variable 
        \begin{align*}
            \generalMean_{\sequenceVariable}^{\headIndex_1 \headIndex_2}
            \coloneqq
            \frac{1}{\layerDimension{\depthSymbol - 1} (\logitDimension)^{2\scaling}}
            \sum_{\gIndexZA, \gIndexZB}
            \sum_{\substack{u_1 , u_2}}
                    \ntkhat_{
                        \gIndexA \gIndexZA , \gIndexB \gIndexZB
                    }^{\depthSymbol}
                    \Gamma_{\sequenceVariable, \gIndexZA u_1}^{\headIndex_1}
                    \Gamma_{\sequenceVariable, \gIndexZB u_2}^{\headIndex_2}
            \, ,
        \end{align*}
        converges in probability to
        $$
            \delta_{\scaling = \frac{1}{2}}
            \delta_{\headIndex_1 = \headIndex_2}
            \QKVar
            \ntktilde_{\gIndexA \gIndexB}^{\depthSymbol}(\inputSymbol, \inputSymbol')
            \left( 
                \delta_{\substack{d_1 = \gIndexA \\ d_2 = \gIndexB}}
                \kerntildef{\fIndexA \fIndexB}{\depthSymbol}{\inputSymbol}{\inputSymbol'}
                +
                \delta_{\substack{\gIndexA = \fIndexA \\ \gIndexB = \fIndexB}}
                \kerntildef{d_1 d_2}{\depthSymbol}{\inputSymbol}{\inputSymbol'}
            \right)
            \, .
        $$
    \end{lemma}
    
    \begin{proof}
        Notice
        \begin{align}\label{eq:vv_subtask_mean}
            \E [\generalMean_\sequenceVariable^{\headIndex_1 \headIndex_2}]
            =
            &\delta_{\headIndex_1 = \headIndex_2}
            \frac{\QKVar}{(\logitDimension)^{2\scaling - 1}}
            \E 
                \left[
                    \ntkhat_{\gIndexA 1, \gIndexB 1}^{\depthSymbol}
                    \left(
                        \delta_{\substack{d_1 = \gIndexA \\ d_2 = \gIndexB}}
                        \frac{
                            \langle
                                \activitySymbol_{\sequenceVariable, \fIndexA \cdot}^{\depthSymbol}
                                ,
                                \activitySymbol_{\sequenceVariable, \fIndexB \cdot}^{\depthSymbol}
                            \rangle
                        }{
                            \layerDimension{\depthSymbol - 1}
                        }
                        +
                        \delta_{\substack{\gIndexA = \fIndexA \\ \gIndexB = \fIndexB}}
                        \frac{
                            \langle
                                \activitySymbol_{\sequenceVariable, d_1 \cdot}^{\depthSymbol}
                                ,
                                \activitySymbol_{\sequenceVariable, d_2 \cdot}^{\depthSymbol}
                            \rangle
                        }{
                            \layerDimension{\depthSymbol - 1}
                        }
                    \right)
                \right]
                +
                \nonumber
                \\
                &\delta_{\headIndex_1 = \headIndex_2}
                \frac{\QKVar}{(\logitDimension)^{2\scaling - 1}}
                \E \left[
                    \ntkhat_{\gIndexA 1, \gIndexB 1}^{\depthSymbol}
                    \left(
                        \delta_{\substack{d_1 = \gIndexA \\ \gIndexB = \fIndexB}}
                            \frac{
                                \activitySymbol_{\sequenceVariable, \fIndexA 1}^{\depthSymbol - 1}
                                \activitySymbol_{\sequenceVariable, d_2 1}^{\depthSymbol - 1}
                            }{
                                \layerDimension{\depthSymbol - 1}
                            }
                        +
                        \delta_{\substack{\gIndexA = \fIndexA \\ d_2 = \gIndexB}}
                            \frac{
                                \activitySymbol_{\sequenceVariable, d_1 1}^{\depthSymbol - 1}
                                \activitySymbol_{\sequenceVariable, \fIndexB 1}^{\depthSymbol - 1}
                            }{
                                \layerDimension{\depthSymbol - 1}
                            }
                    \right)
                \right]
                \, ,
        \end{align}
        and thus we can combine the fact that $\ntkhat_{\gIndexA 
        \gIndexZA, \gIndexB \gIndexZB}^{\depthSymbol} \convergeProb \delta_{\gIndexZA = \gIndexZB} \ntktilde_{\gIndexA \gIndexB}^{\depthSymbol}$ \citep{yang2019v2} with \Cref{lem:inner_prod_converge,lem:mmnt_propagation}, the continuous mapping theorem, and \Cref{thm:mean_convergence} to obtain 
        $$
            \E [ \generalMean_{\sequenceVariable}^{\headIndex_1 \headIndex_2} ] 
            \to
            \delta_{\scaling = \frac{1}{2}}
            \delta_{\headIndex_2 = \headIndex_2}
            \QKVar
            \ntktilde_{\gIndexA \gIndexB}^{\depthSymbol}(\inputSymbol, \inputSymbol')
            \left( 
                \delta_{\substack{d_1 = \gIndexA \\ d_2 = \gIndexB}}
                \kerntildef{\fIndexA \fIndexB}{\depthSymbol}{\inputSymbol}{\inputSymbol'}
                +
                \delta_{\substack{\gIndexA = \fIndexA \\ \gIndexB = \fIndexB}}
                \kerntildef{d_1 d_2}{\depthSymbol}{\inputSymbol}{\inputSymbol'}
            \right)
            \, ,
        $$
        as $\sequenceVariable \to \infty$.
        To obtain the convergence of $\generalMean_{\sequenceVariable}^{\headIndex_1 \headIndex_2}$ to in probability, it is thus sufficient to show $| \! \E [ (\generalMean_{\sequenceVariable}^{\headIndex_1 \headIndex_2})^2 ] - \{ \E [ \generalMean_{\sequenceVariable}^{\headIndex_1 \headIndex_2} ] \}^2|$ converges to zero as $\sequenceVariable \to \infty$.
        Substituting
        \begin{align*}
            \E [ (\generalMean_{\sequenceVariable}^{\headIndex_1 \headIndex_2})^2 ]
            =
            \frac{1}{(\layerDimension{\depthSymbol - 1} (\logitDimension)^{2\scaling})^2}
            \sum_{\substack{\gIndexZA_1, \gIndexZB_1 \\ \gIndexZA_2, \gIndexZB_2}}
            \sum_{\substack{u_1 , u_2 \\ u_3 , u_4}}
                \E \left[
                    \ntkhat_{
                        \gIndexA \gIndexZA_1 , \gIndexB \gIndexZB_1
                    }^{\depthSymbol}
                    \ntkhat_{
                        \gIndexA \gIndexZA_2 , \gIndexB \gIndexZB_2
                    }^{\depthSymbol}
                    \Gamma_{\sequenceVariable, \gIndexZA_1 u_1}^{1}
                    \Gamma_{\sequenceVariable, \gIndexZB_1 u_2}^{1}
                    \Gamma_{\sequenceVariable, \gIndexZA_2 u_3}^{1}
                    \Gamma_{\sequenceVariable, \gIndexZB_2 u_4}^{1}
                \right]
            \, ,
        \end{align*}
        we can once again restrict our attention to groups of terms that include at least $\mathcal{O}((\layerDimension{\depthSymbol - 1} (\logitDimension)^{2\scaling})^2)$ of the summands as long as each term can be bounded by a constant independent of the $\gIndexZA, \gIndexZB$ and $\sequenceVariable$ indices.
        This bound can be again obtained by a repeated application of H{\" o}lder's inequality, followed by \Cref{lem:mmnt_propagation}.
        We can thus shift our attention to the terms for which either of the following holds: (i)~$(\gIndexZA_1, u_1) = (\gIndexZB_1, u_2)$ and $(\gIndexZA_2, u_3) = (\gIndexZB_2, u_4)$; or (ii)~$(\gIndexZA_1, u_1) = (\gIndexZA_2, u_3)$ and $(\gIndexZB_1, u_2) = (\gIndexZB_2, u_4)$; or $(\gIndexZA_1, u_1) = (\gIndexZB_2, u_4)$ and $(\gIndexZB_1, u_2) = (\gIndexZA_2, u_3)$.
        
        As in \Cref{eq:vv_subtask_mean}, we can use the above established boundedness to see that the contribution from any terms that involve the cross terms like $\query{\fIndexA 1}{\depthSymbol 1} \key{\sequenceVariable, d_1 1}{\depthSymbol 1} \uW_{\sequenceVariable, 1 1}^{\depthSymbol 1 , \keySymbol}
        \uW_{\sequenceVariable 1 1}^{\depthSymbol 1, \querySymbol}$, and terms with either of $\gIndexZA_1 \neq \gIndexZB_2$ and $\gIndexZA_2 \neq \gIndexZB_2$ (the limit of $\ntkhat_{\gIndexA \gIndexZA, \gIndexB, \gIndexZB}$ is zero if $\gIndexZA \neq \gIndexZB$), vanish.
        With some algebraic manipulation analogous to that in \Cref{eq:vv_subtask_mean}, we thus obtain 
        $$
            \E [ (\generalMean_{\sequenceVariable}^{\headIndex_1 \headIndex_2})^2 ] 
            \to
            \delta_{\scaling = \frac{1}{2}}
            \delta_{\headIndex_1 = \headIndex_2}
            \left[
                \QKVar
                \ntktilde_{\gIndexA \gIndexB}^{\depthSymbol}(\inputSymbol, \inputSymbol')
                \left( 
                    \delta_{\substack{d_1 = \gIndexA \\ d_2 = \gIndexB}}
                    \kerntildef{\fIndexA \fIndexB}{\depthSymbol}{\inputSymbol}{\inputSymbol'}
                    +
                    \delta_{\substack{\gIndexA = \fIndexA \\ \gIndexB = \fIndexB}}
                    \kerntildef{d_1 d_2}{\depthSymbol}{\inputSymbol}{\inputSymbol'}
                \right)
            \right]^2
            \, ,
        $$
        as desired.
        Application of Cheybshev's inequality concludes the proof.
    \end{proof}
    
    With $\generalMean_{\sequenceVariable}^{\headIndex_1 \headIndex_2}$ defined as in \Cref{lem:vv_subtask_convg}, we can revisit \Cref{eq:vv_mean} 
    \begin{align*}
        \E [\generalMean_{\sequenceVariable}]
        =
        \delta_{i = j}
        \OVVar
        \E \left[
            \generalMean_{\sequenceVariable}^{1 1}
            \frac{
                \langle
                    \activitySymbol_{\sequenceVariable, c_1 \cdot}^{\depthSymbol - 1}
                    ,
                    \activitySymbol_{\sequenceVariable, c_2 \cdot}^{\depthSymbol - 1}
                \rangle
            }{
                \layerDimension{\depthSymbol - 1}
            }
            \frac{
                \partial
                \tildeLogit{\sequenceVariable, \fIndexA c_1}{\depthSymbol 1}
            }{
                \partial 
                \logitSymbol_{\sequenceVariable, \fIndexA d_1}^{\depthSymbol 1}
            }
            \frac{
                \partial
                \tildeLogit{\sequenceVariable, \fIndexB c_2}{\depthSymbol 1}
            }{
                \partial 
                \logitSymbol_{\sequenceVariable, \fIndexB d_2}^{\depthSymbol 1}
            }
        \right]
        \, .
    \end{align*}
    Note that the first two terms converge in probability to constant by \Cref{lem:inner_prod_converge,lem:vv_subtask_convg} and the continuous mapping theorem, 
    $$
        \frac{
            \partial
            \tildeLogit{\sequenceVariable, \fIndexA c_1}{\depthSymbol 1}
        }{
            \partial 
            \logitSymbol_{\sequenceVariable, \fIndexA d_1}^{\depthSymbol 1}
        }
        \frac{
            \partial
            \tildeLogit{\sequenceVariable, \fIndexB c_2}{\depthSymbol 1}
        }{
            \partial 
            \logitSymbol_{\sequenceVariable, \fIndexB d_2}^{\depthSymbol 1}
        }
        \convergeDist
        \frac{
            \partial
            \tildeLogit{\fIndexA c_1}{\depthSymbol 1}
        }{
            \partial 
            \logitSymbol_{\fIndexA d_1}^{\depthSymbol 1}
        }
        \frac{
            \partial
            \tildeLogit{\fIndexB c_2}{\depthSymbol 1}
        }{
            \partial 
            \logitSymbol_{\fIndexB d_2}^{\depthSymbol 1}
        }
        \, ,
    $$
    if $\scaling = \frac{1}{2}$ by \Cref{thm:gp_convergence_sqrt}, and in probability to a constant if $\scaling = 1$ \citep[appendix A]{yang2019v2}, both using the assumed continuity of $\nabla \softmax$.
    Since $\nabla \softmax$ is also assumed to be bounded, we can combine H{\" o}lder's inequality with \Cref{lem:mmnt_propagation} to establish uniform integrability (see the proof of \Cref{lem:vv_subtask_convg} for the bound on $\generalMean_{\sequenceVariable}^{ 1 1 }$) via \Cref{lem:sup_ui}, and with that convergence of $\E [ \generalMean_\sequenceVariable ]$ by \Cref{lem:slutsky} and \Cref{thm:mean_convergence}, yielding
    \begin{align*}
        \E [\generalMean_{\sequenceVariable}]
        \to
        \delta_{\substack{i = j \\ \scaling = \frac{1}{2}}}
        \OVVar
        \QKVar
        \ntktilde_{\gIndexA \gIndexB}^{\depthSymbol}(\inputSymbol, \inputSymbol')
        \kerntildef{c_1 c_2}{\depthSymbol}{\inputSymbol}{\inputSymbol'}
        \left( 
            \delta_{\substack{d_1 = \gIndexA \\ d_2 = \gIndexB}}
            \kerntildef{\fIndexA \fIndexB}{\depthSymbol}{\inputSymbol}{\inputSymbol'}
            +
            \delta_{\substack{\gIndexA = \fIndexA \\ \gIndexB = \fIndexB}}
            \kerntildef{d_1 d_2}{\depthSymbol}{\inputSymbol}{\inputSymbol'}
        \right)
        \E \left[
            \frac{
                \partial
                \tildeLogit{\fIndexA c_1}{\depthSymbol 1} (\inputSymbol)
            }{
                \partial 
                \logitSymbol_{\fIndexA d_1}^{\depthSymbol 1} (\inputSymbol)
            }
            \frac{
                \partial
                \tildeLogit{\fIndexB c_2}{\depthSymbol 1} (\inputSymbol')
            }{
                \partial 
                \logitSymbol_{\fIndexB d_2}^{\depthSymbol 1} (\inputSymbol')
            }
        \right]
        \, .
    \end{align*}
    
    Convergence of $\generalMean_{\sequenceVariable}$ to the same constant can be obtained via Chebyshev's inequality by proving $| \! \E [ \generalMean_{\sequenceVariable}^2] - \{ \E [ \generalMean_{\sequenceVariable}] \}^2  | \to 0$.
    Using the notation from \Cref{lem:vv_subtask_convg},
    the second moment of $\generalMean_{\sequenceVariable}$ can be written as
    \begin{align*}
        \E [\generalMean_{\sequenceVariable}^2]
        =
        \frac{\outStd^{4}}{(\headDimension \valueDimension)^2}
        \sum_{\substack{\headIndex_1, \headIndex_2, \headIndex_3, \headIndex_4 \\ k_1, k_2, k_3, k_4}}
            \E \left[
                \prod_{t=0}^1
                    \generalMean_{\sequenceVariable}^{\headIndex_{2t + 1} \headIndex_{2t + 2}}
                    \uW_{\sequenceVariable , k_{2t + 1} i}^{\depthSymbol \headIndex_{2t + 1} , \outputSymbol}
                    \uW_{\sequenceVariable , k_{2t + 2} j}^{\depthSymbol \headIndex_{2t + 2} , \outputSymbol}
                    \val{\sequenceVariable, c_1 k_{2t + 1}}{\depthSymbol \headIndex_{2t + 1}}
                    \val{\sequenceVariable, c_2 k_{2t + 2}}{\depthSymbol \headIndex_{2t + 2}}
                    \frac{
                        \partial
                        \tildeLogit{\sequenceVariable, \fIndexA c_1}{\depthSymbol\headIndex_{2t + 1}}
                    }{
                        \partial 
                        \logitSymbol_{\sequenceVariable, \fIndexA d_1}^{\depthSymbol\headIndex_{2t + 1}}
                    }
                    \frac{
                        \partial
                        \tildeLogit{\sequenceVariable, \fIndexB c_2}{\depthSymbol\headIndex_{2t + 2}}
                    }{
                        \partial 
                        \logitSymbol_{\sequenceVariable, \fIndexB d_2}^{\depthSymbol\headIndex_{2t + 2}}
                    }
            \right]
        \, .
    \end{align*}
    Because $\nabla \softmax$ is bounded by assumption, we can again use H{\" o}lder's inequality together with \Cref{lem:mmnt_propagation} to bound each of the summands by a constant independent of the $\headIndex, k$ and $\sequenceVariable$ indices.
    This means we can restrict our attention only to groups of terms that include at least $\mathcal{O}((\headDimension \valueDimension)^2)$ of the summands. 
    These fall into one of the following three categories: (i)~$i = j$, $(\headIndex_1, k_1) = (\headIndex_2, k_2)$, and $(\headIndex_3, k_3) = (\headIndex_4, k_4)$; (i)~$(\headIndex_1, k_1) = (\headIndex_3, k_3)$, and $(\headIndex_2, k_4) = (\headIndex_4, k_4)$; and (iii)~$i = j$, $(\headIndex_1, k_1) = (\headIndex_4, k_4)$, and $(\headIndex_2, k_2) = (\headIndex_3, k_3)$.
    Using exchangeability, we thus obtain
    \begin{align*}
        \E [\generalMean_{\sequenceVariable}^2]
        =
        &\OVStd^4
        \E \left[
            \generalMean_{\sequenceVariable}^{1 2}
            \generalMean_{\sequenceVariable}^{1 2}
            \frac{
                \langle
                    \activitySymbol_{\sequenceVariable, c_1 \cdot}^{\depthSymbol - 1}
                    ,
                    \activitySymbol_{\sequenceVariable, c_1 \cdot}^{\depthSymbol - 1}
                \rangle
            }{
                \layerDimension{\depthSymbol - 1}
            }
            \frac{
                \langle
                    \activitySymbol_{\sequenceVariable, c_2 \cdot}^{\depthSymbol - 1}
                    ,
                    \activitySymbol_{\sequenceVariable, c_2 \cdot}^{\depthSymbol - 1}
                \rangle
            }{
                \layerDimension{\depthSymbol - 1}
            }
            \frac{
                \partial
                \tildeLogit{\sequenceVariable, \fIndexA c_1}{\depthSymbol 1}
            }{
                \partial 
                \logitSymbol_{\sequenceVariable, \fIndexA d_1}^{\depthSymbol 1}
            }
            \frac{
                \partial
                \tildeLogit{\sequenceVariable, \fIndexB c_2}{\depthSymbol 1}
            }{
                \partial 
                \logitSymbol_{\sequenceVariable, \fIndexB d_2}^{\depthSymbol 1}
            }
                    \frac{
                        \partial
                        \tildeLogit{\sequenceVariable, \fIndexA c_1}{\depthSymbol 2}
                    }{
                        \partial 
                        \logitSymbol_{\sequenceVariable, \fIndexA d_1}^{\depthSymbol 2}
                    }
                    \frac{
                        \partial
                        \tildeLogit{\sequenceVariable, \fIndexB c_2}{\depthSymbol 2}
                    }{
                        \partial 
                        \logitSymbol_{\sequenceVariable, \fIndexB d_2}^{\depthSymbol 2}
                    }
        \right]
        +
        \\
        &\delta_{i = j}
        \OVStd^4
        \E \left[
            (
                \generalMean_{\sequenceVariable}^{1 1}
                \generalMean_{\sequenceVariable}^{2 2}
                +
                \generalMean_{\sequenceVariable}^{1 2}
                \generalMean_{\sequenceVariable}^{2 1}
            )
            \left(
                \frac{
                    \langle
                        \activitySymbol_{\sequenceVariable, c_1 \cdot}^{\depthSymbol - 1}
                        ,
                        \activitySymbol_{\sequenceVariable, c_2 \cdot}^{\depthSymbol - 1}
                    \rangle
                }{
                    \layerDimension{\depthSymbol - 1}
                }
            \right)^2
            \frac{
                \partial
                \tildeLogit{\sequenceVariable, \fIndexA c_1}{\depthSymbol 1}
            }{
                \partial 
                \logitSymbol_{\sequenceVariable, \fIndexA d_1}^{\depthSymbol 1}
            }
            \frac{
                \partial
                \tildeLogit{\sequenceVariable, \fIndexB c_2}{\depthSymbol 1}
            }{
                \partial 
                \logitSymbol_{\sequenceVariable, \fIndexB d_2}^{\depthSymbol 1}
            }
                    \frac{
                        \partial
                        \tildeLogit{\sequenceVariable, \fIndexA c_1}{\depthSymbol 2}
                    }{
                        \partial 
                        \logitSymbol_{\sequenceVariable, \fIndexA d_1}^{\depthSymbol 2}
                    }
                    \frac{
                        \partial
                        \tildeLogit{\sequenceVariable, \fIndexB c_2}{\depthSymbol 2}
                    }{
                        \partial 
                        \logitSymbol_{\sequenceVariable, \fIndexB d_2}^{\depthSymbol 2}
                    }
        \right]
        +
        \\
        &\littleO((\headDimension \valueDimension)^2)
        \, .
    \end{align*}
    Note by the assumed continuity of $\nabla \softmax$, \Cref{thm:gp_convergence_sqrt}, \Cref{lem:inner_prod_converge,lem:vv_subtask_convg}, the continuous mapping theorem, and \Cref{lem:slutsky}, both the integrands converge in distribution, which, combined with the above derived bound and \Cref{thm:mean_convergence}, implies
    \begin{align*}
        \E [ \generalMean_{\sequenceVariable}^2 ]
        \to
        \delta_{\substack{i = j \\ \scaling = \frac{1}{2}}}
        \left[
            \OVVar
            \QKVar
            \ntktilde_{\gIndexA \gIndexB}^{\depthSymbol}(\inputSymbol, \inputSymbol')
            \kerntildef{c_1 c_2}{\depthSymbol}{\inputSymbol}{\inputSymbol'}
            \left( 
                \delta_{\substack{d_1 = \gIndexA \\ d_2 = \gIndexB}}
                \kerntildef{\fIndexA \fIndexB}{\depthSymbol}{\inputSymbol}{\inputSymbol'}
                +
                \delta_{\substack{\gIndexA = \fIndexA \\ \gIndexB = \fIndexB}}
                \kerntildef{d_1 d_2}{\depthSymbol}{\inputSymbol}{\inputSymbol'}
            \right)
            \E \left[
                \frac{
                    \partial
                    \tildeLogit{\fIndexA c_1}{\depthSymbol 1} (\inputSymbol)
                }{
                    \partial 
                    \logitSymbol_{\fIndexA d_1}^{\depthSymbol 1} (\inputSymbol)
                }
                \frac{
                    \partial
                    \tildeLogit{\fIndexB c_2}{\depthSymbol 1} (\inputSymbol')
                }{
                    \partial 
                    \logitSymbol_{\fIndexB d_2}^{\depthSymbol 1} (\inputSymbol')
                }
            \right]
        \right]^2
    \end{align*}
    where we have used the fact that $\generalMean_{\sequenceVariable}^{\headIndex_1 \headIndex_2}$ converges in probability to zero whenever $\headIndex_1 \neq \headIndex_2$ (\Cref{lem:vv_subtask_convg}), and the asymptotic indepedence of $\tildeLogit{\sequenceVariable}{\depthSymbol 1}$ and $\tildeLogit{\sequenceVariable}{\depthSymbol 2}$ (\Cref{thm:gp_convergence_sqrt} if $\scaling = \frac{1}{2}$, resp.\ \citep[appendix A]{yang2019v2} if $\scaling = 1$).
\end{proof}

\begin{lemma}\label{lem:gv_ntk}
    $
    \frac{\outVar}{\headDimension \valueDimension}
    \sum_{\gIndexZA, \gIndexZB}
    \sum_{\substack{\headIndex_1, \headIndex_2\\ k_1, k_2}}
        \ntkhat_{
            \gIndexA \gIndexZA , \gIndexB \gIndexZB
        }^{\depthSymbol}
        \uW_{\sequenceVariable , k_1 i}^{\depthSymbol \headIndex_1 , \outputSymbol}
        \uW_{\sequenceVariable , k_2 j}^{\depthSymbol \headIndex_2 , \outputSymbol}
        \tildeLogit{\sequenceVariable, \fIndexA c_1}{\depthSymbol\headIndex_1}
        \val{\sequenceVariable, c_2 k_2}{\depthSymbol \headIndex_2}
        \frac{
            \partial
            \val{\sequenceVariable, c_1 k_1}{\depthSymbol \headIndex_1}
        }{
            \partial \activitySymbol_{\sequenceVariable, \gIndexA \gIndexZA}^{\depthSymbol - 1}
        }
        \frac{
            \partial
            \tildeLogit{\sequenceVariable, \fIndexB c_2}{\depthSymbol\headIndex_2}
        }{
            \partial 
            \logitSymbol_{\sequenceVariable, \fIndexB d_2}^{\depthSymbol\headIndex_2}
        }
        \frac{
            \partial 
            \logitSymbol_{\sequenceVariable, \fIndexB d_2}^{\depthSymbol\headIndex_2}
        }{
            \partial \activitySymbol_{\sequenceVariable, \gIndexB \gIndexZB}^{\depthSymbol - 1}
        }
    \convergeProb
    0
    \, .
    $
\end{lemma}

\begin{proof}[Proof of \Cref{lem:gv_ntk}]
    Observing that 
    $
        \frac{
            \partial
            \val{\sequenceVariable, c_1 k_1}{\depthSymbol \headIndex_1}
        }{
            \partial \activitySymbol_{\sequenceVariable, \gIndexA \gIndexZA}^{\depthSymbol - 1}
        }
        =
        \delta_{c_1 = \gIndexA}
        \valueStd
        \frac{
            \uW_{\sequenceVariable, \gIndexZA k_1}^{\depthSymbol \headIndex_1, \valueSymbol}
        }{
            \sqrt{\layerDimension{\depthSymbol - 1}}
        }
    $
    and setting
    \begin{align*}
        \generalMean_{\sequenceVariable}
        =
        \frac{\valueStd}{\headDimension \valueDimension \sqrt{\layerDimension{\depthSymbol - 1}}}
        \sum_{\gIndexZA, \gIndexZB}
        \sum_{\substack{\headIndex_1, \headIndex_2\\ k_1, k_2}}
            \ntkhat_{
                \gIndexA \gIndexZA , \gIndexB \gIndexZB
            }^{\depthSymbol}
            \uW_{\sequenceVariable , k_1 i}^{\depthSymbol \headIndex_1 , \outputSymbol}
            \uW_{\sequenceVariable , k_2 j}^{\depthSymbol \headIndex_2 , \outputSymbol}
            \uW_{\sequenceVariable, \gIndexZA k_1}^{\depthSymbol \headIndex_1, \valueSymbol}
            \val{\sequenceVariable, c_2 k_2}{\depthSymbol \headIndex_2}
            \tildeLogit{\sequenceVariable, \fIndexA c_1}{\depthSymbol\headIndex_1}
            \frac{
                \partial
                \tildeLogit{\sequenceVariable, \fIndexB c_2}{\depthSymbol\headIndex_2}
            }{
                \partial 
                \logitSymbol_{\sequenceVariable, \fIndexB d_2}^{\depthSymbol\headIndex_2}
            }
            \frac{
                \partial 
                \logitSymbol_{\sequenceVariable, \fIndexB d_2}^{\depthSymbol\headIndex_2}
            }{
                \partial \activitySymbol_{\sequenceVariable, \gIndexB \gIndexZB}^{\depthSymbol - 1}
            }
        \, ,
    \end{align*}
    we immediately see that 
    \begin{align*}
        \E [\generalMean_\sequenceVariable]
        =
        \delta_{i = j}
        \frac{\valueVar}{\layerDimension{\depthSymbol - 1}}
        \sum_{\gIndexZA, \gIndexZB}
            \E \left[
                \ntkhat_{
                    \gIndexA \gIndexZA , \gIndexB \gIndexZB
                }^{\depthSymbol}
                \activitySymbol_{\sequenceVariable, c_2 \gIndexZA}^{\depthSymbol - 1}
                \tildeLogit{\sequenceVariable, \fIndexA c_1}{\depthSymbol 1}
                \frac{
                    \partial
                    \tildeLogit{\sequenceVariable, \fIndexB c_2}{\depthSymbol 1}
                }{
                    \partial 
                    \logitSymbol_{\sequenceVariable, \fIndexB d_2}^{\depthSymbol 1}
                }
                \frac{
                    \partial 
                    \logitSymbol_{\sequenceVariable, \fIndexB d_2}^{\depthSymbol 1}
                }{
                    \partial \activitySymbol_{\sequenceVariable, \gIndexB \gIndexZB}^{\depthSymbol - 1}
                }
            \right]
        \, .
    \end{align*}
    Analogously to the proof of \Cref{lem:vv_ntk}, we define
    \begin{align}\label{eq:gv_subtask_variable}
        \generalMean_{\sequenceVariable}^{\headIndex}
        &=
        \frac{1}{\layerDimension{\depthSymbol - 1}}
        \sum_{\gIndexZA, \gIndexZB}
            \ntkhat_{
                \gIndexA \gIndexZA , \gIndexB \gIndexZB
            }^{\depthSymbol}
            \activitySymbol_{\sequenceVariable, c_2 \gIndexZA}^{\depthSymbol - 1}
            \frac{
                \partial 
                \logitSymbol_{\sequenceVariable, \fIndexB d_2}^{\depthSymbol \headIndex}
            }{
                \partial \activitySymbol_{\sequenceVariable, \gIndexB \gIndexZB}^{\depthSymbol - 1}
            }
        \\
        &=
        \frac{1}{(\layerDimension{\depthSymbol - 1})^{3/2} (\logitDimension)^{\scaling}}
        \sum_{\gIndexZA, \gIndexZB}
        \sum_{u_1 = 1}^{\logitDimension}
            \ntkhat_{
                \gIndexA \gIndexZA , \gIndexB \gIndexZB
            }^{\depthSymbol}
            \activitySymbol_{\sequenceVariable, c_2 \gIndexZA}^{\depthSymbol - 1}
                \biggl(
                    \underbrace{
                        \delta_{d_2 = \gIndexB}
                        \keyStd
                        \query{\sequenceVariable, \fIndexB u_1}{\depthSymbol\headIndex}
                        \uW_{\sequenceVariable, \gIndexZB u_1}^{\depthSymbol\headIndex , \keySymbol}
                        +
                        \delta_{\gIndexB = \fIndexB}
                        \queryStd
                        \key{\sequenceVariable, d_2 u_1}{\depthSymbol \headIndex}
                            \uW_{\sequenceVariable, \gIndexZB u_1}^{\depthSymbol\headIndex , \querySymbol}
                    }_{\eqqcolon \Gamma_{\sequenceVariable, \gIndexZB u_1}^{\headIndex}}
                \biggr)
        \, ,
    \end{align}
    and make use of an auxiliary lemma.
    
    \begin{lemma}\label{lem:gv_subtask_convg}
        $\generalMean_{\sequenceVariable}^{\headIndex} \convergeProb 0$.
    \end{lemma}
    
    \begin{proof}
        Observe $\E [ \generalMean_{\sequenceVariable}^{\headIndex} ] = 0$ if $\scaling = \frac{1}{2}$ (independence of key and query weights), and
        \begin{align*}
            \E [ \generalMean_{\sequenceVariable}^{\headIndex} ]
            =
            \frac{\QKStd}{(\layerDimension{\depthSymbol - 1})^2}
            \sum_{\gIndexZA, \gIndexZB}
                \E \left[
                    \ntkhat_{
                        \gIndexA \gIndexZA , \gIndexB \gIndexZB
                    }^{\depthSymbol}
                    \activitySymbol_{\sequenceVariable, c_2 \gIndexZA}^{\depthSymbol - 1}
                        \biggl(
                            \delta_{d_2 = \gIndexB}
                            \activitySymbol_{\sequenceVariable, \fIndexB \gIndexZB}^{\depthSymbol - 1}
                            +
                            \delta_{\gIndexB = \fIndexB}
                            \activitySymbol_{\sequenceVariable, d_2 \gIndexZB}^{\depthSymbol - 1}
                        \biggr)
                \right]
            \, ,
        \end{align*}
        if $\scaling = 1$ (key and query weights are equal a.s.).
        Since each of the summands can be bounded by a constant indpendent of the $\gIndexZA, \gIndexZB$ and $\sequenceVariable$ indices by \Cref{lem:mmnt_propagation}, we can restrict our focus to the terms for which $\gIndexZA \neq \gIndexZB$, yielding
        \begin{align*}
            \E [ \generalMean_{\sequenceVariable}^{\headIndex} ]
            =
            \QKStd
            \frac{\layerDimension{\depthSymbol - 1}(\layerDimension{\depthSymbol - 1} - 1)}{(\layerDimension{\depthSymbol - 1})^2}
            \E \left[
                \ntkhat_{
                    \gIndexA 1 , \gIndexB 2
                }^{\depthSymbol}
                \activitySymbol_{\sequenceVariable, c_2 1}^{\depthSymbol - 1}
                    \biggl(
                        \delta_{d_2 = \gIndexB}
                        \activitySymbol_{\sequenceVariable, \fIndexB \gIndexZB}^{\depthSymbol - 1}
                        +
                        \delta_{\gIndexB = \fIndexB}
                        \activitySymbol_{\sequenceVariable, d_2 \gIndexZB}^{\depthSymbol - 1}
                    \biggr)
            \right]
            +
            \littleO((\layerDimension{\depthSymbol - 1})^2)
            \, ,
        \end{align*}
        Since $\ntkhat_{\gIndexA 1 , \gIndexB 2}^{\depthSymbol} \convergeProb 0$ \citep{yang2019v2}, and the $\activitySymbol_{\sequenceVariable, c_2 \gIndexZA}^{\depthSymbol - 1} \activitySymbol_{\sequenceVariable, \fIndexB \gIndexZB}^{\depthSymbol - 1}$ products converge in distribution by continuity of the assumed $\nonlinearity$ and the continuous mapping theorem, the integrand converges to zero in distribution by \Cref{lem:slutsky}.
        Using \Cref{lem:mmnt_propagation,lem:sup_ui} and \Cref{thm:mean_convergence}, we again establish $\E [ \generalMean_{\sequenceVariable}^{\headIndex} ] \to 0$.
        
        To obtain convergence in probability, observe
        \begin{align*}
            \E [ (\generalMean_{\sequenceVariable}^\headIndex)^2 ]
            =
            \frac{1}{(\layerDimension{\depthSymbol - 1})^{3} (\logitDimension)^{2\scaling}}
            \sum_{\substack{\gIndexZA_1, \gIndexZB_1 \\ \gIndexZA_2, \gIndexZB_2}}
            \sum_{u_1, u_2}
                \E \left[
                    \ntkhat_{
                        \gIndexA \gIndexZA_1 , \gIndexB \gIndexZB_1
                    }^{\depthSymbol}
                    \ntkhat_{
                        \gIndexA \gIndexZA_2 , \gIndexB \gIndexZB_2
                    }^{\depthSymbol}
                    \activitySymbol_{\sequenceVariable, c_2 \gIndexZA_1}^{\depthSymbol - 1}
                    \activitySymbol_{\sequenceVariable, c_2 \gIndexZA_2}^{\depthSymbol - 1}
                    \Gamma_{\sequenceVariable, \gIndexZB_1 u_1}^{\headIndex}
                    \Gamma_{\sequenceVariable, \gIndexZB_2 u_2}^{\headIndex}
                \right]
            \, ,
        \end{align*}
        and note that we can again bound each of the summands using H{\" o}lder's inequality and \Cref{lem:mmnt_propagation} as in to \Cref{lem:vv_subtask_convg}.
        We can thus restrict our attention to groups of terms that include at least $\mathcal{O}((\layerDimension{\depthSymbol - 1})^3 (\logitDimension)^{2\scaling})$ of the summands.
        If $\scaling = 1$, we can thus focus on $u_1 \neq u_2$, in which case integrating $\Gamma_{\sequenceVariable, \gIndexZB_1 u_1}^{\headIndex} \Gamma_{\sequenceVariable, \gIndexZB_2 u_2}^{\headIndex}$
        over key and query weights will yield an additional $\layerDimension{\depthSymbol - 1}$ factor,
        for example
        $$
            \E [
                \query{\fIndexB u_1}{\depthSymbol\headIndex}
                \uW_{\sequenceVariable, \gIndexZB_1 u_1}^{\depthSymbol\headIndex , \keySymbol}
                \query{\fIndexB u_2}{\depthSymbol\headIndex}
                \uW_{\sequenceVariable, \gIndexZB_2 u_2}^{\depthSymbol\headIndex , \keySymbol}
            ]
            =
            \frac{\queryVar}{\layerDimension{\depthSymbol}}
            \activitySymbol_{\sequenceVariable, \gIndexB \gIndexZB_1}
            \activitySymbol_{\sequenceVariable, \gIndexB \gIndexZB_2}
            \, ,
        $$
        using the equality of key and query weights.
        Since $\ntkhat_{\gIndexA \gIndexZA, \gIndexB \gIndexZB}^{\depthSymbol}$ converges in probability to zero whenever $\gIndexZA \neq \gIndexZB$ \citep{yang2019v2}, and there are only $(\layerDimension{\depthSymbol - 1})^2$ terms for which $\gIndexZA_1 = \gIndexZB_1$ and $\gIndexZA_2 = \gIndexZB_2$, we can use the continuous mapping theorem, \Cref{lem:slutsky}, and \Cref{thm:mean_convergence} to establish that $\E [ (\generalMean_{\sequenceVariable}^{\headIndex})^2 ] \to 0$.
        If $\scaling = \frac{1}{2}$, all terms for which $u_1 \neq 0$ will have zero expectation (independence of key and query weights), and thus analogous argument to the one for $\scaling = 1$.
    \end{proof}
    
    With \Cref{lem:gv_subtask_convg} at hand, we can simplify
    \begin{align*}
        \E [\generalMean_\sequenceVariable]
        =
        \delta_{i = j}
        \valueVar
        \E \left[
            \generalMean_{\sequenceVariable}^1
            \tildeLogit{\sequenceVariable, \fIndexA c_1}{\depthSymbol 1}
            \frac{
                \partial
                \tildeLogit{\sequenceVariable, \fIndexB c_2}{\depthSymbol 1}
            }{
                \partial 
                \logitSymbol_{\sequenceVariable, \fIndexB d_2}^{\depthSymbol 1}
            }
        \right]
        \, ,
    \end{align*}
    and use the assumed continuity of $\nabla \softmax$ together with the continuous mapping theorem and \Cref{lem:slutsky} to establish that the integrand converges in distribution to zero.
    Since $\nabla \softmax$ is bounded by assumption, we can used H{\" o}lder's inequality and \Cref{lem:mmnt_propagation} to establish uniform integrability via \Cref{lem:sup_ui} (see the proof of \Cref{lem:gv_subtask_convg} for the bound on $\generalMean_\sequenceVariable^1$).
    We thus have $\E [\generalMean_\sequenceVariable] \to 0$
    by \Cref{thm:mean_convergence}.
    
    To establish $\generalMean_{\sequenceVariable} \convergeProb 0$, it is sufficient to show $\E [ (\generalMean_{\sequenceVariable})^2 ] \to 0$ and apply Chebyshev's inequality.\footnote{We will be using the explicit parenthesis here to distinguish between $\generalMean_{\sequenceVariable}^\headIndex$ with $\headIndex = 2$, and $(\generalMean_\sequenceVariable)^2$.}
    We have
    \begin{align*}
        \E [(\generalMean_{\sequenceVariable})^2]
        =
        \frac{\valueVar}{(\headDimension \valueDimension\layerDimension{\depthSymbol - 1})^2}
        \sum_{\substack{\gIndexZA_1, \gIndexZB_1 \\ \gIndexZA_2, \gIndexZB_2}}
        \sum_{\substack{\headIndex_1, \headIndex_2, \headIndex_3, \headIndex_4 \\ k_1, k_2, k_3, k_4}}
            \begin{aligned}[t]
            \E \left[
                \prod_{t=0}^1
                \right.&\left.
                \ntkhat_{
                    \gIndexA \gIndexZA_{t+1} , \gIndexB \gIndexZB_{t+1}
                }^{\depthSymbol}
                \uW_{\sequenceVariable , k_{2t + 1} i}^{\depthSymbol \headIndex_{2t + 1} , \outputSymbol}
                \uW_{\sequenceVariable , k_{2t + 2} j}^{\depthSymbol \headIndex_{2t + 2} , \outputSymbol}
                \right.
                \\
                &\left.
                \uW_{\sequenceVariable, \gIndexZA_{t+1} k_{2t + 1}}^{\depthSymbol \headIndex_{2t + 1}, \valueSymbol}
                \val{\sequenceVariable, c_2 k_{2t + 2}}{\depthSymbol \headIndex_{2t + 2}}
                \tildeLogit{\sequenceVariable, \fIndexA c_1}{\depthSymbol\headIndex_{2t + 1}}
                \frac{
                    \partial
                    \tildeLogit{\sequenceVariable, \fIndexB c_2}{\depthSymbol\headIndex_{2t + 2}}
                }{
                    \partial 
                    \logitSymbol_{\sequenceVariable, \fIndexB d_2}^{\depthSymbol\headIndex_{2t + 2}}
                }
                \sqrt{\layerDimension{\depthSymbol - 1}}
                \frac{
                    \partial 
                    \logitSymbol_{\sequenceVariable, \fIndexB d_2}^{\depthSymbol\headIndex_{2t + 2}}
                }{
                    \partial \activitySymbol_{\sequenceVariable, \gIndexB \gIndexZB_{t+1}}^{\depthSymbol - 1}
                }
        \right]
        \, ,
        \end{aligned}
    \end{align*}
    where notice we are multiplying
    $
        \frac{
            \partial 
            \logitSymbol_{\sequenceVariable, \fIndexB d_2}^{\depthSymbol\headIndex_{2t + 2}}
        }{
            \partial \activitySymbol_{\sequenceVariable, \gIndexB \gIndexZB_{t+1}}^{\depthSymbol - 1}
        }
    $
    by $\sqrt{\layerDimension{\depthSymbol - 1}}$ as this term scales as $(\layerDimension{\depthSymbol - 1})^{-1/2}$ (see \Cref{eq:gv_subtask_variable}).
    Since $\nabla \softmax$ is bounded by assumption, we can use the H{\" o}lder's inequality to bound each of the summands by
    \begin{align*}
        \poly \biggl(
            \max_{\substack{
                \gIndexA, \gIndexB \in [\spatialDimension],
                \gIndexZA, \gIndexZB \in \{ 1, 2 \} \\
                z, z' \in \{ \inputSymbol, \inputSymbol' \}
            }}
                \E [
                    \ntkhat_{
                        \gIndexA \gIndexZA , \,
                        \gIndexB \gIndexZB
                    }^{\depthSymbol} (z, z')^4
                ]
            \, ,
            \max_{\substack{
                c, c' \in [\spatialDimension]  \\
                z \in \{ \inputSymbol , \inputSymbol' \}
            }}
                \E [
                    \tildeLogit{\sequenceVariable, c c'}{\depthSymbol 1} (z)^8
                ]
            \, ,
            \max_{\substack{c \in [\spatialDimension] \\ z \in \{ \inputSymbol, \inputSymbol' \}}}
                \E | \indexedActivity{\sequenceVariable, \depthSymbol - 1}{c 1}{z} |^{16}
        \biggr)
        \, ,
    \end{align*}
    which will be bounded by a constant independent of the $\gIndexZA, \gIndexZB, \headIndex, k$ and $\sequenceVariable$ by \Cref{lem:mmnt_propagation}.
    By \Cref{lem:sup_ui}, we can thus restrict our attention to the terms that are not $\littleO((\headDimension \valueDimension \layerDimension{\depthSymbol - 1})^2)$, which fall into one of the following three categories: 
    (i)~$i = j$, $(\headIndex_1, k_1) = (\headIndex_2 , k_2)$, and $(\headIndex_3, k_3) = (\headIndex_4, k_4)$; 
    (ii)~$(\headIndex_1, k_1, \gIndexZA_1) = (\headIndex_3 , k_3, \gIndexZA_2)$, and $(\headIndex_2, k_2, \gIndexZB_1) = (\headIndex_4, k_4, \gIndexZB_2)$; 
    (iii)~$i = j$, $(\headIndex_1, k_1) = (\headIndex_3, k_3)$, and $(\headIndex_2 , k_2) = (\headIndex_4, k_4)$.
    Using exchangeability, we thus obtain
    \begin{align}\label{eq:gv_square}
        \E [(\generalMean_{\sequenceVariable})^2]
        =
        &\valueStd^4
        \E \left[
            \frac{
                \langle
                    \activitySymbol_{\sequenceVariable, c_2 \cdot}^{\depthSymbol - 1}
                    ,
                    \activitySymbol_{\sequenceVariable, c_2 \cdot}^{\depthSymbol - 1}
                \rangle
            }{
                \layerDimension{\depthSymbol - 1}
            }
            \left(
                \ntkhat_{
                    \gIndexA 1, \gIndexB 2
                }^{\depthSymbol}
                \tildeLogit{\sequenceVariable, \fIndexA c_1}{\depthSymbol 1}
                \frac{
                    \partial
                    \tildeLogit{\sequenceVariable, \fIndexB c_2}{\depthSymbol 1}
                }{
                    \partial 
                    \logitSymbol_{\sequenceVariable, \fIndexB d_2}^{\depthSymbol 1}
                }
                \sqrt{\layerDimension{\depthSymbol - 1}}
                \frac{
                    \partial 
                    \logitSymbol_{\sequenceVariable, \fIndexB d_2}^{\depthSymbol 2}
                }{
                    \partial \activitySymbol_{\sequenceVariable, \gIndexB 2}^{\depthSymbol - 1}
                }
            \right)^2
        \right]
        +
        \\
        &
        \delta_{i = j}
        2
        \valueVar
        \E \left[
            \generalMean_{\sequenceVariable}^{1}
            \generalMean_{\sequenceVariable}^{2}
            \tildeLogit{\sequenceVariable, \fIndexA c_1}{\depthSymbol 1}
            \tildeLogit{\sequenceVariable, \fIndexA c_1}{\depthSymbol 2}
            \frac{
                \partial
                \tildeLogit{\sequenceVariable, \fIndexB c_2}{\depthSymbol 1}
            }{
                \partial 
                \logitSymbol_{\sequenceVariable, \fIndexB d_2}^{\depthSymbol 1}
            }
                    \frac{
                        \partial
                        \tildeLogit{\sequenceVariable, \fIndexB c_2}{\depthSymbol 2}
                    }{
                        \partial 
                        \logitSymbol_{\sequenceVariable, \fIndexB d_2}^{\depthSymbol 2}
                    }
        \right]
        +
        \littleO((\headDimension \valueDimension \layerDimension{\depthSymbol - 1})^2)
        \, ,
    \end{align}
    where we have used
    $
        \E [
            \uW_{\sequenceVariable, \gIndexZA_{1} 1}^{\depthSymbol 1, \valueSymbol}
            \val{\sequenceVariable, c_2 1}{\depthSymbol 1} 
            \uW_{\sequenceVariable, \gIndexZA_2 2}^{\depthSymbol 2, \valueSymbol}
            \val{\sequenceVariable, c_2 2}{\depthSymbol 2}
        ]
        =
        \frac{\valueVar}{\layerDimension{\depthSymbol - 1}}
        \activitySymbol_{\sequenceVariable, c_2 \gIndexZA_1}^{\depthSymbol - 1}
        \activitySymbol_{\sequenceVariable, c_2 \gIndexZA_2}^{\depthSymbol - 1}
    $
    and the definition of $\generalMean_{\sequenceVariable}^\headIndex$ from \Cref{eq:gv_subtask_variable}.
    We prove convergence of both of these expectations to zero separately.
    
    Starting with the second expectation in \Cref{eq:gv_square}, we can use the assumed continuity of $\nabla \softmax$, \Cref{thm:gp_convergence_sqrt}, \Cref{eq:gv_subtask_variable}, the continuous mapping theorem, and \Cref{lem:slutsky} to establish that the integrand converges in distribution to zero.
    Because $\nabla \softmax$ is bounded by assumption, we can combine H{\" o}lder's inequality and \Cref{lem:mmnt_propagation} to establish uniform integrability via \Cref{lem:sup_ui} (see the proof of \Cref{lem:gv_subtask_convg} for the bound on $\generalMean_{\sequenceVariable}^\headIndex$), and thus convergence of the expectation to zero by \Cref{thm:mean_convergence}.
    
    For the first expectation in \Cref{eq:gv_square}, note that the absolute value of the expectation can be upper bounded by
    \begin{align*}
        \E \left[
            \left|
                \frac{
                    \langle
                        \activitySymbol_{\sequenceVariable, c_2 \cdot}^{\depthSymbol - 1}
                        ,
                        \activitySymbol_{\sequenceVariable, c_2 \cdot}^{\depthSymbol - 1}
                    \rangle
                }{
                    \layerDimension{\depthSymbol - 1}
                }
            \right|
            (
                \ntkhat_{
                    \gIndexA 1, \gIndexB 2
                }^{\depthSymbol}
            )^2
            \left(
                \left(
                    \tildeLogit{\sequenceVariable, \fIndexA c_1}{\depthSymbol 1}
                    \frac{
                        \partial
                        \tildeLogit{\sequenceVariable, \fIndexB c_2}{\depthSymbol 1}
                    }{
                        \partial 
                        \logitSymbol_{\sequenceVariable, \fIndexB d_2}^{\depthSymbol 1}
                    }
                \right)^2
                +
                \layerDimension{\depthSymbol - 1}
                \left(
                    \frac{
                        \partial 
                        \logitSymbol_{\sequenceVariable, \fIndexB d_2}^{\depthSymbol 2}
                    }{
                        \partial \activitySymbol_{\sequenceVariable, \gIndexB 2}^{\depthSymbol - 1}
                    }
                \right)^2
            \right)
        \right]
        \, ,
    \end{align*}
    where, when multiplied out, we can use that $\ntkhat_{\gIndexA 1, \gIndexB 2}^{\depthSymbol} \convergeProb 0$ \citep{yang2019v2}, and an argument analogous to the one above---using \Cref{lem:inner_prod_converge} and the continuous mapping theorem to obtain convergence in probability for the inner product---to establish convergence to zero.
    Finally, for the second term, observe
    \begin{align*}
        \sqrt{\layerDimension{\depthSymbol - 1}}
        \frac{
            \partial 
            \logitSymbol_{\sequenceVariable, \fIndexB d_2}^{\depthSymbol 2}
        }{
            \partial \activitySymbol_{\sequenceVariable, \gIndexB 2}^{\depthSymbol - 1}
        }
        =
        \frac{1}{(\logitDimension)^{\scaling}}
        \sum_{u_1 = 1}^{\logitDimension}
            \delta_{d_2 = \gIndexB}
            \keyStd
            \query{\sequenceVariable, \fIndexB u_1}{\depthSymbol\headIndex}
            \uW_{\sequenceVariable, \gIndexZB u_1}^{\depthSymbol\headIndex , \keySymbol}
            +
            \delta_{\gIndexB = \fIndexB}
            \queryStd
            \key{\sequenceVariable, d_2 u_1}{\depthSymbol \headIndex}
                \uW_{\sequenceVariable, \gIndexZB u_1}^{\depthSymbol\headIndex , \querySymbol}
        \, ,
    \end{align*}
    which converges in probability to a constant if $\scaling = 1$ by \Cref{lem:wlln_exch} (using \citep{yang2019v2} to establish convergence in distribution of the keys and queries, and \Cref{lem:mmnt_propagation} for the moment bound).
    If $\scaling = \frac{1}{2}$, the sum will converge in distribution \citep{yang2019v2} and since the rest of the term in the expectation converge in probability, their product converges in distribution by \Cref{lem:slutsky}.
    One can then again combine H{\" o}lder's inequality, \Cref{lem:mmnt_propagation,lem:sup_ui} and \Cref{thm:mean_convergence} to obtain convergence of the expectation to zero.
    Hence $\E [ (\generalMean_{\sequenceVariable})^2 ] \to 0$, implying $\generalMean_{\sequenceVariable} \convergeProb 0$ as desired.
\end{proof}

\subsection{Expressivity of $\genericDimenstion^{-1}$ and $\genericDimenstion^{-1/2}$ induced attention kernels}

\linearNoConv*

\begin{proof}[Proof of \Cref{prop:linear_scale_no_conv}]
    Consider $\kernelf{aa}{\text{CNN}}{\inputSymbol}{\inputSymbol} = \sum_{i = 1}^{\genericDimenstion_f} \kerntildef{N_a(i) N_a(i)}{}{\inputSymbol}{\inputSymbol} \frac{1}{\genericDimenstion_f}$, and the corresponding attention kernel $\kernelf{aa}{\text{ATTN}}{\inputSymbol}{\inputSymbol} = \sum_{i, j=1}^{\spatialDimension} \kerntildef{i j}{}{\inputSymbol}{\inputSymbol} \softmaxMean_{ai}^{\inputSymbol} \softmaxMean_{aj}^{\inputSymbol}$.
    Note that $\kernelf{aa}{\text{CNN}}{\inputSymbol}{\inputSymbol}$ is just sum of terms on a subset of the diagonal of $\kerntildef{}{}{\inputSymbol}{\inputSymbol}$.
    Hence it must be that $\softmaxMean_{ai}^{\inputSymbol} = \pm (\genericDimenstion_f)^{-1/2}$ since we require that the same set of coefficients $\{ \softmaxMean_{ai}^\inputSymbol \colon i \in [\spatialDimension] \}$ works all kernels $\kerntilde$ simultaneously, and thus for any $\kerntildef{aa}{}{\inputSymbol}{\inputSymbol}$ including all diagonal matrices with non-negative entries. 
    Therefore $\softmaxMean_{ai}^{\inputSymbol} \softmaxMean_{aj}^{\inputSymbol} = \pm (\genericDimenstion_f)^{-1}$ for all $i, j$, making signs the only degree of freedom.\footnote{As a side note, this degree of freedom disappears when $\softmaxMean$ is a limit of the softmax variables (non-negativity).}
    We conclude by noting that we can make $\kernelf{aa}{\text{ATTN}}{\inputSymbol}{\inputSymbol} \neq \kernelf{aa}{\text{CNN}}{\inputSymbol}{\inputSymbol}$ by choosing $\kerntildef{aa}{}{\inputSymbol}{\inputSymbol}$ diagonal except for one pair of off-diagonal entries.
\end{proof}

\sqrtConvRecover*

\begin{proof}[Proof of \Cref{prop:sqrt_scale_conv_recover}]
    We provide a simple construction here, and expand on more realistic ones after the proof.
    
    Consider $\Omega = [0, 1)$ with the usual Borel $\sigma$-algebra $\mathcal{B}$ and the Lebesgue measure $\lambda$.
    Let $\xbar{\mathbb{R}} = \R{} \cup \{-\infty, \infty \}$ be the extended real axis and $\xbar{\mathcal{B}}$ be the $\sigma$-algebra generated by the interval topology on $\xbar{\mathbb{R}}$.
    Now construct the random variables $\logitSymbol_{ai} \colon \Omega \to \xbar{\mathbb{R}}$ such that $\logitSymbol_{ai} = - \infty$ a.s.\ if $i \notin N_a$, and $\logitSymbol_{ai} = \infty \cdot \indicator{A_{ai}}$ a.s.\ where and $A_{a i} \coloneqq \left[\frac{i_{(a)} - 1}{\genericDimenstion_f}, \frac{i_{(a)}}{\genericDimenstion}\right)$, with $i_{(a)}$ being the position of $i$ in the ordered set $N_a$, and $\infty \cdot 0$ is to be interpreted as $0$.
\end{proof}

For a more realistic construction consider the usual $\logitSymbol(\inputSymbol) = \genericDimenstion^{-1/2} \querySymbol(\inputSymbol) \keySymbol(\inputSymbol)^\top$ but now additionally multiply each row of $\querySymbol(\inputSymbol$ by a corresponding scalar random variable $c_a^\querySymbol \colon \Omega \to \xbar{\mathbb{R}}$, similarly each row of $\keySymbol(\inputSymbol)$ by $c_a^\keySymbol \colon \Omega \to \xbar{\mathbb{R}}$.
Then $\logitSymbol_{ab}(\inputSymbol) = \genericDimenstion^{-1/2} c_a^\querySymbol c_b^\keySymbol \langle \querySymbol_{a\cdot}(\inputSymbol) , \keySymbol_{b\cdot}(\inputSymbol) \rangle$ and thus one can achieve the desired result by setting up the joint distribution of $\{ c_1^{\querySymbol} , \ldots , c_{\spatialDimension}^{\querySymbol}, c_{1}^{\keySymbol}, \ldots , c_{\spatialDimension}^{\keySymbol} \}$ in analogy to that in the above proof.

\subsection{Auxiliary results}

\begin{lemma}[Billingsley, 1986, p. 19]\label{lem:fin_dim_marg}
    Let $X, (X_n)_{n \geq 1}$ be random variables taking values in $(\R{\natnum}, \mathcal{B}^N)$, $\mathcal{B}^\natnum$ the~usual Borel $\sigma$-algebra.
    Then $X_n \convergeDist X$ if and only if for each finite $J \subset \natnum$ and the~corresponding projection $\Gamma^J \colon \R{\natnum} \to \R{J}$, $\Gamma^J (X_n) \convergeDist \Gamma^J (X)$ as $n \to \infty$.
\end{lemma}

\begin{lemma}[Billingsley, 1986, p. 31]\label{lem:sup_ui}
    A sequence of real valued random variables $(\genericRV_\rowIndex)_{\rowIndex \geq 1}$ is uniformly integrable if 
    $$\sup_{\rowIndex} \E | \genericRV_\rowIndex |^{1 + \varepsilon}  < \infty \, .$$
\end{lemma}

\begin{theorem}[Billingsley, 1986, theorem 3.5]\label{thm:mean_convergence}
    If $(\genericRV_\rowIndex)_{\rowIndex \geq 1}$ are uniformly integrable and $\genericRV_{\rowIndex} \convergeDist \genericRV$, then $\genericRV$ is integrable and 
    $$\E [\genericRV_{\rowIndex}] \to \E [\genericRV] \, .$$
\end{theorem}

\begin{lemma}[Slutsky's lemmas]\label{lem:slutsky}
    Let $X, (X_n)_{n \geq 1}$ and $(Y_n)_{n \geq 1}$ be real valued random variables defined on the same probability space, and assume $X_n \convergeDist X$ and $Y_n \convergeProb c$ for some $c \in \R{}$.
    Then
    \begin{align}
        &X_n Y_n \convergeDist c X
        \, ,
        &X_n + Y_n \convergeDist X + c
        \, .
    \end{align}
\end{lemma}

\begin{lemma}[Weak LLN for exchangeable triangular arrays]\label{lem:wlln_exch}
    Let 
    $\genericRV_\rowIndex \coloneqq \{ \genericRV_{\rowIndex, \colIndex} \colon \colIndex = 1, 2, \ldots \}$
    be an infinitely exchangeable sequence of random variables on $\R{\natnum}$ 
    s.t.\ $\limsup_{\rowIndex \to \infty} \E  |\genericRV_{\rowIndex, 1}|^{2 + \varepsilon} < \infty$ for some $\varepsilon > 0$, and
    define
    $
        \generalMean_{\rowIndex} 
        \coloneqq
        \frac{1}{\genericDimenstion_\rowIndex}
        \sum_{\colIndex=1}^{\genericDimenstion_\rowIndex} 
            \genericRV_{\rowIndex, \colIndex} 
        \, , 
    $
    for some sequence $(\genericDimenstion_\rowIndex)_{\rowIndex \geq 1}$ s.t.\ $\lim_{\rowIndex \to \infty} \genericDimenstion_\rowIndex = \infty$.
    Assuming all $\genericRV_{\rowIndex}$ are defined on the same space, if $\genericRV_{\rowIndex}$ converges in distribution to some infinitely exchangeable $\genericRV_{*} = \{ \genericRV_{*, \colIndex} \colon \colIndex = 1, 2, \ldots \}$ s.t.\ $\E [\genericRV_{*, 1} \genericRV_{* , 2}] = (\E [ \genericRV_{*, 1} ] )^2$,
    then as $\rowIndex \to \infty$,
    $\E [\generalMean_{\rowIndex} ] \to \E [\genericRV_{*, 1}]$, 
    $\E [\generalMean_{\rowIndex}^2 ] \to (\E [\genericRV_{*, 1}] )^2$, and
    $$\generalMean_{\rowIndex} \convergeProb \E [\genericRV_{*, 1}] \, .$$
\end{lemma}

\begin{proof}[Proof of \Cref{lem:wlln_exch}]
    By exchangeability $\E [\generalMean_{\rowIndex}] = \E [\genericRV_{\rowIndex, 1}]$, and thus by \Cref{lem:sup_ui} and \Cref{thm:mean_convergence}, $\E [\generalMean_{\rowIndex}] \to \E [ \genericRV_{\star, 1}]$.
    Similarly, 
    \begin{align*}
        \E[ \generalMean_\rowIndex^2 ]
        =
        \frac{1}{\genericDimenstion_{\rowIndex}}
        \E [\genericRV_{\rowIndex, 1}^2]
        +
        \frac{\genericDimenstion_{\rowIndex} (\genericDimenstion_{\rowIndex} - 1)}{\genericDimenstion_{\rowIndex}^2}
        \E [ \genericRV_{\rowIndex, 1} \genericRV_{\rowIndex, 2}]
        \, ,
    \end{align*}
    and thus by the continuous mapping theorem and again by \Cref{lem:sup_ui} and \Cref{thm:mean_convergence}, $\E[ \generalMean_\rowIndex^2 ] \to (\E[ \genericRV_{\star,1} ] )^2$ as $\rowIndex \to \infty$.
    Finally, the convergence in probability follows by Chebyshev's inequality
    \begin{equation*}
        \Prob \left\{
            \abs{\generalMean_{\rowIndex} - \E \generalMean_{\rowIndex} }
            \geq
            \delta
        \right\}
        \leq
        \frac{
            \E[ \generalMean_\rowIndex^2 ]
            -
            (\E [\generalMean_\rowIndex])^2
        }{
            \delta^2
        }
        \, .
        \qedhere
    \end{equation*}
\end{proof}

\begin{lemma}[Moment propagation]\label{lem:mmnt_propagation}
    Under the assumptions of \Cref{thm:gp_convergence_sqrt}, for any $\inputSymbol, \inputSymbol' \in \indexSet$, $\depthSymbol \in [\depth + 1]$, and $t \geq 1$
    \begin{align*}
       \sup_{\substack{c \in [\spatialDimension] \\ i \in \natnum}}
       &\sup_{\sequenceVariable}
            \E | 
                \indexedActivity{\depthSymbol - 1}{\sequenceVariable, c i}{\inputSymbol} 
            |^t
        <
        \infty
        \, ,
        \\
       \sup_{\substack{c \in [\spatialDimension] \\ i \in \natnum}}
       &\sup_{\sequenceVariable}
            \E | 
                \indexedActivation{\depthSymbol}{\sequenceVariable, c i}{\inputSymbol} 
            |^t
        <
        \infty
        \, ,
        \\
        \sup_{\substack{c \in [\spatialDimension] \\ \headIndex, i \in \natnum}}
        &\sup_{\sequenceVariable}
        \E |
            \indexedActivation{\depthSymbol\headIndex}{\sequenceVariable, c i}{\inputSymbol}
        |^{t}
        <
        \infty
        \, ,
        \\
        \sup_{\substack{
            c, c' \in [\spatialDimension] \\
            \headIndex \in \natnum
        }}
        &\sup_{\sequenceVariable}
            \E |
                \tildeLogit{\sequenceVariable, c c'}{\depthSymbol \headIndex} (\inputSymbol)
            |^t
        <
        \infty
        \, ,
        \\
        \sup_{\substack{
            a, b \in [\spatialDimension] \\
            i, j \in \natnum
        }}
        &\sup_{\sequenceVariable}
            \E |
                \ntkhat_{a i , \, b j}^{\depthSymbol} ( \inputSymbol, \inputSymbol' ) 
            |^t
        <
        \infty
        \, .
    \end{align*}
\end{lemma}

\begin{proof}[Proof of \Cref{lem:mmnt_propagation}]
    Beginning with
    $
        \E | 
            \indexedActivity{\depthSymbol - 1}{\sequenceVariable, c i}{\inputSymbol} 
        |^t
    $
    and
    $
        \E | 
            \indexedActivation{\depthSymbol}{\sequenceVariable, c i}{\inputSymbol} 
        |^t
    $,
    we see that this condition holds if none of the layers $1, \ldots, \depthSymbol - 1$ uses attention by the assumed polynomial boundedness of $\nonlinearity$ as a corollary of \citep[lemma 20]{matthews2018gaussian} for dense, and \citep[pages 14 and 15]{garriga2019deep} for convolutional networks.\footnote{Note that the bound on $\E | \indexedActivity{0}{\sequenceVariable, c i}{\inputSymbol} |^t = |x_{c i}|^t$ is trivial, which then leads to a bound on $\E |\indexedActivation{\depthSymbol}{\sequenceVariable, c i}{\inputSymbol} |^t$ as the individual columns will be i.i.d.\ Gaussian for any $\sequenceVariable$.}
    To extend to the case where one or more of the preceding layers include attention, we see that it is sufficient to focus on bound for $\activationSymbol^{\depthSymbol}$ in the first attention layer (i.e., with the lowest $\depthSymbol$ among the attention layer), as the bound for the following $\activitySymbol^{\depthSymbol}$ can be obtained from the assumed polynomial boundedness of $\nonlinearity$ and exchangeability, and the bound on the following layers by a simple inductive argument.
    
    Thus, focusing on 
    $
        \E | 
            \indexedActivation{\depthSymbol}{\sequenceVariable, c i}{\inputSymbol} 
        |^t
        =
        \E | 
            \indexedActivation{\depthSymbol}{\sequenceVariable, c 1}{\inputSymbol} 
        |^t
    $ (exchangeability),
    we see that proving the bound for an arbitrary fixed $c \in [\spatialDimension]$ will be sufficient as $\spatialDimension$ is finite.
    Substituting
    \begin{align*}
        \E | 
            \indexedActivation{\depthSymbol}{\sequenceVariable, c 1}{\inputSymbol} 
        |^t
        =
        \E \left\{
            \E \left[
                \Biggl|
                        \sum_{\headIndex=1}^{\headDimension}
                        \sum_{k=1}^{\valueDimension}
                            \indexedActivation{\depthSymbol\headIndex}{\sequenceVariable, c k}{\inputSymbol}
                            \W_{\sequenceVariable, k 1}^{\depthSymbol \headIndex , \outputSymbol}
                    \Biggr|^t
                    \, \, 
                    \Bigg\vert
                    \, \,
                    \indexedActivation{\depthSymbol 1}{\sequenceVariable, c \cdot }{\inputSymbol}
                    \, , \ldots \, ,
                    \indexedActivation{\depthSymbol\headDimension}{\sequenceVariable, c \cdot }{\inputSymbol}
            \right]
        \right\}
        \lesssim
        \E \, \Biggl|
            \frac{\outVar}{\headDimension \valueDimension}
            \sum_{h , k}
                \indexedActivation{\depthSymbol\headIndex}{\sequenceVariable, c k }{\inputSymbol}^2
        \Biggr|^{\frac{t}{2}}
        \, ,
    \end{align*}
    where we have used that if $\varepsilon \sim \gauss( 0, I)$ and $v \in \R{\genericDimenstion}$ is a fixed vector, then $\langle v , \varepsilon \rangle$ is in distribution equal to $\| v \|_2 \varepsilon'$ where $\varepsilon' \sim \gauss(0, 1)$ by standard Gaussian identities, and the fact that $\E | \varepsilon' |^t < \infty$.
    Using H{\" o}lder's inequality if necessary, we can assume $t$ is even, and with that multiply out the r.h.s.\ above, leading to 
    \begin{align*}
        \E \, \Biggl|
            \frac{1}{\headDimension \valueDimension}
            \sum_{h , k}
                \indexedActivation{\depthSymbol\headIndex}{\sequenceVariable, c k }{\inputSymbol}^2
        \Biggr|^{\frac{t}{2}}
        \lesssim
        \E |
            \indexedActivation{\depthSymbol 1}{\sequenceVariable, c 1 }{\inputSymbol}
        |^{t}
        \, ,
    \end{align*}
    by exchangeability.
    It will thus be sufficient to establish 
    $
        \sup_{c \in [\spatialDimension], \headIndex, i \in \natnum}
        \sup_{\sequenceVariable}
        \E |
            \indexedActivation{\depthSymbol\headIndex}{\sequenceVariable, c i}{\inputSymbol}
        |^{t}
        <
        \infty
    $
    for any $t \geq 1$.
    Observe
    \begin{align*}
        \E |
            \indexedActivation{\depthSymbol\headIndex}{\sequenceVariable, c i}{\inputSymbol}
        |^{t}
        =
        \E \left\{
            \E \left[
                \Biggl|
                    \sum_{j = 1}^{\layerDimension{\depthSymbol - 1}}
                        \tildeLogitN{\sequenceVariable, c \cdot}{\depthSymbol\headIndex}{\inputSymbol}
                        \indexedActivity{\depthSymbol - 1}{\sequenceVariable, \cdot j}{\inputSymbol}
                        \weightV{\cdot i}
                \Biggr|^{t}
                \, \, 
                \Bigg\vert
                \, \,
                \tildeLogit{\sequenceVariable}{\depthSymbol\headIndex} (\inputSymbol) \, ,
                \indexedActivity{\depthSymbol - 1}{\sequenceVariable}{\inputSymbol}
            \right]
        \right\}
        \lesssim
        \E \Biggl|
            \frac{\valueVar}{\layerDimension{\depthSymbol - 1}}
            \sum_{j}
                \left(
                    \tildeLogitN{\sequenceVariable, c \cdot}{\depthSymbol\headIndex}{\inputSymbol}
                    \indexedActivity{\depthSymbol - 1}{\sequenceVariable, \cdot j}{\inputSymbol}
                \right)^2
        \Biggr|^{\frac{t}{2}}
        \, ,
    \end{align*}
    meaning we can combine an argument analogous to the one above with H{\" o}lder's inequality and exchangeability to obtain
    \begin{align*}
        \E \Biggl|
            \frac{1}{\layerDimension{\depthSymbol - 1}}
            \sum_{j}
                \left(
                    \tildeLogitN{\sequenceVariable, c \cdot}{\depthSymbol\headIndex}{\inputSymbol}
                    \indexedActivity{\depthSymbol - 1}{\sequenceVariable, \cdot j}{\inputSymbol}
                \right)^2
        \Biggr|^{\frac{t}{2}}
        \lesssim
        \poly \biggl(
            \max_{c' \in [\spatialDimension]}
            \E |
                \tildeLogitN{\sequenceVariable, c c'}{\depthSymbol\headIndex}{\inputSymbol}
            |^{2t}
            \, ,
            \max_{c' \in [\spatialDimension]}
            \sup_{j \in \natnum}
            \E |
                \indexedActivity{\depthSymbol - 1}{\sequenceVariable, c' j}{\inputSymbol}
            |^{2t}
        \biggr)
        \, .
    \end{align*}
    As shown at the beginning of this proof, we can bound
    $
        \E |
            \indexedActivity{\depthSymbol - 1}{\sequenceVariable, c' j}{\inputSymbol}
        |^{4t}
    $
    by a constant independent of $c', j$ and $\sequenceVariable$, and thus it only remains to show that
    $
        \max_{c, c' \in [\spatialDimension]}
        \sup_{\substack{
            \headIndex \in \natnum
        }}
        \sup_{\sequenceVariable}
            \E |
                \tildeLogit{\sequenceVariable, c c'}{\depthSymbol \headIndex} (\inputSymbol)
            |^t
        <
        \infty
    $
    in order to bound the $\E | \indexedActivation{\depthSymbol}{\sequenceVariable, c 1}{\inputSymbol} |^t$.
    Using the assumed entrywise polynomial boundedness of $\softmax$, we can see it will be sufficient to establish
    $
        \max_{c, c' \in [\spatialDimension]}
        \sup_{\substack{
            \headIndex \in \natnum
        }}
        \sup_{\sequenceVariable}
            \E |
                \logitSymbol_{\sequenceVariable, c c'}^{\depthSymbol \headIndex} (\inputSymbol)
            |^t
        <
        \infty
    $.
    We do this separately for $\scaling = 1$ and $\scaling = \frac{1}{2}$.
    
    Starting with the former, we can again replicate the argument from above, yielding
    \begin{align*}
        \E |
            \logitSymbol_{\sequenceVariable, c c'}^{\depthSymbol \headIndex} (\inputSymbol)
        |^t
        =
        \E  \Biggl|
            \frac{1}{\logitDimension}
            \sum_{k = 1}^{\logitDimension}
                \query{\sequenceVariable , c k}{\depthSymbol\headIndex} (\inputSymbol)
                \key{\sequenceVariable, c' k}{\depthSymbol\headIndex} (\inputSymbol)
        \Biggr|^{t}
        \lesssim
        \E |
            \query{\sequenceVariable , c 1}{\depthSymbol 1} (\inputSymbol)
        |^{2t}
        \lesssim
        \E |
            \indexedActivity{\depthSymbol - 1}{\sequenceVariable, c 1}{\inputSymbol}
        |^{4t}
        \, ,
    \end{align*}
    by exchangeability, H{\" o}lder's inequality, and the assumed $\weightsQ = \weightsK$ a.s.\ under $\scaling = 1$ (\Cref{sect:linear_scaling_limit}).
    For the $\scaling = \frac{1}{2}$ case, we start by w.l.o.g.\ assuming we need bound for $t \in \natnum$ even
    \begin{align*}
        \E |
            \logitSymbol_{\sequenceVariable, c c'}^{\depthSymbol \headIndex} (\inputSymbol)
        |^{t}
        =
        \E  \Biggl|
            \frac{1}{\sqrt{\logitDimension}}
            \sum_{k = 1}^{\logitDimension}
                \query{\sequenceVariable , c k}{\depthSymbol\headIndex} (\inputSymbol)
                \key{\sequenceVariable, c' k}{\depthSymbol\headIndex} (\inputSymbol)
        \Biggr|^{t}
        =
        \left(\frac{1}{\sqrt{\logitDimension}}\right)^t
        \sum_{k_1, \ldots, k_t}
        \E \left[
            \prod_{s=1}^{t}
            \query{\sequenceVariable , c k_s}{\depthSymbol\headIndex} (\inputSymbol)
            \key{\sequenceVariable, c' k_s}{\depthSymbol\headIndex} (\inputSymbol)
        \right]
        \, ,
    \end{align*}
    and noting that because $\weightsQ$ and $\weightsK$ are assumed independent under $\scaling = \frac{1}{2}$, there will be at most $\mathcal{O} ( (\logitSymbol)^{t / 2} )$ terms with non-zero expectation, meaning that we can again apply exchangeability and the distributional equivalence between $\query{\sequenceVariable}{\depthSymbol\headIndex}(\inputSymbol)$ and $\key{\sequenceVariable}{\depthSymbol\headIndex}(\inputSymbol)$ to obtain
    \begin{equation*}
        \E |
            \logitSymbol_{\sequenceVariable, c c'}^{\depthSymbol \headIndex} (\inputSymbol)
        |^{t}
        \lesssim
        \E |
            \query{\sequenceVariable , c 1}{\depthSymbol 1} (\inputSymbol)
        |^{2t}
        \lesssim
        \E |
            \indexedActivity{\depthSymbol - 1}{\sequenceVariable, c 1}{\inputSymbol}
        |^{4t}
        \, .
    \end{equation*}
    
    The convergence of 
    $
        \max_{\substack{
            a, b \in [\spatialDimension]
        }}
        \sup_{i, j \in \natnum}
        \sup_{\sequenceVariable}
            \E |
                \ntkhat_{a i , \, b j}^{\depthSymbol} ( \inputSymbol, \inputSymbol' ) 
            |^t
        <
        \infty
    $
    for non-attention layers can be obtained by combining exchangeability between the two groups of $\ntkhat_{a i , \, b j}^{\depthSymbol}$ variables with $i \neq j$ (resp.\ $i = j$) indices, and the results in \citep{yang2019v1} which show that expectations of polynomially bounded functions converge (this is essentially due to the assumed polynomial boundedness of $\nonlinearity$ and $\softmax$ and their (weak) derivatives, the fact that the pre-nonlinearities in the first layer are Gaussian for all of the considered architectures by linearity of Gaussian variables, and the standard combination of \Cref{lem:sup_ui} and \Cref{thm:mean_convergence}---see the proofs of theorems 4.3 and 5.1 in \citep{yang2019v1}).
    This can then be used to prove the bound for the first attention layer by inspecting the proofs in \Cref{sect:ntk_proofs} and noting that 
    $
        \sup_{\sequenceVariable}
            \E |
                \ntkhat_{a i , \, b j}^{\depthSymbol} ( \inputSymbol, \inputSymbol' ) 
            |^t
    $
    can be always bounded by a polynomial in suprema over quantities from previous layers that we already know are uniformly bounded.
    The proof for subsequent attention layers can thus proceed by induction.
\end{proof}

\begin{lemma}[Convergence of inner products]\label{lem:inner_prod_converge}
    Under the assumptions of \Cref{thm:gp_convergence_sqrt}, the following holds for any $a, b \in [\spatialDimension]$, $\inputSymbol, \inputSymbol' \in \indexSet$, $\depthSymbol \in [\depth + 1]$, and $\headIndex \in [\headDimension]$
    \begin{align}
        &
        \E \left[
            \frac{
                \langle \indexedActivity{\depthSymbol - 1}{\sequenceVariable, a \cdot}{\inputSymbol} , \indexedActivity{\depthSymbol - 1}{\sequenceVariable, b \cdot}{\inputSymbol'} \rangle
            }{
                \layerDimension{\depthSymbol - 1}
            }
        \right]
        \overset{\sequenceVariable \to \infty}{\to}
        \kerntildef{a b}{\depthSymbol}{\inputSymbol}{\inputSymbol'}
        \, ,
        &&
        \frac{
            \langle \indexedActivity{\depthSymbol - 1}{\sequenceVariable, a \cdot}{\inputSymbol} , \indexedActivity{\depthSymbol - 1}{\sequenceVariable, b \cdot}{\inputSymbol'} \rangle
        }{
            \layerDimension{\depthSymbol - 1}
        }
        \convergeProb
        \kerntildef{a b}{\depthSymbol}{\inputSymbol}{\inputSymbol'}
        \, ,
        \\
        &
        \E \left[
            \frac{
                \langle \query{\sequenceVariable, a \cdot}{\depthSymbol \headIndex}(\inputSymbol) , \query{\sequenceVariable, b \cdot}{\depthSymbol \headIndex}(\inputSymbol') \rangle
            }{
                \logitDimension
            }
        \right]
        \overset{\sequenceVariable \to \infty}{\to}
        \queryVar
        \kerntildef{a b}{\depthSymbol}{\inputSymbol}{\inputSymbol'}
        \, ,
        &&
        \frac{
            \langle \query{\sequenceVariable, a \cdot}{\depthSymbol \headIndex}(\inputSymbol) , \query{\sequenceVariable, b \cdot}{\depthSymbol \headIndex}(\inputSymbol') \rangle
        }{
            \logitDimension
        }
        \convergeProb
        \queryVar
        \kerntildef{a b}{\depthSymbol}{\inputSymbol}{\inputSymbol'}
        \, ,
        \\
        &
        \E \left[
            \frac{
                \langle \key{\sequenceVariable, a \cdot}{\depthSymbol \headIndex}(\inputSymbol) , \key{\sequenceVariable, b \cdot}{\depthSymbol \headIndex}(\inputSymbol') \rangle
            }{
                \logitDimension
            }
        \right]
        \overset{\sequenceVariable \to \infty}{\to}
        \keyVar
        \kerntildef{a b}{\depthSymbol}{\inputSymbol}{\inputSymbol'}
        \, ,
        &&
        \frac{
            \langle \key{\sequenceVariable, a \cdot}{\depthSymbol \headIndex}(\inputSymbol) , \key{\sequenceVariable, b \cdot}{\depthSymbol \headIndex}(\inputSymbol') \rangle
        }{
            \logitDimension
        }
        \convergeProb
        \keyVar
        \kerntildef{a b}{\depthSymbol}{\inputSymbol}{\inputSymbol'}
        \, ,
        \\
        &
        \E \left[
            \frac{
                \langle \query{\sequenceVariable, a \cdot}{\depthSymbol \headIndex}(\inputSymbol) , \key{\sequenceVariable, b \cdot}{\depthSymbol \headIndex}(\inputSymbol') \rangle
            }{
                \logitDimension
            }
        \right]
        \overset{\sequenceVariable \to \infty}{\to}
        \delta_{\scaling = 1}
        \QKStd
        \kerntildef{a b}{\depthSymbol}{\inputSymbol}{\inputSymbol'}
        \, ,
        &&
        \frac{
            \langle \query{\sequenceVariable, a \cdot}{\depthSymbol \headIndex}(\inputSymbol) , \key{\sequenceVariable, b \cdot}{\depthSymbol \headIndex}(\inputSymbol') \rangle
        }{
            \logitDimension
        }
        \convergeProb
        \delta_{\scaling = 1}
        \QKStd
        \kerntildef{a b}{\depthSymbol}{\inputSymbol}{\inputSymbol'}
        \, .
    \end{align}
\end{lemma}

\begin{proof}[Proof of \Cref{lem:inner_prod_converge}]
    Notice that all the statements involving $\querySymbol^{\depthSymbol\headIndex}$ or $\keySymbol^{\depthSymbol\headIndex}$ are of the form
    \begin{align*}
        \frac{
            \indexedActivity{\depthSymbol - 1}{\sequenceVariable, a \cdot}{\inputSymbol}
            \weightMatSymbol_{\sequenceVariable}^{\depthSymbol\headIndex}
            (\weightMatSymbol_{\sequenceVariable}^{\depthSymbol\headIndex})^\top
            \indexedActivity{\depthSymbol - 1}{\sequenceVariable, b \cdot}{\inputSymbol'}^\top
        }{
            \logitDimension
        }
        \, ,
    \end{align*}
    i.e., with the same weight matrix multiplying the layer inputs $\activitySymbol_{\sequenceVariable}^{\depthSymbol - 1}(\inputSymbol)$ (recall that under $\scaling = 1$, we assumed $\weightsQ = \weightsK$ a.s.).
    Taking the expectation of the above term we obtain a term proportional up to $\queryStd$ or $\keyStd$ to
    \begin{align*}
        \E \left[
            \frac{
                \langle \indexedActivity{\depthSymbol - 1}{\sequenceVariable, a \cdot}{\inputSymbol} , \indexedActivity{\depthSymbol - 1}{\sequenceVariable, b \cdot}{\inputSymbol'} \rangle
            }{
                \layerDimension{\depthSymbol - 1}
            }
        \right]
        =
        \E \left[
            \indexedActivity{\depthSymbol - 1}{\sequenceVariable, a 1}{\inputSymbol} , \indexedActivity{\depthSymbol - 1}{\sequenceVariable, b 1}{\inputSymbol'} 
        \right]
        \, ,
    \end{align*}
    by the assumed exchangeability of $\activitySymbol_{\sequenceVariable}^{\depthSymbol - 1}$.
    Since the integrand converges in distribution by the assumed continuity of $\nonlinearity$ and the continuous mapping theorem \citep[theorem~9.3.7]{dudley02}, we can combine \Cref{lem:mmnt_propagation} with \Cref{lem:sup_ui} to obtain uniform integrability and thus
    $
        \E [
            \indexedActivity{\depthSymbol - 1}{\sequenceVariable, a 1}{\inputSymbol} , \indexedActivity{\depthSymbol - 1}{\sequenceVariable, b 1}{\inputSymbol'} 
        ]
        \to
        \kerntildef{a b}{\depthSymbol}{\inputSymbol}{\inputSymbol'}
    $
    by \Cref{thm:mean_convergence}, proving the convergence of expectations.
    
    The obtain the convergence in probability, it is sufficient to show that 
    \begin{align*}
        \E \Biggl|
                \frac{
                    \indexedActivity{\depthSymbol - 1}{\sequenceVariable, a \cdot}{\inputSymbol}
                    \weightMatSymbol_{\sequenceVariable}^{\depthSymbol\headIndex}
                    (\weightMatSymbol_{\sequenceVariable}^{\depthSymbol\headIndex})^\top
                    \indexedActivity{\depthSymbol - 1}{\sequenceVariable, b \cdot}{\inputSymbol'}^\top
                }{
                    \logitDimension
                }
        \Biggr|^2
        =
        \sum_{\substack{i_1, j_1 \\ i_2 j_2}}^{\layerDimension{\depthSymbol - 1}}
            \E \left[
                \indexedActivity{\depthSymbol - 1}{\sequenceVariable, a i_1}{\inputSymbol}
                \indexedActivity{\depthSymbol - 1}{\sequenceVariable, b j_1}{\inputSymbol}
                \indexedActivity{\depthSymbol - 1}{\sequenceVariable, a i_2}{\inputSymbol}
                \indexedActivity{\depthSymbol - 1}{\sequenceVariable, b j_2}{\inputSymbol}
                \weightMatSymbol_{\sequenceVariable, i_1 \cdot}^{\depthSymbol\headIndex}
                (\weightMatSymbol_{\sequenceVariable, j_1 \cdot}^{\depthSymbol\headIndex})^\top
                \weightMatSymbol_{\sequenceVariable, i_2 \cdot}^{\depthSymbol\headIndex}
                (\weightMatSymbol_{\sequenceVariable, j_2 \cdot}^{\depthSymbol\headIndex})^\top
            \right]
        \, ,
    \end{align*}
    converges to the square of the mean as by Chebyshev's inequality: $\Prob(|\genericRV_\sequenceVariable - \E \genericRV| \geq \delta) \leq \delta^{-2}(\E [\genericRV_\sequenceVariable^2] - \{ \E [\genericRV_\sequenceVariable] \}^2)$.
    Since we can bound each of the summands using H{\" o}lder's inequality combined with \Cref{lem:mmnt_propagation}, the limit of the above expectation will up to a constant coincide with that of
    \begin{align*}
        \frac{1}{(\layerDimension{\depthSymbol - 1})^2}
        \sum_{i , j}^{\layerDimension{\depthSymbol - 1}}
            \E \left[
                \indexedActivity{\depthSymbol - 1}{\sequenceVariable, a i}{\inputSymbol}
                \indexedActivity{\depthSymbol - 1}{\sequenceVariable, b i}{\inputSymbol}
                \indexedActivity{\depthSymbol - 1}{\sequenceVariable, a j}{\inputSymbol}
                \indexedActivity{\depthSymbol - 1}{\sequenceVariable, b j}{\inputSymbol}
            \right]
        =
        \E \left[
            \indexedActivity{\depthSymbol - 1}{\sequenceVariable, a 1}{\inputSymbol}
            \indexedActivity{\depthSymbol - 1}{\sequenceVariable, b 1}{\inputSymbol}
            \indexedActivity{\depthSymbol - 1}{\sequenceVariable, a 2}{\inputSymbol}
            \indexedActivity{\depthSymbol - 1}{\sequenceVariable, b 2}{\inputSymbol}
        \right]
        +
        \littleO((\layerDimension{\depthSymbol-1})^2)
        \, ,
    \end{align*}
    where the equality is by the assumed exchangeability.
    Since the individual columns of $\activitySymbol_{\sequenceVariable}^{\depthSymbol - 1}$ are asymptotically independent by assumption, we can use an argument analogous to that we made for the 
    $
        \E [
            \indexedActivity{\depthSymbol - 1}{\sequenceVariable, a 1}{\inputSymbol}
            \indexedActivity{\depthSymbol - 1}{\sequenceVariable, b 1}{\inputSymbol}
        ]
    $
    above to obtain the $(\kerntildef{a b}{\depthSymbol}{\inputSymbol}{\inputSymbol'})^2$ limit.
    Noting that the l.h.s.\ above is equal to
    $
        \E \bigl[
            \bigl(
                    \langle \indexedActivity{\depthSymbol - 1}{\sequenceVariable, a \cdot}{\inputSymbol} , \indexedActivity{\depthSymbol - 1}{\sequenceVariable, b \cdot}{\inputSymbol'} \rangle
                /
                    \layerDimension{\depthSymbol - 1}
            \bigr)^2
        \bigr]
    $
    concludes the proof.
\end{proof}

\section{Positional encodings}\label{sect:positional_encodings_appendix}

As in the proofs for attention without positional encodings, we assume the `infinite width, finite fan-out' construction of the~sequence of NNs.
In particular, we will assume that for any $\sequenceVariable \in \natnum$, there is a countably infinite set of random variables $\{ \posEmb{\sequenceVariable, \cdot i}{\depthSymbol} \colon i \in \natnum \} = \{ \posEmb{\cdot i}{\depthSymbol} \colon i \in \natnum \}$, where $\posEmb{\cdot i}{\depthSymbol} \sim \gauss ( 0 , \covPosEmb{}{} )$ i.i.d.\ over the $i$ index, but only a finite number $\peDimension \in \natnum$ is \texttt{add}-ed, 
$$
    \tildeActivity{\depthSymbol - 1}{\sequenceVariable}{\inputSymbol} = \sqrt{\interpKernCoeff}
    \activitySymbol_{\sequenceVariable}^{\depthSymbol - 1}(\inputSymbol) + \sqrt{1 - \interpKernCoeff}
    \posEmb{\sequenceVariable}{\depthSymbol}
    \, ,
$$
or \texttt{append}-ed
$$
    \tildeActivity{\depthSymbol - 1}{\sequenceVariable}{\inputSymbol} = [ \activitySymbol_{\sequenceVariable}^{\depthSymbol - 1}(\inputSymbol) , \posEmb{\sequenceVariable}{\depthSymbol}]
    \, ,
$$
to each of the layer inputs $\indexedActivity{\depthSymbol - 1}{\sequenceVariable}{\inputSymbol}$.
In the \texttt{append} case, we further assume $\interpKernCoeff = \lim_{\sequenceVariable \to \infty} \layerDimension{\depthSymbol - 1} / (\peDimension + \layerDimension{\depthSymbol - 1})$.

\subsection{NNGP limit}

Note that all \Cref{thm:gp_convergence_sqrt} relies on is that the layer's inputs $\{ \indexedActivity{\depthSymbol - 1}{\sequenceVariable}{\inputSymbol}  \colon \inputSymbol \in \indexSet \}$ converge in distribution to some $\indexedActivity{\depthSymbol - 1}{}{\inputSymbol}$ with mutually independent columns, and on the fact that the elementwise absolute moments of $\indexedActivity{\depthSymbol - 1}{\sequenceVariable, a i}{\inputSymbol}$ are bounded uniformly in $a, i$ and $\sequenceVariable$.
Let us thus replace $\activitySymbol_{\sequenceVariable}^{\depthSymbol - 1}(\inputSymbol)$ by $\tildeActivity{\depthSymbol - 1}{\sequenceVariable}{\inputSymbol}$, and see whether the proofs still apply.

\textbf{Exchangeability:} 
The proofs of exchangeability in \Cref{lem:exchangeability,lem:head_exchangeability,lem:logit_exchangeability} are all based on conditioning on $\indexedActivity{\depthSymbol - 1}{\sequenceVariable}{\inputSymbol}$ for some fixed finite subset of the inputs $\inputSymbol$, and then showing that the random variables are conditionally i.i.d.\ for any given $\sequenceVariable \in \natnum$.
If positional encodings are used, the variables will be again i.i.d.\ if we add $\posEmb{\sequenceVariable}{\depthSymbol}$ into the conditioning set.

\textbf{Convergence in distribution:}
To establish $\{ \tildeActivity{\depthSymbol - 1}{\sequenceVariable}{\inputSymbol} \colon \inputSymbol \in \indexSet \}$ converges in distribution in the \texttt{add} case, we can use a simple argument based on the Cram{\' e}r-Wold device and pointwise convergence of the characteristic function, which implies convergence in distribution by L{\' e}vy's continuity theorem (using the fact that $\posEmb{\sequenceVariable, \cdot i}{\depthSymbol}$ are assumed to be i.i.d.\ $\gauss(0, \covPosEmb{}{})$ and the distribution of a particular $\posEmb{\sequenceVariable, \cdot i}{\depthSymbol}$ does not change with $\sequenceVariable$).
An alternative approach has to be taken for the \texttt{append} case where the weak limit of the layer input's distribution may not be well defined; closer inspection of \Cref{sect:nngp_proofs} reveals that all the proofs depend on the convergence of the layer inputs only through \Cref{lem:inner_prod_converge,lem:mmnt_propagation}, which we discuss next.

\textbf{Convergence of inner products and boundedness of moments:}
The proof of each statement of \Cref{lem:mmnt_propagation} relies on $\{ \indexedActivity{\depthSymbol - 1}{\sequenceVariable}{\inputSymbol} \colon \inputSymbol \in \indexSet \}$ only through the bound  
$
    \max_{c \in [\spatialDimension]}
    \sup_{\substack{i \in \natnum}}
    \sup_{\sequenceVariable}
        \E | 
            \indexedActivity{\depthSymbol - 1}{\sequenceVariable, c i}{\inputSymbol} 
        |^t
    <
    \infty
$
which is essentially established using the assumed polynomial bound on $| \nonlinearity |$ and the Gaussianity of the weights at initialisation.
All we need to extend \Cref{lem:mmnt_propagation} to the case where positional encodings are used is to establish 
$
    \max_{c \in [\spatialDimension]}
    \sup_{\substack{i \in \natnum}}
    \sup_{\sequenceVariable}
        \E | 
            \tildeActivity{\depthSymbol - 1}{\sequenceVariable, c i}{\inputSymbol} 
        |^t
    <
    \infty
$.
This can be done by observing 
$
    \E | 
        \tildeActivity{\depthSymbol - 1}{\sequenceVariable, c i}{\inputSymbol} 
    |^t
    \leq
    \max \{
        \E | 
            \indexedActivity{\depthSymbol - 1}{\sequenceVariable, c i}{\inputSymbol}
        |^t
        ,
        \E |
            \posEmb{\sequenceVariable, c 1}{\depthSymbol}
        |^t
    \}
    <
    \infty
$
by the assumption $\posEmb{\sequenceVariable, \cdot i}{\depthSymbol} \sim \generalSum(0, \covPosEmb{}{})$ i.i.d.\ over the $i$ index for any $\sequenceVariable \in \natnum$.

Similarly, the proof of \Cref{lem:inner_prod_converge} can be modified by observing that
\begin{align*}
    \E \left[
        \frac{
            \langle \tildeActivity{\depthSymbol - 1}{\sequenceVariable, a \cdot}{\inputSymbol} , \tildeActivity{\depthSymbol - 1}{\sequenceVariable, b \cdot}{\inputSymbol'} \rangle
        }{
            \peDimension + \layerDimension{\depthSymbol - 1}
        }
    \right]
    =
    \frac{\layerDimension{\depthSymbol - 1}}{\peDimension + \layerDimension{\depthSymbol - 1}}
    \E \left[
        \indexedActivity{\depthSymbol - 1}{\sequenceVariable, a 1}{\inputSymbol} , \indexedActivity{\depthSymbol - 1}{\sequenceVariable, b 1}{\inputSymbol'} 
    \right]
    +
    \frac{\peDimension}{\peDimension + \layerDimension{\depthSymbol - 1}}
    \underbrace{
        \E \left[
            \posEmb{\sequenceVariable, a 1}{\depthSymbol}
            \posEmb{\sequenceVariable, b 1}{\depthSymbol}
        \right]
    }_{= \covPosEmb{a b}{}}
    \, ,
\end{align*}
in the \texttt{append} case by the independence of $\posEmb{\sequenceVariable}{\depthSymbol}$ and exchangeability.
Using the Gaussianity of positional encodings and $\interpKernCoeff = \lim_{\sequenceVariable \to \infty}   \layerDimension{\depthSymbol - 1} / (\peDimension + \layerDimension{\depthSymbol - 1})$, an analogous argument to that made in \Cref{lem:inner_prod_converge} can be used to establish convergence of the r.h.s.\ to $\interpKern\circ\kerntildef{a b}{\depthSymbol}{\inputSymbol}{\inputSymbol'} = \interpKernCoeff \kerntildef{a b}{\depthSymbol}{\inputSymbol}{\inputSymbol'} + (1 - \interpKernCoeff) \covPosEmb{a b}{}$ in both probability and expectation.
For the \texttt{add} case, 
\begin{align*}
    \E \left[
        \frac{
            \langle \tildeActivity{\depthSymbol - 1}{\sequenceVariable, a \cdot}{\inputSymbol} , \tildeActivity{\depthSymbol - 1}{\sequenceVariable, b \cdot}{\inputSymbol'} \rangle
        }{
            \layerDimension{\depthSymbol - 1}
        }
    \right]
    =
    \interpKernCoeff
    \E \left[
        \indexedActivity{\depthSymbol - 1}{\sequenceVariable, a 1}{\inputSymbol} , \indexedActivity{\depthSymbol - 1}{\sequenceVariable, b 1}{\inputSymbol'} 
    \right]
    +
    (1 - \interpKernCoeff)
    \underbrace{
        \E \left[
            \posEmb{\sequenceVariable, a 1}{\depthSymbol}
            \posEmb{\sequenceVariable, b 1}{\depthSymbol}
        \right]
    }_{= \covPosEmb{a b}{}}
    \, ,
\end{align*}
again by the independence of $\posEmb{\sequenceVariable}{\depthSymbol}$ and exchangeability, and thus a similar argument to the one above applies.

Putting all of the above together, addition of positional encodings does not prevent GP behaviour in the infinite width limit;
the only modification of the results in \Cref{sect:nngp_proofs} is thus replacement of any $\kerntildef{a b}{\depthSymbol}{\inputSymbol}{\inputSymbol'}$ in the expression for the limiting covariance of $\activationSymbol^{\depthSymbol}$ and $\logitSymbol^{\depthSymbol}$ by $\interpKern\circ\kerntildef{a b}{\depthSymbol}{\inputSymbol}{\inputSymbol'}$.

\subsection{NTK limit}

There are two sets of changes to the NTK limit.
First, the gradients w.r.t.\ $\indexedActivity{\depthSymbol - 1}{\sequenceVariable}{\inputSymbol}$ in the \emph{indirect} part will now be multiplied by $\sqrt{\interpKernCoeff}$ in the \texttt{add} case, and by $[\layerDimension{\depthSymbol - 1} / (\layerDimension{\depthSymbol} + \peDimension)]^{1/2}$---to ensure convergence of corresponding inner products---in the \texttt{append} case, and all the terms of the form 
$$
    \E \Biggl|
        \frac{
            \langle \tildeActivity{\depthSymbol - 1}{\sequenceVariable, a \cdot}{\inputSymbol} , \tildeActivity{\depthSymbol - 1}{\sequenceVariable, b \cdot}{\inputSymbol'} \rangle
        }{
            \peDimension + \layerDimension{\depthSymbol - 1}
        }
    \Biggr|^k
    \, ,
$$
for some $k > 0$ in the \emph{direct} part will converge to the $k$\textsuperscript{th} power of the $\interpKern\circ\kerntildef{a b}{\depthSymbol}{\inputSymbol}{\inputSymbol'} = \alpha \kerntildef{a b }{\depthSymbol}{\inputSymbol}{\inputSymbol'} + (1 - \alpha) \covPosEmb{a b}{}$ kernel as discussed.
Since we have shown that \Cref{lem:mmnt_propagation,lem:inner_prod_converge} hold mutatis mutandis in the previous section, the rest of the proofs in the direct part can be modified in the obvious way, replacing $\kerntildef{a b}{\depthSymbol}{\inputSymbol}{\inputSymbol'}$ by $\interpKern\circ\kerntildef{a b}{\depthSymbol}{\inputSymbol}{\inputSymbol'}$ as necessary.

Second, there will be a new contribution to the \emph{direct} part due to the gradient w.r.t.\ the trainable $\posEmb{\sequenceVariable}{\depthSymbol}$.
Since $\posEmb{\sequenceVariable}{\depthSymbol}$ is \texttt{add}-ed (resp.\ \texttt{append}-ed) to the layer input $\indexedActivity{\depthSymbol - 1}{\sequenceVariable}{\inputSymbol}$, this contribution will be quite similar to the \emph{indirect} contribution, however with $\ntkhatf{\gIndexA \gIndexZA, \gIndexB \gIndexZB}{\depthSymbol}{\inputSymbol}{\inputSymbol'}$ (\Cref{eq:ntk_hat}) replaced by $(1 - \interpKernCoeff) \delta_{\gIndexA = \gIndexB} \delta_{\gIndexZA = \gIndexZB}$.
Inspecting \Cref{lem:gg_ntk,lem:vv_ntk,lem:gv_ntk}, this will lead to two changes.
Firstly, since $\E |(1 - \interpKernCoeff) \delta_{\gIndexA = \gIndexB} \delta_{\gIndexZA = \gIndexZB}|^{t} < \infty$, all bounds involving $\E | \ntkhat_{\gIndexA \gIndexZA, \gIndexB \gIndexZB}|^t$ can be trivially reduced.
Secondly, as shown in the previous section, 
all appearances of the $\kerntildef{a b}{\depthSymbol}{\inputSymbol}{\inputSymbol'}$ are to be replaced $\interpKern\circ\kerntildef{a b}{\depthSymbol}{\inputSymbol}{\inputSymbol'}$, including those involved indirectly through the modified asymptotic distribution of $\logitSymbol_{\sequenceVariable}^{\depthSymbol}$.
The rest of the proofs is affected by the introduction of positional encodings only through \Cref{lem:mmnt_propagation,lem:inner_prod_converge} which, as mentioned, do hold in a modified form.
Substituting $(1 - \interpKernCoeff) \delta_{\gIndexA = \gIndexB} \delta_{\gIndexZA = \gIndexZB}$ for $\ntkhatf{\gIndexA \gIndexZA, \gIndexB \gIndexZB}{\depthSymbol}{\inputSymbol}{\inputSymbol'}$ in \Cref{lem:gg_ntk,lem:vv_ntk}, we thus conclude that the new contribution to the NTK due to the gradient w.r.t.\ $\posEmb{\sequenceVariable}{\depthSymbol}$ is
\begin{align*}
    &
    (1 - \interpKernCoeff)
    \OVVar
    \sum_{c = 1}^{\spatialDimension}
        \E [
            \tildeLogit{\fIndexA c}{\depthSymbol 1} (\inputSymbol)
            \tildeLogit{\fIndexB c}{\depthSymbol 1} (\inputSymbol')
        ]
    +
    \\
    \delta_{\scaling = \frac{1}{2}}
    &(1 - \interpKernCoeff)
    \OVVar
    \QKVar
    \sum_{\substack{c_1, c_2 \\ d_1, d_2}}^{\spatialDimension}
        \interpKern \circ \kerntildef{c_1 c_2}{\depthSymbol}{\inputSymbol}{\inputSymbol'}
        \left( 
            \delta_{\substack{d_1 = d_2}}
            \interpKern \circ \kerntildef{\fIndexA \fIndexB}{\depthSymbol}{\inputSymbol}{\inputSymbol'}
            +
            \delta_{\substack{\fIndexA = \fIndexB}}
            \interpKern \circ \kerntildef{d_1 d_2}{\depthSymbol}{\inputSymbol}{\inputSymbol'}
        \right)
        \E \left[
            \frac{
                \partial
                \tildeLogit{\fIndexA c_1}{\depthSymbol 1} (\inputSymbol)
            }{
                \partial 
                \logitSymbol_{\fIndexA d_1}^{\depthSymbol 1} (\inputSymbol)
            }
            \frac{
                \partial
                \tildeLogit{\fIndexB c_2}{\depthSymbol 1} (\inputSymbol')
            }{
                \partial 
                \logitSymbol_{\fIndexB d_2}^{\depthSymbol 1} (\inputSymbol')
            }
        \right]
    \, .
\end{align*}

\section{Residual attention}\label{sect:residual_attention_appendix}

Observe that by \citep{garriga2019deep,yang2019v2}, the covariance induced by the skip connection, $\indexedActivation{\depthSymbol}{\sequenceVariable}{\inputSymbol} = \sqrt{\interpKernCoeff} \indexedActivity{\depthSymbol - 1}{\sequenceVariable}{\inputSymbol} + \sqrt{1 - \interpKernCoeff} \tildeActivation{\depthSymbol}{\sequenceVariable}{\inputSymbol}$, in the infinite width limit is equal to 
\begin{align*}
    \E [
        \indexedActivation{\depthSymbol}{\fIndexA 1}{\inputSymbol}
        \indexedActivation{\depthSymbol}{\fIndexB 1}{\inputSymbol'}
    ]
    &=
    \interpKernCoeff
    \E [
        \indexedActivity{\depthSymbol - 1}{\fIndexA 1}{\inputSymbol}
        \indexedActivity{\depthSymbol - 1}{\fIndexB 1}{\inputSymbol'}
    ]
    +
    (1 - \interpKernCoeff)
    \E [
        \tildeActivation{\depthSymbol}{\fIndexA 1}{\inputSymbol}
        \tildeActivation{\depthSymbol}{\fIndexA 1}{\inputSymbol'}
    ]
    \\
    &=
    \interpKernCoeff
    \kerntildef{\fIndexA \fIndexB}{\depthSymbol}{\inputSymbol}{\inputSymbol'}
    +
    (1 - \interpKernCoeff)
    \E [
        \tildeActivation{\depthSymbol}{\fIndexA 1}{\inputSymbol}
        \tildeActivation{\depthSymbol}{\fIndexA 1}{\inputSymbol'}
    ]
    \, .
\end{align*}
To obtain the 
$
    \interpKernCoeff \kerntildef{\fIndexA \fIndexB}{\depthSymbol}{\inputSymbol}{\inputSymbol'}
    +
    (1 - \interpKernCoeff)
    \covPosEmb{\fIndexA \cdot}{}
    \kerntildef{}{\depthSymbol}{\inputSymbol}{\inputSymbol'}
    \covPosEmb{\fIndexB \cdot}{\top}
$
from \Cref{eq:residual_kernel}, it is thus sufficient to choose $\tildeActivation{\depthSymbol}{\sequenceVariable}{\inputSymbol}$ to be the output of attention layer under the $\genericDimenstion^{-1}$ scaling with structured positional encodings (covariance $\covPosEmb{}{}$), identity function for $\softmax$ and the interpolation parameter \emph{for the attention layer} set to zero, resulting in the 
$    
    \covPosEmb{\fIndexA \cdot}{}
    \kerntildef{}{\depthSymbol}{\inputSymbol}{\inputSymbol'}
    \covPosEmb{\fIndexB \cdot}{\top}
$ (see \Cref{tab:kernel_overview}).

\end{document}